\documentclass[nohyperref]{article}
\usepackage{fullpage}
\usepackage{parskip}

\usepackage[utf8]{inputenc} %
\usepackage[T1]{fontenc}    %
\usepackage{url}            %
\usepackage{booktabs}       %
\usepackage{amsfonts}       %
\usepackage{natbib}
\usepackage{array}
\usepackage{mathtools}
\usepackage{graphicx}

\usepackage{hyperref}

\usepackage{shortcuts}

\newcommand{\cvar}{\op{CVaR}}

\newcommand{\mabregret}{\op{Regret}^{\textsc{MAB}}}
\newcommand{\rlregret}{\op{Regret}^{\textsc{RL}}}
\newcommand{\bon}{\textsc{Bon}}
\newcommand{\hobon}{\textsc{Bon}^{\textsc{Hoeff}}}

\newcommand{\bebon}{\textsc{Bon}^\textsc{Bern}}

\newcommand{\densitymin}{p_{\min}}

\newcommand{\ie}{\textit{i.e.}}
\newcommand{\Ie}{\textit{I.e.}}

\newcommand{\Saug}{\Scal^{\op{Aug}}}
\newcommand{\Piaug}{\Pi^{\op{Aug}}}
\newcommand{\Haug}{\Hcal^{\op{Aug}}}
\newcommand{\unif}[1]{\op{Unif}\prns{#1}}

\newcommand{\alg}{\op{Alg}}
\newcommand{\ber}{\op{Ber}}
\newcommand{\kl}[2]{D_{\textsc{KL}}\prns{#1,#2}}
\newcommand{\grid}{\op{grid}_\eta([0,1])}

\newcommand{\disc}{\op{disc}}
\newcommand{\adapted}{\op{adapted}}

\newcommand{\mabalg}{\textsc{Bernstein-UCB}}
\newcommand{\rlalg}{$\cvar$-\textsc{UCBVI}}

\title{Near-Minimax-Optimal Risk-Sensitive \\ Reinforcement Learning with CVaR}
\newcommand{\addLink}[1]{\href{#1}{\nolinkurl{#1}}}
\usepackage{authblk}
\author[1,2]{Kaiwen Wang\thanks{Correspondence to \addLink{https://kaiwenw.github.io/}.}}
\author[1,2]{Nathan Kallus}
\author[1]{Wen Sun}
\affil[1]{Cornell University}
\affil[2]{Cornell Tech}
\affil[ ]{\small\texttt{\{kw437,~kallus,~ws455\}@cornell.edu}}
\date{\today}

\begin{document}

\maketitle

\begin{abstract}
In this paper, we study risk-sensitive Reinforcement Learning (RL), focusing on the objective of
Conditional Value at Risk (CVaR) with risk tolerance $\tau$.
Starting with multi-arm bandits (MABs), we show the minimax CVaR regret rate is $\Omega(\sqrt{\tau^{-1}AK})$, where $A$ is the number of
actions and $K$ is the number of episodes, and that it is achieved by an Upper Confidence Bound algorithm with a novel Bernstein bonus.
For online RL in tabular Markov Decision Processes (MDPs), we show a minimax regret lower bound of $\Omega(\sqrt{\tau^{-1}SAK})$ (with normalized cumulative rewards), where $S$ is the number of states, and we propose a novel bonus-driven Value Iteration procedure.
We show that our algorithm achieves the optimal regret of $\widetilde\Ocal(\sqrt{\tau^{-1}SAK})$ under a continuity assumption and
in general attains a near-optimal regret of $\widetilde\Ocal(\tau^{-1}\sqrt{SAK})$, which is minimax-optimal for constant $\tau$.
This improves on the best available bounds.
By discretizing rewards appropriately, our algorithms are computationally efficient.
\end{abstract}

\section{Introduction}
Reinforcement Learning (RL) \citep{sutton2018reinforcement} is the canonical framework for sequential decision making under uncertainty, with applications in personalizing recommendations \citep{bottou2013counterfactual}, robotics \citep{rajeswaran2017epopt}, healthcare \citep{murphy2003optimal} and education \citep{singla2021reinforcement}.
In vanilla RL, the objective is to maximize the \emph{average} of returns, the cumulative rewards collected by the policy.
As RL is increasingly applied in consequential settings, it is often necessary to account for risk beyond solely optimizing for the average.

Conditional Value at Risk ($\cvar$) is a popular coherent measure of risk
\citep{rockafellar2000optimization,filippi2020conditional}.
For a random return $X$ (where higher is better), the CVaR with risk tolerance $\tau\in(0,1]$ is defined as
\begin{align}\textstyle
    \cvar_\tau(X) := \sup_{b\in\RR}(b-\tau^{-1}\EE[\prns{b-X}^+]), \label{eq:cvar-definition}
\end{align}
where $x^+ = \max(x,0)$.
$\cvar_\tau(X)$ is the average outcome among the \emph{worst} $\tau$-percent of cases, and when $X$ is continuous this exactly corresponds to those less than or equal to the $\tau$-th quantile
\citep{acerbi2002coherence}, \ie,
\begin{align}\textstyle
    \cvar_\tau(X) = \EE[X\mid X\leq F_X^\dagger(\tau)], \label{eq:cvar-continuous-identity}
\end{align}
where $F_X^\dagger(\tau)=\inf\braces{ x: F_X(x)\geq \tau }$ is the $\tau$-th quantile of $X$, a.k.a. the Value at Risk (VaR).
A high risk tolerance $\tau=1$ recovers the risk-neutral expectation, \ie, $\cvar_1(X) = \EE X$.
As $\tau$ decreases, $\cvar_\tau$ models the worst-case outcome, \ie, $\lim_{\tau\to 0}\cvar_\tau(X) = \op{ess}\inf X$.
In the CVaR RL model we consider, $X$ is the return of a policy, so our objective captures the tail-risk of the returns distribution.
Another motivating perspective is that CVaR RL is equivalent to the robust MDP model, \ie, expected value under worst-case perturbation of the transition kernel \citep{chow2015risk}.
Thus, CVaR RL is an attractive alterantive to vanilla RL in safety-critical applications.

In this paper, we provide algorithms with state-of-the-art regret guarantees for tabular, online decision making with the CVaR objective.
To start, we prove tight lower bounds on the expected CVaR regret (formalized in \cref{sec:problem-setup}) for both multi-arm bandit (MAB) and RL problems.
We then propose \mabalg{}, an Upper Confidence Bound (UCB) algorithm with a novel bonus constructed using Bernstein's inequality, and we prove it is
minimax-optimal\footnote{Following \citet{azar2017minimax}, we say an algorithm is \emph{minimax optimal} if its regret matches (up to log terms) our novel minimax lower bound, in all problem parameters. Sometimes, this is also referred to as ``nearly-minimax-optimal'' \citep{zhou2021nearly}. }.
Compared to Brown-UCB \citep{tamkin2019distributionally}, \mabalg{} is minimax-optimal in general, without requiring reward distributions to be continuous.

We then turn to tabular RL with the CVaR objective.
We propose \rlalg{}, a novel bonus-driven Value Iteration (VI) algorithm in an augmented MDP.
The augmented MDP framework of \citet{bauerle2011markov} reveals Bellman equations for CVaR RL that served as the initial catalyst for the VI approach.
We provide two choices of bonuses for \rlalg{}, based on Hoeffding's and Bernstein's inequalities.
With the Bernstein bonus, we guarantee a CVaR regret of $\wt\Ocal\prns*{\tau^{-1}\sqrt{SAK}}$, where $K$ is the number of episodes, $S,A$ are the sizes of the state and action spaces, respectively.
This improves over the previous bound of $\wt\Ocal\prns*{\tau^{-1}\sqrt{S^3AHK}}$ \citep{bastani2022regret} in $S$ and $H$, the horizon length. Note that we work under the normalized returns model, so there should be no $H$ scaling in the bound \citep{jiang2018open}.
If $\tau$ is a constant, our result is already minimax-optimal.
Surprisingly, however, our lower bound actually scales as $\tau^{-1/2}$.
Under an assumption that the returns of any policies are continuously distributed with lower-bounded density,
we improve the upper bound on the regret of \rlalg{}
with the Bernstein bonus
to
$\wt\Ocal\prns*{\sqrt{\tau^{-1}SAK}}$.
This establishes \rlalg{} as the first algorithm with minimax-optimal regret for risk-sensitive CVaR RL, under the continuity assumption.
Our key technical innovation is decomposing the regret using a novel simulation lemma for CVaR RL and
precisely bounding the sum of variance bonuses with the Law of Total Variance \citep{azar2017minimax}.

\subsection{Related Literature}
\textbf{CVaR MAB: }
\citet{kagrecha2019distribution} proposed a successive rejects algorithm for best CVaR arm identification, but it does not have regret guarantees.
\citet{tamkin2019distributionally} proposed two algorithms for CVaR MAB and analyze upper bounds on their regret, but not lower bounds.
Their ``CVaR-UCB" builds a confidence band for the reward distribution of each arm via Dvoretzky-Kiefer-Wolfowitz inequality, resulting in an optimistic estimate of CVaR.
This leads to a suboptimal $\tau^{-1}$ dependence in the regret but may empirically work better if $\tau$ is not approaching $0$.
Their ``Brown-UCB" is structurally similar to our \mabalg{}, but they use a Hoeffding bonus that ensures optimism only if all arms have continuously distributed rewards \citep[Theorem 4.2]{brown2007large}.
We propose a Bernstein bonus that attains the minimax-optimal rate $\sqrt{\tau^{-1}AK}$ without any assumptions on the reward distribution.

\textbf{Regret bounds for CVaR RL: }
To the best of our  knowledge, \citet{bastani2022regret} is the first and only work with regret bounds for CVaR RL (formalized in \cref{sec:problem-setup}).
Their algorithm iteratively constructs optimistic MDPs by routing unexplored states to a sink state with the maximum reward.
This approach leads to a CVaR regret bound of $\wt\Ocal\prns*{\tau^{-1}\sqrt{S^3AHK}}$ \citep[Theorem 4.1]{bastani2022regret}, which is sub-optimal.
The authors conjectured that bonus-based optimism could improve the bound by a $S\sqrt{H}$ factor.
Our proposed \rlalg{} indeed enjoys these improvements,
leading to a $\wt\Ocal\prns*{\tau^{-1}\sqrt{SAK}}$ regret guarantee in \cref{thm:bernstein-bonus-regret}.
If returns are continuously distributed, we further improve the $\tau$ dependence,
leading to the minimax-optimal result in \cref{thm:bernstein-bonus-regret2}.

\textbf{CVaR RL without regret guarantees: }
\citet{keramati2020being} proposed a distributional RL approach \citep{bellemare2017distributional} for RL with the CVaR objective.
A key difference is that \citet{keramati2020being} focuses on the easier task of identifying a policy with high CVaR.
On the other hand, \citet{bastani2022regret} and our work focuses on algorithms with low CVaR regret, which guarantees safe exploration.
Note that low-regret methods can be converted into probably approximately correct (PAC) CVaR RL, by taking the uniform mixture of policies from the low-regret algorithm.

\citet{tamar2015optimizing} derived the policy gradient for the CVaR RL objective and showed asymptotic convergence to a local optimum.
\citet{chow2014algorithms} developed actor-critic algorithms for the mean-CVaR objective, \ie, maximizing expected returns subject to a CVaR constraint.
Another motivating perspective for CVaR RL is its close ties to robust MDPs \citep{wiesemann2013robust}.
Specifically, \citet[Proposition 1]{chow2015risk} showed that the CVaR of returns is equivalent to the expected returns under the worst-case perturbation of the transition kernel in some uncertainty set.
While the uncertainty set is not rectangular, \citet{chow2015risk} derived tractable robust Bellman equations and proved convergence to a globally optimal CVaR policy.
However, these methods for CVaR RL do not lead to low-regret algorithms, which is our focus.

\textbf{Risk-sensitive RL with different risk measures:  }
Prior works have also proved risk-sensitive RL regret bounds in the context of other risk measures that are not directly comparable to the CVaR RL setting we consider.
\citet{fei2020risk,fei2021exponential,liang2022bridging} showed Bellman equations and regret guarantees with the entropic risk measure based on an exponential utility function.
\citet{du2022risk,lam2023riskaware} studied the more conservative \emph{Iterated} CVaR objective, which considers the risk of the reward-to-go at every step along the trajectory.
In contrast, our setup aims to holistically maximize the CVaR of the \emph{total} returns.

\textbf{Risk-sensitive regret lower bounds: }
\citet{fei2020risk,liang2022bridging} showed regret lower bounds for risk-sensitive RL with the entropic risk measure.
We show \emph{tight} lower bounds for risk-sensitive MAB and RL with the CVaR objective, which to the best of our knowledge are the first lower bounds for this problem.

\textbf{Safety in offline RL: }
While our focus is online RL, risk-aversion has also been studied in offline RL. Some past works include offline learning with risk measures \citep{urpi2021risk} and distributional robustness \citep{panaganti2022robust,si2020distributional,kallus2022doubly,zhou2021finite}.

\section{Problem Setup}\label{sec:problem-setup}
As warmup, we consider $\cvar_\tau$ regret in a multi-arm bandit (MAB) problem with $K$ episodes.
At each episode $k\in[K]$, the learner selects an arm $a_k\in\Acal$ and observes reward $r_k\sim\nu(a_k)$, where $\nu(a)$ is the reward distribution of arm $a$.
The learner's goal is to compete with the arm with the highest $\cvar_\tau$ value, \ie, $a^\star = \max_{a\in\Acal}\cvar_\tau(\nu(a))$, and minimize the regret, defined as
$
\mabregret_\tau(K) = \sum_{k=1}^K \cvar_\tau(\nu(a^\star))-\cvar_\tau(\nu(a_k)).
$

The focus of this paper is online RL, which generalizes the MAB (where $S=H=1$).
The statistical model is tabular Markov Decision Processes (MDPs) \citep{rltheorybookAJKS}, with finite state space $\Scal$ of size $S$,
finite action space $\Acal$ of size $A$, and horizon $H$.
Let $\Pi_\Hcal$ be the set of history-dependent policies, so each policy is $\pi=\prns{\pi_h:\Scal\times\Hcal_h\to\Delta(\Acal)}_{h\in[H]}$ where $\Hcal_h = \braces{(s_i,a_i,r_i)}_{i\in[h-1]}$ is the history up to time $h$.
At each episode $k\in[K]$, the learner plays a history-dependent policy $\pi^k\in\Pi_\Hcal$, which induces a distribution over trajectories as follows.
First, start at the fixed initial state $s_1\in\Scal$.
Then, for each time $h=1,2,\ldots,H$, sample an action $a_h\sim\pi_h^k(s_h,\Hcal_h)$, which leads to a reward $r_h\sim R(s_h,a_h)$ and the next state $s_{h+1}\sim P^\star(s_h,a_h)$.
Here, $P^\star:\Scal\times\Acal\to\Delta(\Scal)$ is the \emph{unknown} Markov transition kernel and $R:\Scal\times\Acal\to\Delta([0,1])$ is the \emph{known} reward distribution.
The return is the sum of rewards from this process, $R(\pi)=\sum_{h=1}^Hr_h$, which is a random variable.
We posit the return is almost surely \emph{normalized} in that $R(\pi)\in[0,1]$ w.p. $1$ \citep[as in][Section 2.1]{jiang2018open}.
We note normalized returns allows for sparsity in the rewards, and thus is strictly more general for regret upper bounds.
Many prior works do not normalize, so their returns may scale with $H$.
When comparing to their bounds, we make the scaling consistent  by dividing them by $H$.

We focus on the setting we call \emph{CVaR RL}, in which the learner's goal is to compete with a $\cvar_\tau$-optimal policy, \ie, $\pi^\star \in \argmax_{\pi\in\Pi_\Hcal}\cvar_\tau(R(\pi))$.
Toward this end, we define the regret as $\rlregret_\tau(K) = \sum_{k=1}^K\cvar_\tau^\star-\cvar_\tau(R(\pi^k))$, where $\cvar_\tau^\star=\cvar_\tau(R(\pi^\star))$.
CVaR RL captures vanilla risk-neutral RL when $\tau=1$ and Worst Path RL \citep{du2022risk} when $\tau\to 0$.
We prove lower bounds in expected regret.
For upper bounds, we give high probability regret bounds, which implies upper bounds (with the same dependencies on problem parameters) in expected regret by integrating over the failure probability $\delta\in(0,1)$.

\textbf{Notation: }
$[i:j]=\braces{i,i+1,\dots,j}$, $[n]=[1:n]$ and $\Delta(\Scal)$ is the set of distributions on $\Scal$.
We set $L=\log(HSAK/\delta)$ (for MAB, $L=\log(AK/\delta)$), where $\delta$ is the desired failure probability provided to the algorithm.
Please see \cref{tab:notation} for a comprehensive list of notations.

\section{Lower Bounds}
We start with the minimax lower bound for $\cvar_\tau$ MAB.
\begin{restatable}{theorem}{mabLowerBound}\label{thm:mab-lower-bound}
Fix any $\tau\in(0,\nicefrac12),A\in\NN$.
For any algorithm, there is a MAB problem with Bernoulli rewards s.t. if $K\geq\sqrt{\frac{A-1}{8\tau}}$, then $\Eb{\mabregret_\tau(K)}\geq \frac{1}{24e}\sqrt{\frac{(A-1)K}{\tau}}$.
\end{restatable}
\begin{proofsketch}
Our proof is inspired by the lower bound construction for the vanilla MAB \citep[Theorem 15.2]{lattimore2020bandit}.
The key idea is to fix any learner, and construct two MAB instances that appear similar to the learner but in reality have very different CVaR value.
Specifically, for any $\eps\in(0,1)$, we need two reward distributions such that their KL-divergence is $\Ocal(\eps^2\tau^{-1})$ while their CVaRs differ by $\Omega(\tau^{-1}\eps)$.
We show that $\ber(1-\tau)$ and $\ber(1-\tau+\eps)$ satisfy this.
\end{proofsketch}
Compared to the vanilla MAB minimax lower bound of $\Omega(\sqrt{AK})$, our result for $\cvar_\tau$ MAB has an extra $\sqrt{\tau^{-1}}$ factor.
This proves that it is information-theoretically harder to be more risk-averse with $\cvar_\tau$.
While Brown-UCB of \citet{tamkin2019distributionally} appears to match this lower bound,
their proof hinges on the continuity of reward distributions, which is invalid for Bernoulli rewards.
In \cref{thm:bandit upper}, we show that \mabalg{} is minimax-optimal over all reward distributions.

We next extend the above result to the RL setting.
\begin{restatable}{corollary}{rlLowerBound}\label{cor:rl-lower-bound}
Fix any $\tau\in(0,\nicefrac12),A,H\in\NN$.
For any algorithm, there is an MDP (with $S=\Theta(A^{H-1})$) s.t. if $K\geq\sqrt{\frac{S(A-1)}{8\tau}}$, then $\Eb{\rlregret_\tau(K)}\geq \frac{1}{24e}\sqrt{\frac{S(A-1)K}{\tau}}$.
\end{restatable}
The argument is to show that a class of MDPs with rewards only at the last layer
essentially reduces to a MAB with exponentially many actions.
Thus, the hardest CVaR RL problems are actually very big CVaR MAB problems.
The bound does not scale with $H$ as we've assumed returns to be normalized in $[0,1]$ \citep{jiang2018open}.

\section{Risk Sensitive MAB} \label{sec:bandit-upper}
In this section, we propose a simple modification to the classic Upper Confidence Bound (UCB) with a novel Bernstein bonus that enjoys minimax-optimal regret.
In the classic UCB algorithm \citep{auer2002finite}, the bonus quantifies the confidence band from Hoeffding's inequality.
Instead, we propose to build a confidence band of $\mu(b,a)=\Eb[R\sim\nu(a)]{(b-R)^+}$ by using a Bernstein-based bonus (\cref{eq:mab-ucb-bon}).
The standard deviation $\sqrt{\tau}$ in our bonus is crucial for obtaining the minimax-optimal regret.
\begin{restatable}{theorem}{banditUpper}\label{thm:bandit upper}
For any $\delta\in(0,1)$, w.p. at least $1-\delta$, \mabalg{} with $\eps\leq\sqrt{A/2\tau K}$ enjoys
\begin{align*}
    \mabregret_\tau(K)\leq 4\sqrt{\tau^{-1}AK}L + 16\tau^{-1}AL^2.
\end{align*}
\end{restatable}
\begin{algorithm*}[t!]
    \caption{\mabalg}
    \label{alg:mab-ucb}
    \begin{algorithmic}[1]
        \State\textbf{Input:} risk tolerance $\tau$, number of episodes $K$,
        failure probability $\delta$, approximation parameter $\eps$.
        \For{episode $k=1,2,\dots,K$}
            \State Compute counts $N_k(a) = 1\vee \sum_{i=1}^{k-1}\I{a_i=a}$.
            \State Define pessimistic estimate of $\mu(b,a)=\Eb[R\sim\nu(a)]{(b-R)^+}$, \ie, for all $b,a$,
            \begin{align}
                &\wh\mu_k(b,a) = \frac{1}{N_k(a)}\sum_{i=1}^{k-1}(b-r_i)^+\I{a_i=a},
                \qquad
                \bon_k(a) = \sqrt{ \frac{2\tau\log(AK/\delta)}{N_k(a)} } + \frac{\log(AK/\delta)}{N_k(a)} \label{eq:mab-ucb-bon}.
            \end{align}
            \State Compute $\eps$-optimal solutions $\wh b_{a,k}$, \ie, for all $a$, \label{line:mab-compute-eps-optimal-bs}
            \begin{align*}\textstyle
                &\wh f_k(\wh b_{a,k},a)\geq \max_{b\in[0,1]}\wh f_k(b,a) - \eps,
                ~~~~\text{where,}~~~~\wh f_k(b,a) = b-\tau^{-1}\prns{ \wh\mu_k(b,a)-\bon_k(a) }.
            \end{align*}
            \State Compute and pull the action for this episode, $a_k = \argmax_{a\in\Acal}\wh f_k(\wh b_{a,k},a)$.
            Receive reward $r_k\sim \nu(a_k)$. \label{line:mab-compute-ak}
        \EndFor
    \end{algorithmic}
\end{algorithm*}
\begin{proofsketch}
First, we use Bernstein's inequality to build a confidence band of $\mu(b,a)$ at $$b^\star_a=\argmax_{b\in[0,1]}\braces{ b-\tau^{-1}\mu(b,a) }.$$
Conveniently, $b^\star_a$ is the $\tau$-th quantile of $\nu(a)$, hence $\Varb[R\sim\nu(a)]{(b^\star_a-R)^+}\leq\tau$.
This proves pessimism with our Bernstein-based bonus, \ie, $\wh\mu_k(b^\star_a,a)-\bon_k(a)\leq\mu(b^\star_a,a)$.
Pessimism in turn implies \emph{optimism} in CVaR, \ie, $\cvar_\tau^\star\leq \wh\cvar_\tau^k := \max_{b\in[0,1]}\braces{b-\tau^{-1}\prns{\wh\mu_k(b,a_k)-\bon_k(a_k)}}$.
This allows us to decompose regret into (1) the sum of bonuses plus (2) the difference between empirical and true CVaR of $\nu(a_k)$.
(1) is handled using a standard pigeonhole argument.
To bound (2), we prove a new concentration inequality for CVaR that holds for any bounded
random variable (\cref{thm:cvar-concentration-main}), which may be of independent interest. \looseness=-1
\end{proofsketch}
Up to log terms, the resulting bound matches our lower bound in \cref{thm:mab-lower-bound}, proving that \mabalg{} is minimax-optimal.
As noted earlier, under the assumption that rewards are continuous, Brown-UCB \citep{tamkin2019distributionally} also matches our novel lower bound.
When working with continuous distributions, CVaR takes the convenient form in \cref{eq:cvar-continuous-identity},
which roughly suggests that Hoeffding's inequality on the lower $\tau N$ data points suffices for a CVaR concentration bound \citep[Theorem 4.2]{brown2007large}.
This is why the bonus of Brown-UCB, which is $\sqrt{\frac{5\tau\log(3/\delta)}{N_k(a)}}$, does not yield optimism when continuity fails.
In general, the $\frac{1}{N_k(a)}$ term in our bonus from Bernstein's inequality is needed for proving optimism in general.

\textbf{Computational Efficiency: }
In \cref{line:mab-compute-eps-optimal-bs}, the objective function $\wh f_k(\cdot,a)$ is concave and unimodal.
So, its optimal value can be efficiently approximated, \eg,
by golden-section search \citep{kiefer1953sequential} or gradient ascent in $1/\eps^2=\Ocal(\tau K/A)$ steps \citep{boyd2004convex}.
Thus, \mabalg{} is \emph{both} minimax-optimal for regret and computationally efficient.

\section{Risk Sensitive RL}
We now shift gears to CVaR RL.
First, we review the augmented MDP framework due to \citet{bauerle2011markov} and derive Bellman equations for our problem \citep{bdr2022}.
Using this perspective, we propose \rlalg{}, a bonus-driven Value Iteration algorithm in the augmented MDP, which we show enjoys strong regret guarantees.

\subsection{Augmented MDP and Bellman Equations}\label{sec:augmented-mdp-bellman-eqs}
For any history-dependent $\pi\in\Pi_\Hcal$, timestep $h\in[H]$, state $s_h\in\Scal$, budget $b_h\in[0,1]$, and history $\Hcal_h$, define
\begin{align*}\textstyle
    V_h^\pi(s_h,b_h;\Hcal_h) &=\textstyle \Eb[\pi]{ \prns{ b_h-\sum_{t=h}^Hr_t }^+ \mid s_h, \Hcal_h}.
\end{align*}
Then, the CVaR RL objective can be formulated as,
\begin{align}
    &\cvar_\tau^\star
    \textstyle= \max_{\pi\in\Pi_\Hcal}\max_{b\in[0,1]}\braces{ b-\tau^{-1}V_1^\pi(s_1,b) } \nonumber
    \\&~~\qquad\textstyle=\max_{b\in[0,1]}\braces{ b-\tau^{-1}\min_{\pi\in\Pi_\Hcal} V_1^\pi(s_1,b) }.\label{eq:cvar_def_v_star}
\end{align}
\citet{bauerle2011markov} showed a remarkable fact about $\min_{\pi\in\Pi_\Hcal}V^\pi_1(s_1,b_1)$:
there exists an optimal policy $\rho^\star=\braces{\rho_h^\star:\Saug\to\Acal}_{h\in[H]}$ that is deterministic and Markov in an augmented MDP, which we now describe.
The augmented state is $(s,b)\in\Saug\coloneqq\Scal\times[0,1]$.
Given any $b_1\in[0,1]$, the initial state is $(s_1,b_1)$.
Then, for each $h=1,2,...,H$, $a_h=\rho^\star_h(s_h,b_h),r_h\sim R(s_h,a_h),s_{h+1}\sim P^\star(s_h,a_h), b_{h+1}=b_h-r_h$.
Intuitively, the extra state $b_h$ is the amount of budget left from the initial $b_1$, and is a sufficient statistic of the history for the CVaR RL problem.
Let $\Piaug$ denote the set of deterministic, Markov policies in the augmented MDP.
Then, we may optimize over this simpler policy class without losing optimality!

Before we formalize the optimality result,
we first derive Bellman equations \citep[as in][Chapter 7.8]{bdr2022}.
For any $\rho\in\Piaug$, overload notation and define
\begin{align*}\textstyle
    V_h^\rho(s_h,b_h) = \Eb[\rho]{ \prns{ b_h-\sum_{t=h}^Hr_t }^+\mid s_h,b_h },
\end{align*}
where $r_h,\dots, r_H$ are the rewards generated by executing $\rho$ in the augmented MDP starting from $s_h,b_h$ at time step $h$.
Observe that $V^\rho_h$ satisfies the Bellman equations,
\begin{align*}
    &V^\rho_h(s_h,b_h) = \Eb[a_h\sim\rho_h(s_h,b_h)]{ U^\rho_h(s_h,b_h,a_h) },
  \\&U^\rho_h(s_h,b_h,a_h) = \Eb[s_{h+1},r_h]{ V^\rho_{h+1}(s_{h+1},b_{h+1})},
\end{align*}
where $s_{h+1}\sim P^\star(s_h,a_h),r_h\sim R(s_h,a_h), b_{h+1}=b_h-r_h$ and $V_{H+1}^\rho(s,b)=b^+$.
Analogously, define $V^\star_h$ and $\rho^\star$ inductively with the Bellman optimality equations,
\begin{align*}
    &V^\star_h(s_h,b_h) = \min_{a\in\Acal}U^\star_h(s_h,b_h,a_h),
    \\&\rho^\star_h(s_h,b_h) = \argmin_{a\in\Acal}U^\star_h(s_h,b_h,a_h),
    \\&U^\star_h(s_h,b_h,a_h) = \Eb[s_{h+1},r_h]{ V^\star_{h+1}(s_{h+1},b_{h+1})},
\end{align*}
where $s_{h+1}\sim P^\star(s_h,a_h),r_h\sim R(s_h,a_h), b_{h+1}=b_h-r_h$ and $V_{H+1}^\star(s,b)=b^+$.
Armed with these definitions, we formalize the optimality result in the following theorem.
\begin{restatable}[Optimality of $\Piaug$]{theorem}{optimalityAugmentedMarkovPolicies}\label{thm:aug_mdp_theorem}
For any $b\in[0,1]$,
\begin{align*}\textstyle
    V_1^\star(s_1,b)=V_1^{\rho^\star}(s_1,b)=\inf_{\pi\in\Pi_\Hcal}V_1^\pi(s_1,b).
\end{align*}
\end{restatable}
This is a known result in the infinite-horizon, discounted setting \citep{bauerle2011markov,bdr2022}.
We provide a proof from first principles for the finite-horizon setting in \cref{sec:augmented-mdp}, by inductively unravelling the Bellman optimality equations.
As a technical remark, we show optimality over history-dependent policies in the augmented MDP with memory, larger than the history-dependent class defined here.

These facts imply that we could compute $V_1^\star$ and $\rho^\star$ using dynamic programming (DP) if we knew the true transitions $P^\star$,
and the procedure is similar to the classic Value Iteration procedure in vanilla RL.
Based on \cref{thm:aug_mdp_theorem} and \cref{eq:cvar_def_v_star}, by executing $\rho^\star$ starting from $(s_1, b^\star)$ with $b^\star := \arg\max_{b\in [0,1]} \braces{ b - \tau^{-1}V_1^\star(s_1,b)}$, we achieve the maximum CVaR value in the original MDP.
Below we leverage this DP perspective on the augmented MDP to design exploration algorithms to solve CVaR RL. %

\begin{algorithm*}[!t]
    \caption{\rlalg{}}
    \label{alg:ucbvi-cvar}
    \begin{algorithmic}[1]
        \State\textbf{Input:} risk tolerance $\tau$, number of episodes $K$, failure probability $\delta$, bonus function $\bon_{h,k}(s,b,a)$.
        \For{episode $k=1,2,\dots,K$}
            \State Compute counts and empirical transition estimate, \label{line:empirical-transition-estimate}
            \begin{align*}
                N_{k}(s,a,s')&=\sum_{h=1}^H\sum_{i=1}^{k-1} \I{\prns{s_{h,i},a_{h,i},s_{h+1,i}} = (s,a,s')},
                \\N_{k}(s,a)&=1\vee \sum_{s'\in\Scal}N_{k}(s,a,s'), \;
                \wh P_{k}(s'\mid s,a) =\frac{N_{k}(s,a,s')}{N_{k}(s,a)},
            \end{align*}
            \State For all $s\in\Scal,b\in[0,1]$, set $\wh V_{H+1,k}^\uparrow(s,b)=\wh V_{H+1,k}^\downarrow(s,b)=b^+$.
            \For{$h=H,H-1,\dots,1$} 
                \State Define pessimistic estimates of $V^\star$ and $\wh\rho^k$, \ie, for all $s,b,a$, \label{line:pessimism-estimates}
                \begin{align*}
                    \wh U_{h,k}^\downarrow(s,b,a)&= \wh P_{k}(s,a)^\top \Eb[r_h\sim R(s,a)]{\wh V_{h+1,k}^\downarrow(\cdot,b-r_h)}-\bon_{h,k}(s,b,a),
                    \\\wh\rho^k_h(s,b) &= \argmin_a\wh U_{h,k}^\downarrow(s,b,a),
                    \qquad \wh V_{h,k}^\downarrow(s,b)=\max\braces{\wh U_{h,k}^\downarrow(s,b,\wh\rho^k_h(s,b)),0}.
                \end{align*}
                \State If using Bernstein bonus (\cref{sec:bernstein-bonus}), also define optimistic estimates for $V^\star$, \ie, for all $s,b,a$, \label{line:optimistic-estimates}
                \begin{align*}
                    \wh U_{h,k}^\uparrow(s,b,a)&= \wh P_{k}(s,a)^\top \Eb[r_h\sim R(s,a)]{\wh V_{h+1,k}^\uparrow(\cdot,b-r_h)} + \bon_{h,k}(s,b,a),
                    \\\wh V_{h,k}^\uparrow(s,b)&=\min\braces{\wh U_{h,k}^\uparrow(s,b,\wh\rho^k_h(s,b)),1}.
                \end{align*}
            \EndFor
            \State Calculate $\wh b_k = \argmax_{b\in[0,1]}\braces{b-\tau^{-1}\wh V_{1,k}^\downarrow(s_1,b)}$.\label{line:hat-b-k}
            \State Collect $\braces{(s_{h,k},a_{h,k},r_{h,k})}_{h\in[H]}$ by rolling in $\wh\rho^{k}$ starting from $(s_1,\wh b_k)$ in the augmented MDP. \label{line:roll-in}
        \EndFor
    \end{algorithmic}
\end{algorithm*}

\subsection{\rlalg{}}
In this section, we introduce our algorithm \rlalg{} (\cref{alg:ucbvi-cvar}), an extension of the classic UCBVI algorithm of \citet{azar2017minimax} which attained the minimax-optimal regret for vanilla RL.
Our contribution is showing that bonus-driven pessimism, which guarantees that the learned $\wh V^\downarrow_{h,k}$ is a high probability lower confidence bound (LCB) on the optimal $V^\star_h$, is sufficient and in fact optimal for CVaR RL.
This remarkably shows that the bonus-driven exploration paradigm from vanilla RL, \ie, ``optimism/pessimism under uncertainty,'' can be used to optimally conduct safe exploration for CVaR RL.

\rlalg{} iterates over $K$ episodes, where the $k$-th episode proceeds as follows.
First, in \cref{line:empirical-transition-estimate}, we compute an empirical estimate $\wh P_{k}$ of the transition dynamics $P^\star$ using the previous episodes' data.
Then, in \cref{line:pessimism-estimates}, we inductively compute $\wh U_{h,k}^\downarrow$ from $h=H$ to $h=1$ by subtracting a bonus that accounts for the error from using $\wh P_{k}$ instead of $P^\star$.
Next, $\wh V_{h,k}^\downarrow$ and $\wh\rho^k$ are computed greedily w.r.t. $\wh U_{h,k}^\downarrow$ to mimic the Bellman optimality equations (\cref{sec:augmented-mdp-bellman-eqs}).
Subtracting the bonus is key to showing $\wh U_{h,k}^\downarrow$ (resp. $\wh V_{h,k}^\downarrow$) is a high probability lower bound of $U^\star_h$ (resp. $V^\star_h$).
Next, in \cref{line:hat-b-k}, we compute $\wh b_k$ using the pessimistic $\wh V_{1,k}^\downarrow$.
Similar to our MAB algorithm, this guarantees that $\wh\cvar_\tau^k := \wh b_k-\tau^{-1}\wh V_{1,k}^\downarrow(s_1,\wh b_k)$ is an \emph{optimistic} estimate of $\cvar_\tau^\star$.
Finally, in \cref{line:roll-in}, we roll in with the learned, augmented policy $\wh\rho^k$ starting
from $\wh b_k$ in the augmented MDP to collect data for the next iterate.
We highlight that in \cref{line:roll-in}, the algorithm is still only interacting with the original MDP described in \cref{sec:problem-setup}.
To roll in with an augmented policy, the algorithm can imagine this augmented MDP by keeping track of the $b_h$ via the update $b_{h+1}=b_h-r_h$.
There is virtually no overhead as it is only a scalar with known transitions.

\subsection{The Hoeffding Bonus}

Two types of bonuses may be used in \rlalg{}: Hoeffding (\cref{eq:hoeffding-bonus-def}) and Bernstein (\cref{eq:bernstein-bonus-def}).
We now show that a simple Hoeffding bonus, defined below, can already provide the best CVaR regret bounds in the current literature:
\begin{align}
    \hobon_{h,k}(s,a) = \sqrt{\frac{L}{N_{k}(s,a)}}, \label{eq:hoeffding-bonus-def}
\end{align}
where $L=\log(HSAK/\delta)$.
\begin{restatable}{theorem}{hoeffingBonusRegret}\label{thm:hoeffding-bonus-regret}
For any $\delta\in(0,1)$, w.p. at least $1-\delta$,
\rlalg{} with the Hoeffding bonus (\cref{eq:hoeffding-bonus-def}) enjoys%
\begin{align*}
    \rlregret_\tau(K)\leq 4e\tau^{-1}\sqrt{SAHK}L + 10e\tau^{-1}S^2AHL^2.
\end{align*}
\end{restatable}
\begin{proofsketch}
The first step is to establish pessimism, \ie, $\wh V_{1,k}^\downarrow\leq V^\star_1$, which implies optimism of $\wh\cvar_\tau^k\geq\cvar_\tau^\star$.
At this point, we cannot apply CVaR concentration as we did for MAB, since $\wh V_{1,k}^\downarrow$ is not the empirical CVaR. Instead, we show that the simulation lemma (\cref{lem:simulation}) extends to the augmented MDP, which gives
$V^{\wh\rho^k}_1(s_1,\wh b_k)-\wh V_{1,k}^\downarrow(s_1,\wh b_k) \leq \Eb[\wh\rho^k,\wh b_k]{ \sum_{h=1}^H2\hobon_{h,k}(s_h,a_h)+\xi_{h,k}(s_h,a_h) }$.
The expectation is over the distribution of rolling in $\wh\rho^k$ from $\wh b_k$,
which is exactly how we explore and collect $s_{h,k},a_{h,k}$.
Thus, we can apply Azuma and elliptical potential to conclude the proof, as in the usual UCBVI analysis. \looseness=-1
\end{proofsketch}

The leading term of the Hoeffding bound is $\wt\Ocal\prns*{\tau^{-1}\sqrt{SAHK}}$,
which is optimal in $S,A,K$.
Notably, it has a $S$ factor improvement over the current best bound $\tau^{-1}\sqrt{S^3AHKL}$ from \citet{bastani2022regret} (we've divided their bound by $H$ to make returns scaling consistent).
While \cref{thm:hoeffding-bonus-regret} is already the tightest in the literature, our lower bound suggests the possibility of removing another $\sqrt{\tau^{-1}H}$.

\subsection{Improved Bounds with the Bernstein Bonus}\label{sec:bernstein-bonus}
Precise design of the exploration bonus is critical to enabling tighter performance bounds, even in vanilla RL \citep{azar2017minimax,zanette2019tighter}.
In this subsection, we propose the Bernstein bonus and prove two tighter regret bounds.
The bonus depends on the sample variance, which recall, for any function $f$, is defined as
$\Varb[s'\sim\wh P_{k}(s,a)]{f(s')}=\wh P_{k}(s,a)^\top\prns{ f(\cdot)-\bar f_N }^2$ with $\bar f_N = \wh P_{k}(s,a)^\top f$ being the sample mean \citep{maurer2009empirical}.
We define the Bernstein bonus as follows,
{\small
\begin{align}
&\bebon_{h,k}(s,b,a) = \sqrt{ \frac{2\Varb[s'\sim\wh P_{k}(s,a)]{ \Eb[r_h]{\wh V^\downarrow_{h+1,k}(s',b')} }L}{N_{k}(s,a)} } \nonumber
\\&+\sqrt{ \frac{2\Eb[s'\sim\wh P_{k}(s,a),r_h]{ \prns{ \wh V^\uparrow_{h+1,k}(s',b')-\wh V^\downarrow_{h+1,k}(s',b') }^2 }L}{N_{k}(s,a)} } \nonumber
\\&+\frac{L}{N_{k}(s,a)}, \text{  where } b' = b-r_h, \text{ and } r_h\sim R(s,a). \label{eq:bernstein-bonus-def}
\end{align}}\looseness=-1
When using the Bernstein bonus, \cref{line:optimistic-estimates} should be activated to compute
optimistic estimates $\wh V^\uparrow_{h,k}$ by adding the bonus.
Together, $(\wh V^\downarrow_{h,k},\wh V^\uparrow_{h,k})$ forms a tight confidence band around $V^\star_h$,
which we will use to inductively prove pessimism and optimism at all $h\in[H]$ \citep[as in][]{zanette2019tighter}.
Compared to \citet{zanette2019tighter}, our Bernstein bonus also depends on the state augmentation $b$, since our UCBVI procedure is running in the augmented MDP.
We now prove the first Bernstein bound, which tightens the Hoeffding bound by a $\sqrt{H}$ factor.
\begin{restatable}{theorem}{bernsteinBonusRegret}\label{thm:bernstein-bonus-regret}
For any $\delta\in(0,1)$, w.p. at least $1-\delta$, \rlalg{} with the Bernstein bonus (\cref{eq:bernstein-bonus-def}) enjoys a regret guarantee of
\begin{align*}
\rlregret_\tau(K)\leq 10e\tau^{-1}\sqrt{SAK}L + \tau^{-1}\xi,
\end{align*}
where $\xi\in\wt\Ocal\prns*{SAHK^{1/4} + S^2AH}$ is a lower order term.
\end{restatable}
\begin{proofsketch}
We first establish pessimism and optimism of $\wh V^\downarrow_{h,k}$ and $\wh V^\uparrow_{h,k}$, similar to \citet{zanette2019tighter}.
Then, apply Simulation lemma as in \cref{thm:hoeffding-bonus-regret}.
The key observation is that the $H$ sample variances from the bonus,
\ie, $\sum_{h=1}^H\Varb[s'\sim\wh P_{k}(s,a)]{ \EE_{r_h}[\wh V^\downarrow_{h+1,k}(s',b-r_h)] }$,
can be reduced to a \emph{single} variance $\Varb[\wh\rho^k,\wh b_k]{ \prns*{\wh b_k-\sum_{h=1}^Hr_h}^+}$, which we bound by $1$.
To do so, we show that \citet{azar2017minimax}'s Law of Total Variance technique
also applies to our $V^{\wh\rho^k}$, which, unlike the value function of vanilla RL, depends on the state augmentation $b$.
\end{proofsketch}
The leading term of \cref{thm:bernstein-bonus-regret} is $\tau^{-1}\sqrt{SAK}$, and improves over \citet{bastani2022regret} by a $S\sqrt{H}$ factor.
Up to log terms, this matches our lower bound in all parameters except $\tau$, which implies \rlalg{} is minimax-optimal for a constant $\tau$.
In particular, $\tau=1$ recovers the risk-neutral vanilla RL setting, where \rlalg{} matches the minimax result \citep{azar2017minimax}.
To get the optimal $\tau^{-1/2}$ scaling (\cref{cor:rl-lower-bound}),
we cannot loosely bound each variance term by $1$, as they should scale as $\tau$ if $\wh b_k$
approximates the $\tau$-th quantile of $R(\wh\rho^k,\wh b_k)$. We show this is indeed the case under a continuity assumption.
\begin{assumption}\label{ass:continuous-densities}
For all $\rho\in\Piaug$ and $b_1\in[0,1]$, the returns of rolling in $\rho$ from $b_1$, \ie, $R(\rho,b_1)$, is continuously distributed with a density lower bounded by $\densitymin$.
\end{assumption}
\begin{restatable}{theorem}{bernsteinBonusRegretTwo}\label{thm:bernstein-bonus-regret2}
Under \cref{ass:continuous-densities}, the bound in \cref{thm:bernstein-bonus-regret} can be refined to,
\begin{align*}
\rlregret_\tau(K)\leq 12e\sqrt{\tau^{-1}SAK}L + \tau^{-1}\densitymin^{-1/2}\xi.
\end{align*}
\end{restatable}
\begin{proofsketch}
The only divergence from \cref{thm:bernstein-bonus-regret} is how we bound
$\Varb[\wh\rho^k,\wh b_k]{ \prns*{\wh b_k-\sum_{h=1}^Hr_h}^+}$.
Since the density of $R(\wh\rho^k,\wh b_k)$ is lower bounded,
the CVaR objective $f(b)=b-\tau^{-1}\Eb[\wh\rho^k,\wh b_k]{\prns*{b-\sum_{h=1}^Hr_h}^+}$ is \emph{strongly concave}.
This implies that $\wh b_k$ approximates the true $\tau$-th quantile $b_k^\star=\argmax_{b\in[0,1]}f(b)$,
\ie, $\frac{\densitymin}{2}\prns*{\wh b_k-b_k^\star}^2\leq \tau \prns*{f(b^\star_k)-f(\wh b_k)}\leq V_1^{\wh\rho^k}(s_1,\wh b_k)-\wh V_{1,k}^\downarrow(s_1,\wh b_k)$.
Leveraging this fact, we show \\$\Varb[\wh\rho^k,\wh b_k]{ \prns*{\wh b_k-\sum_{h=1}^Hr_h}^+}\leq 2\tau + 4\densitymin^{-1}\prns*{V_1^{\wh\rho^k}(s_1,\wh b_k)-\wh V_{1,k}^\downarrow(s_1,\wh b_k)}$,
which notably scales with $\tau$.
We conclude the proof by showing the error term, \ie, $\sum_{k=1}^K\prns*{V_1^{\wh\rho^k}(s_1,\wh b_k)-\wh V_{1,k}^\downarrow(s_1,\wh b_k)}\in\wt\Ocal(\sqrt{SAHK})$, is lower order.
\end{proofsketch}
The leading term of \cref{thm:bernstein-bonus-regret2} is $\sqrt{\tau^{-1}SAK}$,
which matches our lower bound in \cref{cor:rl-lower-bound} and establishes the \emph{optimality} of
\rlalg{} for return distributions satisfying \cref{ass:continuous-densities}.
Notably, $\densitymin$ only multiplies with the lower order term $\xi$.
This result highlights the importance of the Bernstein bonus for CVaR RL
 -- it improves the regret bound of the Hoeffding bonus by $\sqrt{\tau^{-1}H}$, whereas in vanilla RL, the improvement is $\sqrt{H}$.

With regards to \cref{ass:continuous-densities}, lower bounded densities, \ie, strong monotonicity of the CDF, is standard for identifying the quantile \citep{ma2021conservative}.
In fact, PAC results for quantile identification is not possible without some control of the mass at the quantile.
As a simple example, consider estimating the $0.5$-th quantile using $N$ i.i.d. datapoints sampled from $\ber(0.5)$.
The correct answer is $0$, but by symmetry, the sample median is always distributed as $\ber(0.5)$ for any $N$. So we always have a $0.5$-probability of being incorrect. We provide an information theoretic lower bound to rule out all estimators -- not just the sample median -- in \cref{thm:quantile-estimation-minimax}.

It nonetheless remains an open question whether \cref{ass:continuous-densities} can be removed by eschewing identifying the quantile.
In MABs \cref{thm:bandit upper}, we circumvented the need to identify quantiles by decomposing the regret into (1) the sum of bonuses, plus, (2) the difference between the empirical and true CVaRs, both of which can be shown to have the correct $\tau^{-1/2}$ scaling.
An analogous approach for RL is to decompose the regret into (1) $\sum_{h,k}\Eb[\wh\rho^k,\wh b_k,\wh P_k]{\bebon_{h,k}(s_{h,k},b_{h,k},a_{h,k})}$, plus, (2) $\sum_k \cvar_\tau(\wh\rho^k,\wh b_k;\wh P_k)-\cvar_\tau(\wh\rho^k,\wh b_k)$.
However, it is unclear if both terms can be unconditionally bounded by $\wt\Ocal(\sqrt{\tau^{-1}SAK})$.

\textbf{Remark on $b$-dependence: }
Although \rlalg{} operates in the augmented MDP, our Hoeffding bonus has no dependence on the budget state $b$ and matches the Hoeffding bonus of UCBVI (from vanilla RL; \citealp{azar2017minimax}).
Intuitively, this is possible since the dynamics of $b$ are known (we assume known reward distribution), so there is no need to explore in the $b$-dimension.
In contrast to the Bernstein bonus of UCBVI, our Bernstein bonus depends on $b$ and captures the variance of $\prns{b-R(\wh\rho^k,\wh b_k)}^+$.
This is crucial for obtaining the correct $\tau$ rate.

\section{Computational Efficiency via Discretization}\label{sec:discretized-mdp-computational-efficiency}
Previously, we assumed each line of \cref{alg:ucbvi-cvar} was computed exactly.
This is not computationally feasible since the dynamic programming (DP) step (\cref{line:pessimism-estimates,line:optimistic-estimates})
needs to be done over all $b\in[0,1]$ and the calculation for $\wh b_k$ (\cref{line:hat-b-k}) involves maximizing a non-concave function.
Following \citet{bastani2022regret}, we propose to discretize the rewards so the aforementioned steps need only be performed over a finite grid. Thus, we gain computational efficiency while maintaining the same statistical guarantees.

Fix a precision $\eta\in(0,1)$, define $\phi(r) = \eta\ceil{r/\eta}\wedge 1$, which \emph{rounds-up} $r\in[0,1]$ to an $\eta$-net of $[0,1]$, henceforth referred to as ``the grid''.
The discretized MDP $\disc(\Mcal)$ is an exact replica of the true MDP $\Mcal$ with one exception: its rewards are post-processed with $\phi$, \ie, $R(s,a;\disc(\Mcal))=R(s,a;\Mcal)\circ\phi^{-1}$, where $\circ$ denotes composition.

In $\disc(\Mcal)$, the $\tau$-th quantile of the returns distribution (the argmax of the CVaR objective)
will be a multiple of $\eta$, so it suffices to compute $\wh V^\downarrow_1(s_1,b)$ and maximize \cref{line:hat-b-k} over the grid.
Since $b$ transitions by subtracting rewards, which are multiples of $\eta$, $b_h$ will always stay on the grid.
Hence, the entire DP procedure (\cref{line:pessimism-estimates,line:optimistic-estimates}) only needs to occur on the grid.
In \cref{sec:appendix-discrete-computational-complexity}, we show \rlalg{} has a runtime of $\Ocal(S^2\eta^{-2}AHK)$ in the discretized MDP.
It's worth clarifying that \rlalg{} is still interacting with $\Mcal$, except that it internally discretizes the received rewards to simulate running in the $\disc(\Mcal)$ for computation's sake.
Thus, we still want to compete with the strongest CVaR policy in the true MDP; we've just made our algorithm weaker by restricting it to run in an imagined $\disc(\Mcal)$.

Now, we show that the true regret only increases by $\Ocal(K\eta)$, which can be made lower order by setting $\eta=K^{-1/2}$.
\cref{thm:hoeffding-bonus-regret,thm:bernstein-bonus-regret} made no assumptions on the reward distribution,
so they immediately apply to bound the $\disc(\Mcal)$ regret, \ie, $\rlregret_\tau(K;\disc(\Mcal))=\sum_{k=1}^K\cvar_\tau^\star(\disc(\Mcal))-\cvar_\tau(\wh\rho^k,\wh b_k;\disc(\Mcal))$.
We translate regret in $\disc(\Mcal)$ to regret in $\Mcal$ via a coupling argument, inspired by \citet{bastani2022regret}.
Let $Z_{\pi,\Mcal}$ denote the returns from running $\pi$ in $\Mcal$.
For random variables $X,Y$, we say $Y$ stochastically dominates $X$, denoted $X\preceq Y$, if $\forall t\in\RR: \Pr(Y\leq t)\leq \Pr(X\leq t)$.
Then, for any $\rho\in\Piaug,b_1\in[0,1]$, we show two facts:
\begin{enumerate}[leftmargin=1em]
\item[\textbf{F1}] Running $\rho,b_1$ in the imagined $\disc(\Mcal)$ is equivalent to running a reward-history-dependent policy,
\begin{align*}
\adapted(\rho,b_1)_h(s_h,r_{1:h-1})=\rho_h(s_h,b_1-\phi(r_1)-...-\phi(r_{h-1})).    
\end{align*}
Also, $Z_{\rho,b_1,\disc(\Mcal)}-H\eta\preceq Z_{\adapted(\rho,b_1),\Mcal}\preceq Z_{\rho,b_1,\disc(\Mcal)}$.
\item[\textbf{F2}] There exists a memory-history-dependent\footnote{The memory-MDP model is novel and key to our coupling argument. In \cref{sec:augmented-mdp}, we define this model and show that $\Piaug$ still contains the optimal policy over this seemingly larger class of memory-history-dependent policies, \ie, \cref{thm:aug_mdp_theorem} holds. }
policy $\disc(\rho,b)$ such that $Z_{\rho,b,\Mcal}\preceq Z_{\disc(\rho,b),\disc(\Mcal)}$.
Intuitively, when running in $\disc(\Mcal)$, once the discretized reward $r_h$ is seen,
a memory $m_h$ is generated from the conditional reward distribution of rewards that get rounded-up to $r_h$.
Thus, $m_h$ is essentially sampled from the unconditional reward distribution.
The memory-dependent policy $\disc(\rho,b)$ makes use of these samples to mimic running $\rho,b$ in $\Mcal$.
\end{enumerate}
\textbf{F1} implies $\cvar_\tau(\adapted(\wh\rho^k,\wh b_k);\Mcal)\geq \cvar_\tau(\wh\rho^k,\wh b_k;\disc(\Mcal))-\tau^{-1}H\eta$.
\textbf{F2} implies $\cvar_\tau^\star(\Mcal)\leq \cvar_\tau^\star(\disc(\Mcal))$.
Combining these two facts, we have $\rlregret_\tau(K;\Mcal)\leq \rlregret_\tau(K;\disc(\Mcal)) + K\tau^{-1}H\eta$.

Translating \cref{thm:bernstein-bonus-regret2} requires more care, as its proof relied on continuously distributed returns (\cref{ass:continuous-densities}) which is untrue in $\disc(\Mcal)$.
We show that we only need the \emph{true} returns distribution to be continuous.
\begin{assumption}\label{ass:continuous-densities-discretized}
For all $\rho\in\Piaug$ and $b_1\in[0,1]$, the returns distribution of $\adapted(\rho,b_1)$ in $\Mcal$ is continuous, with a density lower bounded by $\densitymin$.
\end{assumption}
With this premise, we can prove \cref{thm:bernstein-bonus-regret2} for $\disc(\Mcal)$, with an extra term of $\wt\Ocal\prns*{\tau^{-1}\sqrt{\densitymin^{-1} SAHK\eta}}$.
In sum, setting $\eta=1/\sqrt{K}$ ensures that \rlalg{} is \emph{both} near-minimax-optimal for regret and computationally efficient, with a runtime of $\Ocal\prns*{S^2AHK^2}$.
We note the superlinear-in-$K$ runtime from discretization is not even avoidable in Lipschitz bandits \citep{wang2020towards},
and we leave developing more scalable methods for future work.

\section{Concluding Remarks}
In this paper, we presented a more complete picture of risk-sensitive MAB and RL with CVaR by providing
not only novel lower bounds but also procedures and analyses that both improve on the state of the art and match our lower bounds.
One exception where a gap remains is CVaR RL with discontinuous returns and a risk tolerance that is not constant (or, not lower bounded);
in this case, our lower and upper bounds differ by a factor of $\sqrt{\tau}$.
We discuss the feasibility of closing this gap in \cref{sec:bernstein-bonus}.

A direction for future work is to develop algorithms with optimal regret guarantees for more general risk measures, \eg, optimized certainty equivalent (OCE) \citep{ben2007old}.
Another orthogonal direction is to extend our results beyond tabular MDPs. We believe that our techniques in this work are already enough for linear MDPs \citep{jin2020provably} where the transition kernel is linear in some \emph{known} feature space. However, extending the results beyond linear models, such as to low-rank MDPs \citep{agarwal2020flambe,uehara2021representation} and block MDPs \citep{misra2020kinematic,zhang2022efficient} remains a challenge due to the fact that achieving point-wise optimism is harder when nonlinear function approximation is used.

\newpage
\clearpage
\bibliography{paper}

\begin{thebibliography}{58}
\providecommand{\natexlab}[1]{#1}
\providecommand{\url}[1]{\texttt{#1}}
\expandafter\ifx\csname urlstyle\endcsname\relax
  \providecommand{\doi}[1]{doi: #1}\else
  \providecommand{\doi}{doi: \begingroup \urlstyle{rm}\Url}\fi

\bibitem[Acerbi \& Tasche(2002)Acerbi and Tasche]{acerbi2002coherence}
Acerbi, C. and Tasche, D.
\newblock On the coherence of expected shortfall.
\newblock \emph{Journal of Banking \& Finance}, 26\penalty0 (7):\penalty0
  1487--1503, 2002.

\bibitem[Agarwal et~al.(2020)Agarwal, Kakade, Krishnamurthy, and
  Sun]{agarwal2020flambe}
Agarwal, A., Kakade, S., Krishnamurthy, A., and Sun, W.
\newblock Flambe: Structural complexity and representation learning of low rank
  mdps.
\newblock \emph{Advances in neural information processing systems},
  33:\penalty0 20095--20107, 2020.

\bibitem[Agarwal et~al.(2021)Agarwal, Jiang, Kakade, and Sun]{rltheorybookAJKS}
Agarwal, A., Jiang, N., Kakade, S.~M., and Sun, W.
\newblock \emph{Reinforcement learning: Theory and algorithms}.
\newblock 2021.
\newblock \url{https://rltheorybook.github.io/}.

\bibitem[Auer et~al.(2002)Auer, Cesa-Bianchi, and Fischer]{auer2002finite}
Auer, P., Cesa-Bianchi, N., and Fischer, P.
\newblock Finite-time analysis of the multiarmed bandit problem.
\newblock \emph{Machine learning}, 47\penalty0 (2):\penalty0 235--256, 2002.

\bibitem[Auer et~al.(2008)Auer, Jaksch, and Ortner]{auer2008near}
Auer, P., Jaksch, T., and Ortner, R.
\newblock Near-optimal regret bounds for reinforcement learning.
\newblock \emph{Advances in neural information processing systems}, 21, 2008.

\bibitem[Azar et~al.(2017)Azar, Osband, and Munos]{azar2017minimax}
Azar, M.~G., Osband, I., and Munos, R.
\newblock Minimax regret bounds for reinforcement learning.
\newblock In \emph{International Conference on Machine Learning}, pp.\
  263--272. PMLR, 2017.

\bibitem[Bastani et~al.(2022)Bastani, Ma, Shen, and Xu]{bastani2022regret}
Bastani, O., Ma, Y.~J., Shen, E., and Xu, W.
\newblock Regret bounds for risk-sensitive reinforcement learning.
\newblock In Oh, A.~H., Agarwal, A., Belgrave, D., and Cho, K. (eds.),
  \emph{Advances in Neural Information Processing Systems}, 2022.
\newblock URL \url{https://openreview.net/forum?id=yJEUDfzsTX7}.

\bibitem[B{\"a}uerle \& Ott(2011)B{\"a}uerle and Ott]{bauerle2011markov}
B{\"a}uerle, N. and Ott, J.
\newblock Markov decision processes with average-value-at-risk criteria.
\newblock \emph{Mathematical Methods of Operations Research}, 74\penalty0
  (3):\penalty0 361--379, 2011.

\bibitem[Bellemare et~al.(2017)Bellemare, Dabney, and
  Munos]{bellemare2017distributional}
Bellemare, M.~G., Dabney, W., and Munos, R.
\newblock A distributional perspective on reinforcement learning.
\newblock In \emph{International Conference on Machine Learning}, pp.\
  449--458. PMLR, 2017.

\bibitem[Bellemare et~al.(2023)Bellemare, Dabney, and Rowland]{bdr2022}
Bellemare, M.~G., Dabney, W., and Rowland, M.
\newblock \emph{Distributional Reinforcement Learning}.
\newblock MIT Press, 2023.
\newblock \url{http://www.distributional-rl.org}.

\bibitem[Ben-Tal \& Teboulle(2007)Ben-Tal and Teboulle]{ben2007old}
Ben-Tal, A. and Teboulle, M.
\newblock An old-new concept of convex risk measures: the optimized certainty
  equivalent.
\newblock \emph{Mathematical Finance}, 17\penalty0 (3):\penalty0 449--476,
  2007.

\bibitem[Bibaut et~al.(2021)Bibaut, Kallus, Dimakopoulou, Chambaz, and van~der
  Laan]{bibaut2021risk}
Bibaut, A., Kallus, N., Dimakopoulou, M., Chambaz, A., and van~der Laan, M.
\newblock Risk minimization from adaptively collected data: Guarantees for
  supervised and policy learning.
\newblock \emph{Advances in Neural Information Processing Systems},
  34:\penalty0 19261--19273, 2021.

\bibitem[Bottou et~al.(2013)Bottou, Peters, Qui{\~n}onero-Candela, Charles,
  Chickering, Portugaly, Ray, Simard, and Snelson]{bottou2013counterfactual}
Bottou, L., Peters, J., Qui{\~n}onero-Candela, J., Charles, D.~X., Chickering,
  D.~M., Portugaly, E., Ray, D., Simard, P., and Snelson, E.
\newblock Counterfactual reasoning and learning systems: The example of
  computational advertising.
\newblock \emph{Journal of Machine Learning Research}, 14\penalty0 (11), 2013.

\bibitem[Boucheron et~al.(2013)Boucheron, Lugosi, and
  Massart]{boucheron2013concentration}
Boucheron, S., Lugosi, G., and Massart, P.
\newblock \emph{Concentration inequalities: A nonasymptotic theory of
  independence}.
\newblock Oxford university press, 2013.

\bibitem[Boyd et~al.(2004)Boyd, Boyd, and Vandenberghe]{boyd2004convex}
Boyd, S., Boyd, S.~P., and Vandenberghe, L.
\newblock \emph{Convex optimization}.
\newblock Cambridge university press, 2004.

\bibitem[Brown(2007)]{brown2007large}
Brown, D.~B.
\newblock Large deviations bounds for estimating conditional value-at-risk.
\newblock \emph{Operations Research Letters}, 35\penalty0 (6):\penalty0
  722--730, 2007.

\bibitem[Chen()]{orderstatsnotes}
Chen, H.
\newblock Chapter 2. order statistics.
\newblock
  \url{http://www.math.ntu.edu.tw/~hchen/teaching/LargeSample/notes/noteorder.pdf}.

\bibitem[Chow \& Ghavamzadeh(2014)Chow and Ghavamzadeh]{chow2014algorithms}
Chow, Y. and Ghavamzadeh, M.
\newblock Algorithms for cvar optimization in mdps.
\newblock \emph{Advances in neural information processing systems}, 27, 2014.

\bibitem[Chow et~al.(2015)Chow, Tamar, Mannor, and Pavone]{chow2015risk}
Chow, Y., Tamar, A., Mannor, S., and Pavone, M.
\newblock Risk-sensitive and robust decision-making: a cvar optimization
  approach.
\newblock \emph{Advances in neural information processing systems}, 28, 2015.

\bibitem[Du et~al.(2022)Du, Wang, and Huang]{du2022risk}
Du, Y., Wang, S., and Huang, L.
\newblock Risk-sensitive reinforcement learning: Iterated cvar and the worst
  path.
\newblock \emph{arXiv preprint arXiv:2206.02678}, 2022.

\bibitem[Fei et~al.(2020)Fei, Yang, Chen, Wang, and Xie]{fei2020risk}
Fei, Y., Yang, Z., Chen, Y., Wang, Z., and Xie, Q.
\newblock Risk-sensitive reinforcement learning: Near-optimal risk-sample
  tradeoff in regret.
\newblock \emph{Advances in Neural Information Processing Systems},
  33:\penalty0 22384--22395, 2020.

\bibitem[Fei et~al.(2021)Fei, Yang, Chen, and Wang]{fei2021exponential}
Fei, Y., Yang, Z., Chen, Y., and Wang, Z.
\newblock Exponential bellman equation and improved regret bounds for
  risk-sensitive reinforcement learning.
\newblock \emph{Advances in Neural Information Processing Systems},
  34:\penalty0 20436--20446, 2021.

\bibitem[Filippi et~al.(2020)Filippi, Guastaroba, and
  Speranza]{filippi2020conditional}
Filippi, C., Guastaroba, G., and Speranza, M.~G.
\newblock Conditional value-at-risk beyond finance: a survey.
\newblock \emph{International Transactions in Operational Research},
  27\penalty0 (3):\penalty0 1277--1319, 2020.

\bibitem[Jiang \& Agarwal(2018)Jiang and Agarwal]{jiang2018open}
Jiang, N. and Agarwal, A.
\newblock Open problem: The dependence of sample complexity lower bounds on
  planning horizon.
\newblock In \emph{Conference On Learning Theory}, pp.\  3395--3398. PMLR,
  2018.

\bibitem[Jin et~al.(2020)Jin, Yang, Wang, and Jordan]{jin2020provably}
Jin, C., Yang, Z., Wang, Z., and Jordan, M.~I.
\newblock Provably efficient reinforcement learning with linear function
  approximation.
\newblock In \emph{Conference on Learning Theory}, pp.\  2137--2143. PMLR,
  2020.

\bibitem[Kagrecha et~al.(2019)Kagrecha, Nair, and
  Jagannathan]{kagrecha2019distribution}
Kagrecha, A., Nair, J., and Jagannathan, K.
\newblock Distribution oblivious, risk-aware algorithms for multi-armed bandits
  with unbounded rewards.
\newblock \emph{arXiv preprint arXiv:1906.00569}, 2019.

\bibitem[Kallus et~al.(2022)Kallus, Mao, Wang, and Zhou]{kallus2022doubly}
Kallus, N., Mao, X., Wang, K., and Zhou, Z.
\newblock Doubly robust distributionally robust off-policy evaluation and
  learning.
\newblock In \emph{International Conference on Machine Learning}, 2022.

\bibitem[Keramati et~al.(2020)Keramati, Dann, Tamkin, and
  Brunskill]{keramati2020being}
Keramati, R., Dann, C., Tamkin, A., and Brunskill, E.
\newblock Being optimistic to be conservative: Quickly learning a cvar policy.
\newblock In \emph{Proceedings of the AAAI Conference on Artificial
  Intelligence}, volume~34, pp.\  4436--4443, 2020.

\bibitem[Kiefer(1953)]{kiefer1953sequential}
Kiefer, J.
\newblock Sequential minimax search for a maximum.
\newblock \emph{Proceedings of the American mathematical society}, 4\penalty0
  (3):\penalty0 502--506, 1953.

\bibitem[Kisiala(2015)]{kisiala2015conditional}
Kisiala, J.
\newblock Conditional value-at-risk: Theory and applications.
\newblock \emph{arXiv preprint arXiv:1511.00140}, 2015.

\bibitem[Lam et~al.(2023)Lam, Verma, Low, and Jaillet]{lam2023riskaware}
Lam, T., Verma, A., Low, B. K.~H., and Jaillet, P.
\newblock Risk-aware reinforcement learning with coherent risk measures and
  non-linear function approximation.
\newblock In \emph{The Eleventh International Conference on Learning
  Representations}, 2023.
\newblock URL \url{https://openreview.net/forum?id=-RwZOVybbj}.

\bibitem[Lattimore \& Szepesv{\'a}ri(2020)Lattimore and
  Szepesv{\'a}ri]{lattimore2020bandit}
Lattimore, T. and Szepesv{\'a}ri, C.
\newblock \emph{Bandit algorithms}.
\newblock Cambridge University Press, 2020.

\bibitem[Liang \& Luo(2022)Liang and Luo]{liang2022bridging}
Liang, H. and Luo, Z.-Q.
\newblock Bridging distributional and risk-sensitive reinforcement learning
  with provable regret bounds.
\newblock \emph{arXiv preprint arXiv:2210.14051}, 2022.

\bibitem[Ma et~al.(2021)Ma, Jayaraman, and Bastani]{ma2021conservative}
Ma, Y., Jayaraman, D., and Bastani, O.
\newblock Conservative offline distributional reinforcement learning.
\newblock \emph{Advances in Neural Information Processing Systems},
  34:\penalty0 19235--19247, 2021.

\bibitem[Maurer \& Pontil(2009)Maurer and Pontil]{maurer2009empirical}
Maurer, A. and Pontil, M.
\newblock Empirical bernstein bounds and sample variance penalization.
\newblock \emph{arXiv preprint arXiv:0907.3740}, 2009.

\bibitem[Misra et~al.(2020)Misra, Henaff, Krishnamurthy, and
  Langford]{misra2020kinematic}
Misra, D., Henaff, M., Krishnamurthy, A., and Langford, J.
\newblock Kinematic state abstraction and provably efficient rich-observation
  reinforcement learning.
\newblock In \emph{International conference on machine learning}, pp.\
  6961--6971. PMLR, 2020.

\bibitem[Murphy(2003)]{murphy2003optimal}
Murphy, S.~A.
\newblock Optimal dynamic treatment regimes.
\newblock \emph{Journal of the Royal Statistical Society: Series B (Statistical
  Methodology)}, 65\penalty0 (2):\penalty0 331--355, 2003.

\bibitem[Panaganti et~al.(2022)Panaganti, Xu, Kalathil, and
  Ghavamzadeh]{panaganti2022robust}
Panaganti, K., Xu, Z., Kalathil, D., and Ghavamzadeh, M.
\newblock Robust reinforcement learning using offline data.
\newblock In \emph{Advances in neural information processing systems}, 2022.

\bibitem[Puterman(2014)]{puterman2014markov}
Puterman, M.~L.
\newblock \emph{Markov decision processes: discrete stochastic dynamic
  programming}.
\newblock John Wiley \& Sons, 2014.

\bibitem[Rajeswaran et~al.(2017)Rajeswaran, Ghotra, Ravindran, and
  Levine]{rajeswaran2017epopt}
Rajeswaran, A., Ghotra, S., Ravindran, B., and Levine, S.
\newblock {EPO}pt: Learning robust neural network policies using model
  ensembles.
\newblock In \emph{International Conference on Learning Representations}, 2017.
\newblock URL \url{https://openreview.net/forum?id=SyWvgP5el}.

\bibitem[Rockafellar \& Uryasev(2000)Rockafellar and
  Uryasev]{rockafellar2000optimization}
Rockafellar, R.~T. and Uryasev, S.
\newblock Optimization of conditional value-at-risk.
\newblock \emph{Journal of risk}, 2:\penalty0 21--42, 2000.

\bibitem[Si et~al.(2020)Si, Zhang, Zhou, and Blanchet]{si2020distributional}
Si, N., Zhang, F., Zhou, Z., and Blanchet, J.
\newblock Distributional robust batch contextual bandits.
\newblock \emph{arXiv preprint arXiv:2006.05630}, 2020.

\bibitem[Singla et~al.(2021)Singla, Rafferty, Radanovic, and
  Heffernan]{singla2021reinforcement}
Singla, A., Rafferty, A.~N., Radanovic, G., and Heffernan, N.~T.
\newblock Reinforcement learning for education: Opportunities and challenges.
\newblock \emph{arXiv preprint arXiv:2107.08828}, 2021.

\bibitem[Slivkins et~al.(2019)]{slivkins2019introduction}
Slivkins, A. et~al.
\newblock Introduction to multi-armed bandits.
\newblock \emph{Foundations and Trends{\textregistered} in Machine Learning},
  12\penalty0 (1-2):\penalty0 1--286, 2019.

\bibitem[Sutton \& Barto(2018)Sutton and Barto]{sutton2018reinforcement}
Sutton, R.~S. and Barto, A.~G.
\newblock \emph{Reinforcement learning: An introduction}.
\newblock MIT press, 2018.

\bibitem[Tamar et~al.(2015)Tamar, Glassner, and Mannor]{tamar2015optimizing}
Tamar, A., Glassner, Y., and Mannor, S.
\newblock Optimizing the cvar via sampling.
\newblock In \emph{Twenty-Ninth AAAI Conference on Artificial Intelligence},
  2015.

\bibitem[Tamkin et~al.(2019)Tamkin, Keramati, Dann, and
  Brunskill]{tamkin2019distributionally}
Tamkin, A., Keramati, R., Dann, C., and Brunskill, E.
\newblock Distributionally-aware exploration for cvar bandits.
\newblock In \emph{NeurIPS 2019 Workshop on Safety and Robustness on Decision
  Making}, 2019.

\bibitem[Uehara et~al.(2022)Uehara, Zhang, and Sun]{uehara2021representation}
Uehara, M., Zhang, X., and Sun, W.
\newblock Representation learning for online and offline {RL} in low-rank
  {MDP}s.
\newblock In \emph{ICLR}, 2022.
\newblock URL \url{https://openreview.net/forum?id=J4iSIR9fhY0}.

\bibitem[Urp{\'\i} et~al.(2021)Urp{\'\i}, Curi, and Krause]{urpi2021risk}
Urp{\'\i}, N.~A., Curi, S., and Krause, A.
\newblock Risk-averse offline reinforcement learning.
\newblock 2021.

\bibitem[Van~de Geer(2000)]{van2000empirical}
Van~de Geer, S.
\newblock \emph{Empirical Processes in M-estimation}, volume~6.
\newblock Cambridge university press, 2000.

\bibitem[Wang et~al.(2020)Wang, Ye, Geng, and Rudin]{wang2020towards}
Wang, T., Ye, W., Geng, D., and Rudin, C.
\newblock Towards practical lipschitz bandits.
\newblock In \emph{Proceedings of the 2020 ACM-IMS on Foundations of Data
  Science Conference}, pp.\  129--138, 2020.

\bibitem[Wang \& Gao(2010)Wang and Gao]{wang2010deviation}
Wang, Y. and Gao, F.
\newblock Deviation inequalities for an estimator of the conditional
  value-at-risk.
\newblock \emph{Operations Research Letters}, 38\penalty0 (3):\penalty0
  236--239, 2010.

\bibitem[Wiesemann et~al.(2013)Wiesemann, Kuhn, and
  Rustem]{wiesemann2013robust}
Wiesemann, W., Kuhn, D., and Rustem, B.
\newblock Robust markov decision processes.
\newblock \emph{Mathematics of Operations Research}, 38\penalty0 (1):\penalty0
  153--183, 2013.

\bibitem[Zanette \& Brunskill(2019)Zanette and Brunskill]{zanette2019tighter}
Zanette, A. and Brunskill, E.
\newblock Tighter problem-dependent regret bounds in reinforcement learning
  without domain knowledge using value function bounds.
\newblock In \emph{International Conference on Machine Learning}, pp.\
  7304--7312. PMLR, 2019.

\bibitem[Zhang(2023)]{tongzhangbook}
Zhang, T.
\newblock \emph{Mathematical Analysis of Machine Learning Algorithms}.
\newblock 2023.
\newblock \url{http://www.tongzhang-ml.org/lt-book.html}.

\bibitem[Zhang et~al.(2022)Zhang, Song, Uehara, Wang, Agarwal, and
  Sun]{zhang2022efficient}
Zhang, X., Song, Y., Uehara, M., Wang, M., Agarwal, A., and Sun, W.
\newblock Efficient reinforcement learning in block mdps: A model-free
  representation learning approach.
\newblock In \emph{International Conference on Machine Learning}, pp.\
  26517--26547. PMLR, 2022.

\bibitem[Zhou et~al.(2021{\natexlab{a}})Zhou, Gu, and
  Szepesvari]{zhou2021nearly}
Zhou, D., Gu, Q., and Szepesvari, C.
\newblock Nearly minimax optimal reinforcement learning for linear mixture
  markov decision processes.
\newblock In \emph{Conference on Learning Theory}, pp.\  4532--4576. PMLR,
  2021{\natexlab{a}}.

\bibitem[Zhou et~al.(2021{\natexlab{b}})Zhou, Zhou, Bai, Qiu, Blanchet, and
  Glynn]{zhou2021finite}
Zhou, Z., Zhou, Z., Bai, Q., Qiu, L., Blanchet, J., and Glynn, P.
\newblock Finite-sample regret bound for distributionally robust offline
  tabular reinforcement learning.
\newblock In \emph{International Conference on Artificial Intelligence and
  Statistics}, pp.\  3331--3339. PMLR, 2021{\natexlab{b}}.

\end{thebibliography}
\bibliographystyle{icml2023}

\newpage
\appendix
\onecolumn

\begin{center}\LARGE
\textbf{Appendices}
\end{center}

\section{Notations}

{\renewcommand{\arraystretch}{1.2}%
\begin{table}[h!]
    \centering
      \caption{List of Notations} \vspace{0.3cm}
    \begin{tabular}{l|l}
    $x^+$ & $\max(x,0)$, \ie, the ReLU function. \\
    $\Scal,\Acal,S,A$ & State and action spaces, with sizes $S=|\Scal|$ and $A = |\Acal|$. In MAB, $S=1$. \\
    $\Saug$ & Augmented state space $\Scal\times[0,1]$ for CVaR RL. \\
    $H\in\NN$ & Horizon of the RL problem. In MAB, $H=1$. \\
    $K\in\NN$ & Number of episodes. \\
    $\delta\in(0,1)$ & Failure probability. \\
    $L$ & $\log(SAHK/\delta)$. \\
    $\Delta(\Scal)$ & The set of distributions supported by $\Scal$. \\
    $R(a)\in\Delta([0,1])$ & Reward distribution of arm $a$ (for MAB). \\
    $P^\star(s,a)\in\Delta(\Scal)$ & Ground truth transition kernel (for RL). \\
    $R(s,a)\in\Delta([0,1])$ & Known reward distribution (for RL). \\
    $R(\pi)$ & Returns distribution of history-dependent policy $\pi$ (for RL). \\
    $R(\rho,b)$ & Returns distribution of augmented policy $\rho\in\Piaug$ starting from $b$ (for RL). \\
    $F^\dagger(t)$ for $t\in[0,1]$ & The $t$-th quantile function of $X$ with CDF $F$, \ie, $\inf\braces{x: F(x)\geq t}$. \\
    $\Ical_{h,k}(s,a)$ & Indices of prior visits of $s,a$ at $h$, \ie, $\braces{i\in[k-1]: (s_{h,i},a_{h,i})=(s,a)}$. \\
    $N_{h,k}(s,a)$ & Number of prior visits of $s,a$ at $h$, \ie, $|\Ical_{h,k}(s,a)|$. \\
    $\xi_{h,k}(s,a)$ & $\min\braces{2,\frac{2HSL}{N_{h,k}(s,a)}}$. \\
    $\Ecal_k$ & Trajectories from episodes $1,2,...,k-1$. \\
    $\Hcal_h (\Hcal_{h,k})$ & History up to and not including time $h$ (in episode $k$). \\
    $\Pi_\Hcal$ & Set of history-dependent policies. \\
    $\Piaug$ & Set of Markov, deterministic policies in the augmented MDP. \\
    $(\rho,b)$ & The policy obtained from rolling in $\rho$ starting from $(s_1,b)$ in the augmented MDP. \\
    $\disc(\Mcal)$ & The discretized MDP obtained by discretizing rewards, \cref{sec:discretized-mdp-computational-efficiency}. \\
    $\adapted(\rho,b_1)$ & The policy from adapting $(\rho,b_1)$ in $\disc(\Mcal)$ to $\Mcal$. \\
    $\disc(\rho,b_1)$ & The policy from discretizing $(\rho,b_1)$ in $\Mcal$ to $\disc(\Mcal)$.
    \end{tabular}
    \label{tab:notation}
\end{table}
}

\section{Concentration Lemmas}\label{sec:concentration-lemmas}
\subsection{Uniform Hoeffding and Bernstein via Lipschitzness}
Recall the classic Hoeffding and Bernstein inequalities (Theorems 2.8 and 2.10 in \citealp{boucheron2013concentration}).
Let $X_{1:N}$ be i.i.d. random variables in $[0,1]$, with mean $\mu$ and variance $\sigma^2$.
Then, for any $\delta$, w.p. at least $1-\delta$, we have
\begin{align*}
    &\abs{ \frac{1}{N}\sum_{i=1}^N X_i-\mu } \leq \frac12 \sqrt{ \frac{\log(4/\delta)}{N} }, \tag{Hoeffding}
    \\&\abs{ \frac{1}{N}\sum_{i=1}^N X_i-\mu } \leq \sqrt{\frac{2\sigma^2\log(4/\delta)}{N}} + \frac{\log(4/\delta)}{N}. \tag{Bernstein}
\end{align*}

Now we consider uniform inequalities for a function class.
Specifically, let $X_{1:N}$ be i.i.d. copies of $X\in\Xcal$ and $\Fcal$ is a (potentially infinite) set of functions $f:\Xcal\to[0,1]$.
Suppose $\Gcal_\eps\subset\Fcal$ is a finite $\ell_\infty$-cover, a.k.a. $\eps$-net, of $\Fcal$ in the sense that:
for any $f\in\Fcal$, there exists $g\in\Gcal_\eps$ such that $\sup_{x\in\Xcal}\abs{f(x)-g(x)}\leq\eps$.
\begin{lemma}
Let $\delta\in(0,1)$.
We have w.p. at least $1-\delta$,
\begin{align*}
    &\sup_{f\in\Fcal} \abs{ \frac{1}{N}\sum_{i=1}^N f(X_i)-\EE f(X) } \leq \sqrt{ \frac{\log(4|\Gcal_{1/N}|/\delta)}{N} }. \tag{Uniform Hoeffding}
\end{align*}
If $N\geq 2\log(4|\Gcal_{1/N}|/\delta)$, we also have
\begin{align*}
    &\sup_{f\in\Fcal} \abs{ \frac{1}{N}\sum_{i=1}^N f(X_i)-\EE f(X) } \leq \sqrt{\frac{2\Varb{f(X)}\log(4|\Gcal_{1/N}|/\delta)}{N}} + \frac{3\log(4|\Gcal_{1/N}|/\delta)}{N}. \tag{Uniform Bernstein}
\end{align*}
\end{lemma}
\begin{proof}
Apply a union bound over the elements of $\Gcal_\eps$.
Then for any $f\in\Fcal$,
\begin{align*}
    \abs{ \frac{1}{N}\sum_{i=1}^N f(X_i)-\EE f(X) }
    &\leq 2\eps + \abs{ \frac{1}{N}\sum_{i=1}^N g(X_i)-\EE g(X) }
    \\&\leq 2\eps + \frac{1}{2}\sqrt{\frac{\log(4|\Gcal_\eps|/\delta)}{N}}.
\end{align*}
Setting $\eps=1/N$ gives the Uniform Hoeffding result.
We also have
\begin{align*}
    \abs{ \frac{1}{N}\sum_{i=1}^N f(X_i)-\EE f(X) }
    &\leq 2\eps + \sqrt{\frac{2\Var(g(X))\log(4|\Gcal_\eps|/\delta)}{N}} + \frac{\log(4|\Gcal_\eps|/\delta)}{N}
    \\&\leq 2\eps + \sqrt{\frac{2\Var(f(X))\log(4|\Gcal_\eps|/\delta)}{N}} + \frac{\log(4|\Gcal_\eps|/\delta)}{N} + \eps\sqrt{\frac{2\log(4|\Gcal_\eps|/\delta)}{N}},
\end{align*}
since $\sqrt{\Varb{g(X)}}-\sqrt{\Varb{f(X)}}\leq \sqrt{\Varb{f(X)-g(X)}} \leq \eps$.
By assumption, $\sqrt{\frac{2\log(4|\Gcal_\eps|/\delta)}{N}}\leq 1$, so the total error is at most $3\eps$.
Thus, setting $\eps=1/N$ gives the Uniform Bernstein result.
\end{proof}

A particularly important application of this for us is that $\Fcal$ will be an finite set of functions $f_b$ parameterized by a continuous parameter $b\in[0,1]$.
These functions are $C$-Lipschitz in the $b$ parameter, so to construct $\Gcal_\eps$,
it suffices to take a grid over $[0,1]$ such that any element is $\eps/C$ close to the grid.
This grid requires $\ceil{C/\eps}$ atoms, and so $\log(|\Gcal_{1/N}|)\leq \log(CN)$.

\textbf{Empirical Bernstein: }
By Theorems 4 and 6 of \citet{maurer2009empirical}, we also have an empirical version of the uniform Bernstein, where we may replace $\Varb{f(X)}$ with $\frac{1}{N(N-1)}\sum_{i,j=1}^N \prns{f(X_i)-f(X_j)}^2$, \ie, the empirical variance.
Another useful result of \citet{maurer2009empirical} is their Theorem 10, which proves a fast convergence of empirical variance to the true variance: w.p. $1-\delta$,
\begin{align*}
    \abs{ \wh\Var f(X) - \Var f(X) } \leq \sqrt{\frac{2\log(2/\delta)}{N-1}},
\end{align*}
where $\wh\Var f(X) = \frac{1}{N(N-1)}\sum_{i,j=1}^N\prns{f(X_i)-f(X_j)}^2$ is the sample variance of $N$ datapoints.
Note that $\wh\Var f(X)$ is the variance under the empirical distribution of these $N$ datapoints, and hence behaves like a variance.
Since $\sqrt{\Var(X+Y)}\leq \sqrt{\Var(X)}+\sqrt{\Var(Y)}$ by Cauchy-Schwartz, this can also be extended to be uniform by the above argument, \ie,
\begin{align*}
    \sup_{f\in\Fcal}\abs{ \sqrt{\wh\Var f(X)}-\sqrt{\Var f(X)} } \leq 2\sqrt{\frac{\log(2|\Gcal_{1/N}|/\delta) }{N-1}}.
\end{align*}
\begin{proof}
For any $f$, let $g$ be its neighbor in the net.
Using the triangle inequality of variance, $\sqrt{\Var(X+Y)}\leq\sqrt{\Var(X)}+\sqrt{\Var(Y)}$, we have
\begin{align*}
    \abs{ \sqrt{\wh\Var f(X)}-\sqrt{\Var f(X)} }
    &\leq\abs{ \sqrt{\wh\Var g(X)}-\sqrt{\Var g(X)} } + \sqrt{\wh\Var((f-g)(X))} + \sqrt{\Var((f-g)(X))}
    \\&\leq 2\sqrt{\frac{\log(2|\Gcal_{1/N}|/\delta) }{N-1}} + 2\eps.
\end{align*}
Setting $\eps=1/N$ completes the proof.
\end{proof}

\subsection{Tape Method for Tabular MAB and RL}\label{sec:tape-method}
In this section, we describe how we are able to prove claims about MAB and RL using uniform concentration inequalities over i.i.d. data,
\ie, without needing to use complicated uniform martingale inequalities, \eg, \citet{bibaut2021risk,van2000empirical}.
We construct a probability space using a ``tape'' method inspired by \citet[Section 1.3.1]{slivkins2019introduction}.
Compared to using black-box uniform martingale inequalities, our approach is potentially loose in log terms.
However, our approach is much cleaner as we only need uniform concentrations for i.i.d. data.
Thus, we prove everything from first principles, so that concentration inequalities do not distract from the main ideas.

First, consider the MAB problem.
Before the protocol starts, nature constructs an (one-indexed) array with $AK$ cells.
For each $a\in\Acal,k\in[K]$, nature fills the index $[a,k]$ with an independent sample from $R(a)$.
Whenever the learner pulls arm $a$ on the $k$-th episode, it receives the contents of $(a,N_k(a)+1)$
where that $N_k(a)$ is the number of times that $a$ has been pulled up until now.

Notice that this formulation will never run out of rewards since we've seeded each arm with $K$ cells. Also, this is equivalent to drawing a sample whenever the learner pulls arm $a$. Crucially, all the rewards are independent and so we can obtain concentration inequalities for $N_k(a)$ for any learner, even before it is executed.

In particular, for any function $f:[0,1]\to[0,1]$ of the rewards, we can union bound Hoeffding/Bernstein over the cells $[a,1:k]$ for all $a,k$ to get: for any $\delta$, w.p. at least $1-\delta$, we have for all $a,k$,
\begin{align*}
    &\abs{ \frac{1}{N_k(a)}\sum_{i\in \Ical_k(a)}f(r_i) - \EE f(R(a)) } \leq \frac12\sqrt{\frac{\log(4AK/\delta)}{N_k(a)}},
    \\&\abs{ \frac{1}{N_k(a)}\sum_{i\in \Ical_k(a)}f(r_i) - \EE f(R(a)) } \leq \sqrt{ \frac{2\Varb{f(R(a))}\log(4AK/\delta))}{N_k(a)} } + \frac{\log(4AK/\delta)}{N_k(a)}.
\end{align*}
Here $\Ical_k(a)$ is the indices where the learner has pulled arm $a$.

We now do something similar for tabular RL. Before the RL algorithm starts, nature constructs an (one-indexed) array with $SAHK$ cells. For each $s\in\Scal,a\in\Acal,h\in[H],k\in[K]$, nature fills the index $[s,a,h,k]$ with an independent sample from $P^\star(s,a)$. Whenever, the learner takes action $a$ at state $s$ and step $h$ on episode $k$, it receives the next state via the content of $(s,a,h,N_{h,k}(s,a)+1)$ where recall $N_{h,k}(s,a)$ is the number of times the learner has taken action $a$ at state $s$ and step $h$ before the current episode.
Then, for any function $f:\Scal\to[0,1]$ of the states, we can union bound Hoeffding/Bernstein over the cells $[s,a,1:H,1:k]$ for all $s,a,k$ to get: for any $\delta$, w.p. at least $1-\delta$, we have for all $s,a,h,k$,
\begin{align*}
    &\abs{ \frac{1}{N_{k}(s,a)}\sum_{h,i\in \Ical_{k}(s,a)}f(s_{h+1,i}) - \EE_{s'\sim P^\star(s,a)} f(s') } \leq \frac12\sqrt{\frac{\log(4SAHK/\delta)}{N_{k}(s,a)}},
    \\&\abs{ \frac{1}{N_{k}(s,a)}\sum_{h,i\in \Ical_{k}(s,a)}f(s_{h+1,i}) - \EE_{s'\sim P^\star(s,a)} f(s') } \leq \sqrt{ \frac{2\Varb[s'\sim P^\star(s,a)]{f(s')}\log(4SAHK/\delta))}{N_{k}(s,a)} } + \frac{\log(4SAHK/\delta)}{N_{k}(s,a)},
\end{align*}
where $\Ical_k(s,a)$ are the $(h,i)$ pairs where the learner has visited $(s,a)$, and $N_k(s,a)$ is the size of $\Ical_k(s,a)$.

Since these are standard Hoeffding/Bernstein results over i.i.d. data, the uniform concentration results from the previous section applies.

\newpage
\section{Concentration of CVaR}
In this section, we derive general concentration results for the empirical CVaR, which may be of independent interest.
The significance of our result is that it applies to any bounded random variable $X$, which may be continuous, discrete or neither.
Prior concentration results from \citet{brown2007large} require $X$ to be continuous.
Some later works \citep{wang2010deviation,kagrecha2019distribution} did not explicitly mention their
dependence on the continuity of $X$, but their proof appears to require it as well and is complicated by casework.
We provide a simple new proof of this concentration based on the Acerbi integral formula for CVaR, \cref{lem:cvar-integral-over-quantiles}.

For any random variable $X$ in $[0,1]$ with CDF $F$, the quantile function is defined as,
\begin{align*}
    F^\dagger(t) = \inf\braces{ x\in[0,1]: F(x)\geq t } = \sup\braces{ x\in[0,1]: F(x) < t }.
\end{align*}
The quantile has many useful properties \citep[Lemma 1]{orderstatsnotes}, which we recall here.
\begin{lemma}\label{lem:quantile-props}
For $t\in(0,1)$, $F^\dagger(t)$ is nondecreasing and left-continuous, and satisfies
\begin{enumerate}
    \item For all $x\in\RR$, $F^\dagger(F(x))\leq x$.
    \item For all $t\in(0,1)$, $F(F^\dagger(t))\geq t$.
    \item $F(x)\geq t \iff x\geq F^\dagger(t)$.
\end{enumerate}
\end{lemma}

The quantile is always a maximizer of the $\cvar$ objective in \cref{eq:cvar-definition}.
This is true for any random variable, discrete, continuous or neither.
\begin{lemma}\label{lem:cvar-quantile-optimality}
For any random variable $X$ in $[0,1]$ with CDF $F$, we have
\begin{align*}
    F^\dagger(\tau)\in\argmax_{b\in[0,1]}\braces{ b-\tau^{-1}\Eb{(b-X)^+} }.
\end{align*}
\end{lemma}
\begin{proof}
Recall the objective in \cref{eq:cvar-definition}, $f(b) = -b+\tau^{-1}\Eb{(b-X)^+}$.
It has a subgradient of
\begin{align*}
    \partial f(b) = -1+\tau^{-1}\prns{ \Pr(X<b)+[0,1]\Pr(X=b) }.
\end{align*}
We want to show that $0\in\partial f(F^\dagger(\tau))$, which is equivalent to showing
\begin{align*}
    0 \overset{(a)}{\leq} \tau-\Pr(X<F^\dagger(\tau)) \overset{(b)}{\leq} \Pr(X=F^\dagger(\tau)).
\end{align*}
For (b), observe that $\Pr(X<F^\dagger(\tau))+\Pr(X=F^\dagger(\tau))=F(F^\dagger(\tau))\geq\tau$ (\cref{lem:quantile-props}). Hence, $\tau-\Pr(X<F^\dagger(\tau))\leq\Pr(X=F^\dagger(\tau))$.

For (a), recall that $\Pr(X<F^\dagger(\tau))=\lim_{n\to\infty}\Pr(X\leq F^\dagger(\tau)-n^{-1})$, since $\braces{X\leq F^\dagger(\tau)-1}\subset \braces{X\leq F^\dagger(\tau)-\nicefrac12}\subset \dots \subset \bigcup_{n\in\NN}\braces{X\leq F^\dagger(\tau)-n^{-1}}=\braces{X<F^\dagger(\tau)}$ and continuity of probability measures.
If for any $n\in\NN$, we had $\Pr(X\leq F^\dagger(\tau)-n^{-1})\geq \tau$, \ie, $F(F^\dagger(\tau)-n^{-1})\geq\tau$, then by definition of $F^\dagger(\tau) = \inf\braces{x\in[0,1]: F(x)\geq t}$, we have $F^\dagger(\tau)\leq F^\dagger(\tau)-n^{-1}$, which is a contradiction. Therefore, it must be that for all $n\in\NN$, we have $\Pr(X\leq F^\dagger(\tau)-n^{-1})<\tau$, so $\Pr(X<F^\dagger(\tau))=\lim_{n\to\infty}\Pr(X\leq F^\dagger(\tau)-n^{-1})\leq\tau$.
\end{proof}

The following interpretation of $\cvar_\tau$ due to \citet{acerbi2002coherence} will be very useful.
An alternative proof was given in \citet[Proposition 2.2]{kisiala2015conditional}.
\begin{lemma}[Acerbi's Integral Formula]\label{lem:cvar-integral-over-quantiles}
For any non-negative random variable $X$ with CDF $F$, we have
\begin{align*}
    \cvar_\tau(X) = \tau^{-1}\int_0^\tau F^\dagger(y)\diff y = \Eb{F^\dagger(U)\mid U\leq\tau},
\end{align*}
where $U\sim\unif{[0,1]}$.
\end{lemma}

Now suppose $X_{1:N}$ are i.i.d. copies of $X\in[0,1]$.
Define the empirical CVaR as the CVaR of the empirical distribution $\wh F_N(x) = \frac{1}{N}\sum_{i=1}^N \I{X_i\leq x}$.
\begin{align*}
    \wh\cvar_\tau(X_{1:N}) = \max_{b\in[0,1]}\braces{ b-\frac{1}{N\tau}\sum_{i=1}^N(b-X_i)^+ }.
\end{align*}
Let $X_{(i)}$ denote the $i$-th increasing order statistic.
\begin{lemma}\label{lem:empirival-cvar-characterize}
The maximum for the empirical CVaR is attained at the empirical quantile $X_{\ceil{N\tau}}$. Hence,
\begin{align*}
    \wh\cvar_\tau(X_{1:N}) = \prns{ 1-\frac{\ceil{N\tau}}{N\tau} }X_{\ceil{N\tau}} + \frac{1}{N\tau}\sum_{i=1}^{\ceil{N\tau}}X_{(i)}.
\end{align*}
\end{lemma}
\begin{proof}
By \cref{lem:cvar-quantile-optimality}, the maximum is attained at the $\tau$-th quantile of the empirical distribution, \ie, $\wh F_N^\dagger(\tau) = \inf\braces{x: \wh F_N(x)\geq \tau}$.
Let $k\in[N]$ be the largest $X_{(k)}$ such that such that $\wh F_N(X_{(k)})=\frac{k}{N}< \tau \leq \frac{k+1}{N}=\wh F_N(X_{(k+1)})$.
This implies that $\wh F_N^\dagger(\tau)=X_{(k+1)}$. Note that $k < N\tau \leq k+1$, so $k+1=\ceil{N\tau}$.
Thus,
\begin{align*}
    \wh\cvar_\tau(X_{1:N})
    &= X_{(\ceil{N\tau})}-\frac{1}{N\tau}\sum_{i=1}^N\prns{X_{(\ceil{N\tau})}-X_i}^+
    \\&= X_{(\ceil{N\tau})}-\frac{1}{N\tau}\sum_{i=1}^{\ceil{N\tau}}\prns{ X_{(\ceil{N\tau})}-X_{(i)} }
    \\&= \prns{ 1-\frac{\ceil{N\tau}}{N\tau} }X_{(\ceil{N\tau})} + \frac{1}{N\tau}\sum_{i=1}^{\ceil{N\tau}}X_{(i)}.
\end{align*}
\end{proof}

\begin{lemma}\label{lem:uniform-dist-quantile-concentrate}
Let $U_{1:N}$ be i.i.d. copies of $\unif{[0,1]}$. Let $p\in(0,1)$.
For any $\delta\in(0,1)$, w.p. at least $1-\delta$, we have
\begin{align*}
    \abs{ U_{\ceil{Np}} - p}\leq \sqrt{\frac{3p(1-p)\log(2/\delta)}{N}} + \frac{5\log(2/\delta)}{N},
\end{align*}
provided that $N\geq 25\log(2/\delta)$.
\end{lemma}
\begin{proof}
Let $F$ be the distribution function of $\unif{[0,1]}$, and $\wh F_N$ the empirical distribution of $U_{1:N}$.
Note that $U_{(\ceil{Np})}=\wh F^\dagger_N(p)$, by reasoning in the proof of \cref{lem:empirival-cvar-characterize}. So the left hand side is $\abs{p-\wh F^\dagger_N(p)}$.

Now consider any error $\eps\in(0,1)$. We have
\begin{align*}
    \wh F^\dagger_N(p) \leq p+\eps
    \iff p\leq \wh F_N(p+\eps)
    \iff p+\eps-\wh F_N(p+\eps)\leq \eps,
\end{align*}
and
\begin{align*}
    \wh F_N^\dagger(p) > p-\eps
    \iff p > \wh F_N(p-\eps)
    \iff \wh F_N(p-\eps)-(p-\eps) < \eps.
\end{align*}

In both cases, we can use Bernstein on $\I{U\leq p-\eps}$ or $\I{U\leq p+\eps}$ to obtain a bound depending on the variance.
In the first case, $\sqrt{\Varb{\I{U\leq p-\eps}}}\leq \sqrt{\Varb{\I{U\leq p}}}+\sqrt{\Varb{\I{U\leq p-\eps}-\I{U\leq p}}}\leq p(1-p)+\eps$.
Similarly, $\sqrt{\Varb{\I{U\leq p+\eps}}}\leq \sqrt{\Varb{\I{U\leq p}}}+\sqrt{\Varb{\I{U\leq p+\eps}-\I{U\leq p}}}\leq p(1-p)+\eps$.
Thus, we have w.p. at least $1-\delta$,
\begin{align*}
    (p+\eps)-\wh F_N(p+\eps) \leq \sqrt{\frac{2p(1-p)\log(2/\delta)}{N}} + \frac{\log(2/\delta)}{N} + \sqrt{\frac{2\eps^2\log(2/\delta)}{N}},
\end{align*}
and
\begin{align*}
    \wh F_N(p-\eps)-(p-\eps) < \sqrt{\frac{3p(1-p)\log(2/\delta)}{N}}+\frac{\log(2/\delta)}{N} + \sqrt{\frac{3\eps^2\log(2/\delta)}{N}}.
\end{align*}
So we can set $\eps = \sqrt{\frac{3p(1-p)\log(2/\delta)}{N}}+\frac{5\log(2/\delta)}{N}$, as the third error term with this setting of $\eps$ is at most $\frac{3\log(2/\delta)}{N}+\frac{5\log(2/\delta)^{1.5}}{N^{1.5}}\leq \frac{4\log(2/\delta)}{N}$ when $N\geq 25\log(2/\delta)$.
Thus w.p. at least $1-\delta$, we have $\abs{\wh F_N^\dagger(p)-p}\leq\eps$.
\end{proof}

\begin{theorem}\label{thm:cvar-concentration-main}
Let $X_{1:N}$ be $N$ i.i.d. copies of a random variable $X\in[0,1]$. Then for any $\delta\in(0,1)$, w.p. at least $1-\delta$, if $N\geq 25\log(2/\delta)$, then we have
\begin{align*}
    \abs{ \wh\cvar_\tau(X_{1:N})-\cvar_\tau(X) }\leq \sqrt{\frac{3\log(2/\delta)}{N\tau}} + \frac{15\log(2/\delta)}{N\tau}.
\end{align*}
\end{theorem}
\begin{proof}
We use the interpretation of the empirical CVaR in \cref{lem:empirival-cvar-characterize}. The first term is lower order since
\begin{align*}
    \abs{1-\frac{\ceil{N\tau}}{N\tau}} \leq \frac{1}{N\tau}.
\end{align*}
Now recall that by the inverse CDF trick, we have $X_i = F^\dagger(U_i)$ where $U_i$ are i.i.d. copies of $\unif{[0,1]}$. Since $F^\dagger$ is non-decreasing, we have $X_{(i)}=F^\dagger(U_{(i)})$. Thus, the second term of \cref{lem:empirival-cvar-characterize} is
\begin{align*}
    \frac{1}{N\tau}\sum_{i=1}^{\ceil{N\tau}}X_{(i)}
    = \frac{1}{N\tau}\sum_{i=1}^{\ceil{N\tau}}F^\dagger(U_{(i)})
    = \frac{1}{N\tau}\sum_{i=1}^{N}F^\dagger(U_{i})\I{U_i\leq U_{(\ceil{N\tau})}},
\end{align*}
which we want to show is close to $\cvar_\tau(X)=\tau^{-1}\Eb{F^\dagger(U)\I{U\leq\tau}}$ by \cref{lem:cvar-integral-over-quantiles}.
If $U_{(\ceil{N\tau})}$ were replaced by $\tau$, we can simply invoke Bernstein and note that $\Varb{F^\dagger(U)\I{U\leq \tau}} \leq \Pr(U\leq \tau) = \tau$, which gives
\begin{align*}
    \abs{ \frac{1}{N\tau}\sum_{i=1}^{N}F^\dagger(U_{i})\I{U_i\leq \tau}-\tau^{-1}\Eb{F^\dagger(U)\I{U\leq\tau}} } \leq \tau^{-1}\prns{ \sqrt{\frac{2\tau\log(2/\delta)}{N}}+\frac{\log(2/\delta)}{N} }.
\end{align*}
Thus, we just need to bound the difference term,
\begin{align*}
    &\phantom{=}\abs{ \frac{1}{N\tau}\sum_{i=1}^{N}F^\dagger(U_{i})\prns{ \I{U_i\leq U_{(\ceil{N\tau})}} - \I{U_i\leq \tau}} }.
\end{align*}
By \cref{lem:uniform-dist-quantile-concentrate}, we have $\abs{ U_{\ceil{N\tau}}-\tau }\leq\eps$ w.p. $1-\delta$, where $\eps = \sqrt{\frac{3\tau(1-\tau)\log(2/\delta)}{N}}+\frac{5\log(2/\delta)}{N}$.
So, for any $U_i$ we have $\I{U_i\leq U_{(\ceil{N\tau})}} - \I{U_i\leq \tau}\leq\I{\tau\leq U_i\leq \tau+\eps}$ and $\I{U_i\leq \tau}\leq \I{\tau-\eps\leq U_i\leq \tau}$, so the difference term is at most,
\begin{align*}
    &\leq \max\braces{\frac{1}{N\tau}\sum_{i=1}^{N}F^\dagger(U_{i})\I{\tau-\eps\leq U_i\leq \tau}, \frac{1}{N\tau}\sum_{i=1}^NF^\dagger(U_i)\I{\tau\leq U_i\leq \tau+\eps} }.
\end{align*}
By applying another Bernstein, and noting that $\sqrt{\Varb{F^\dagger(U)\I{\tau-\eps\leq U\leq \tau}}}\leq \eps$, $\sqrt{\Varb{F^\dagger(U)\I{\tau\leq U\leq \tau+\eps}}}\leq \eps$, we have this is at most
\begin{align*}
    &\tau^{-1}\prns{ \max\braces{\Eb{F^\dagger(U)\I{\tau-\eps\leq U\leq \tau}},\Eb{F^\dagger(U)\I{\tau\leq U\leq \tau+\eps}}}+\sqrt{\frac{2\eps^2\log(2/\delta)}{N}}+\frac{\log(2/\delta)}{N} }
    \\&\leq \tau^{-1}\prns{ \eps+\sqrt{\frac{2\eps^2\log(2/\delta)}{N}}+\frac{\log(2/\delta)}{N} }
    \\&\leq \sqrt{\frac{3\log(2/\delta)}{N\tau}}+\frac{5\log(2/\delta)}{N\tau}+\frac{3\log(2/\delta)}{N\sqrt{\tau}}+\frac{4\log^{1.5}(2/\delta)}{N^{1.5}\tau} + \frac{\log(2/\delta)}{N\tau}
    \\&\leq \sqrt{\frac{3\log(2/\delta)}{N\tau}} + \frac{15\log(2/\delta)}{N\tau}, \tag{when $N\geq 16\log(2/\delta)$}
\end{align*}
where the bound on $\eps$ occurs when $N\geq 25\log(2/\delta)$, which also implies the last inequality.
\end{proof}

\section{Proofs for Lower Bounds}
\subsection{CVaR MAB Lower Bound}
Let us first define some MAB notations that make explicit the dependence on the current MAB problem instance $\nu$ and the learner $\alg$. Recall that $\nu$ is a vector of $A$ reward distributions, and in the $k$-th episode, $\alg$ picks an action $a_k$ based on the historical actions and rewards.
Let $\Delta_a(\nu)=\cvar^\star_\tau(\nu)-\cvar_\tau(\nu(a))$ where $\cvar^\star_\tau(\nu)=\max_{a\in\Acal}\cvar_\tau(\nu(a))$. Let $\mabregret_\tau(K,\nu,\alg)$ denote the regret of running $\alg$ in MAB $\nu$ for $K$ episodes.
Let $T_k(a)$ denote the number of times an arm $a$ has been pulled up to time $K$.

For two distributions $P,Q$, recall the KL-divergence is defined as
\begin{align*}
    \kl{P}{Q}= \begin{cases}
        &\int\log\prns{\frac{\diff P}{\diff Q}(\omega)\diff P(\omega)}, \text{ if } P\ll Q,
        \\&\infty, \text{ otherwise}.
    \end{cases}
\end{align*}
A key inequality for lower bounds is the Bretagnolle-Huber inequality, cf. \citep[Theorem 14.2]{lattimore2020bandit},
\begin{lemma}[Bretagnolle-Huber]\label{lem:bretagnolle-huber}
Let $P,Q$ be probability measures on the same measurable space $(\Omega,\Fcal)$ and $A\in\Fcal$ be any event. Then
\begin{align*}
    P(A)+Q(A^C)\geq \frac12\exp\prns{-\kl{P}{Q}}.
\end{align*}
\end{lemma}

\begin{lemma}[Regret Decomposition]\label{lem:cvar-bandit-regret-decomp}
For any MAB instance $\nu$ and learner $\alg$, we have
\begin{align*}
    \Eb{\mabregret_\tau(K,\nu,\alg)} = \sum_{a\in\Acal}\Delta_a(\nu)\Eb{T_a(K)},
\end{align*}
where the expectations are with respect to the trajectory of running $\alg$ in $\nu$.
\end{lemma}
\begin{proof}
\begin{align*}
    &\Eb{\mabregret_\tau(K,\nu,\alg)}
    \\&=\sum_{k=1}^K \cvar^\star_\tau(\nu)-\Eb{\cvar_\tau(\nu(a_k))}
    \\&=\sum_{k=1}^K \Eb{ \prns{\cvar^\star_\tau(\nu)-\cvar_\tau(\nu(a_k)))} \sum_{a\in\Acal} \I{a_k=a} }
    \\&= \sum_{a\in\Acal}\sum_{k=1}^K \Eb{ \prns{\cvar^\star_\tau(\nu)-\cvar_\tau(\nu(a_k))} \I{a_k=a} }.
\end{align*}
Notice that if once we condition on $a_k$, if $a_k=a$, the difference $\cvar^\star_\tau(\nu)-\cvar_\tau(\nu(a_k))$ is simply $\Delta_{a}(\nu)$.
If $a_k\neq a$, then we get $0$. So, by the tower rule,
\begin{align*}
    \Eb{ \prns{\cvar^\star_\tau(\nu)-\cvar_\tau(\nu(a_k))} \I{a_k=a} } = \Eb{ \I{a_k=a} \Delta_a(\nu) }.
\end{align*}
Therefore, continuing from before,
\begin{align*}
    &= \sum_{a\in\Acal}\sum_{k=1}^K \Eb{ \I{a_k=a}\Delta_a(\nu) }
    \\&= \sum_{a\in\Acal}\Delta_a(\nu)\sum_{k=1}^K \Eb{ \I{a_k=a} }
    \\&= \sum_{a\in\Acal}\Delta_a(\nu) \Eb{ T_a(K) }.
\end{align*}
\end{proof}

\mabLowerBound*
\begin{proof}[Proof of \cref{thm:mab-lower-bound}]
Fix any $\tau\in(0,\nicefrac12)$ and any MAB algorithm $\alg$. WLOG suppose $\Acal=[A]$.
Define the shorthand, $\beta_c = \ber(1-\tau+c\eps)$, \ie, larger $c$ implies $\eps$ more likelihood of pulling $1$.
Construct two MAB instances as follows,
\begin{align*}
    \nu&=(\beta_1,\beta_0,\dots,\beta_0)
    \\&\nu'=(\beta_1,\beta_0,\dots,\beta_0,\underbrace{\beta_2}_{\text{index } i},\beta_0,\dots,\beta_0), \text{ where } i=\argmin_{a>1}\Eb[\nu,\alg]{T_a(K)}.
\end{align*}
For the first MAB instance $\nu$, the optimal action is $a^\star(\nu)=1$, and $\Delta_a(\nu)=\tau^{-1}\eps$ for all $a>1$.
By \cref{lem:cvar-bandit-regret-decomp},
\begin{align*}
    \Eb[\nu,\alg]{\mabregret_\tau(K,\nu,\alg)}
    &=\sum_{a\in\Acal}\Delta_a(\nu)\Eb[\nu,\alg]{ T_a(K) }
    \\&=\tau^{-1}\eps\prns{ K-\Eb[\nu,\alg]{T_1(K)} }
    \\&\geq\tau^{-1}\eps \Pr_{\nu,\alg}\prns{ K-T_1(K)\geq\frac{K}{2} }\frac{K}{2} \tag{Markov's inequality}
    \\&=\geq\tau^{-1}\eps \Pr_{\nu,\alg}\prns{ T_1(K)\leq \frac{K}{2}}\frac{K}{2}.
\end{align*}
For the second MAB instance $\nu'$, the optimal action is $a^\star(\nu')$, and $\Delta_1(\nu)=\tau^{-1}\eps$.
By \cref{lem:cvar-bandit-regret-decomp},
\begin{align*}
    \Eb[\nu',\alg]{\mabregret_\tau(K,\nu',\alg)}
    &=\sum_{a\in\Acal}\Delta_a(\nu')\Eb[\nu',\alg]{T_a(K)}
    \\&\geq \nu_1(\nu')\Eb[\nu',\alg]{T_1(K)}
    \\&>\tau^{-1}\eps \Pr_{\nu',\alg}\prns{ T_1(K)>\frac{K}{2} }\frac{K}{2}. \tag{Markov's inequality}
\end{align*}
Let $\PP_{\nu,\alg}$ denote the trajectory distribution from running $\alg$ in MAB $\nu$. Therefore,
\begin{align*}
    &\phantom{=}\Eb[\nu,\alg]{\mabregret_\tau(K,\nu,\alg)}+\Eb[\nu',\alg]{\mabregret_\tau(K,\nu',\alg)}
    \\&> \frac{K\eps}{2\tau}\prns{ \Pr_{\nu,\alg}\prns{ T_1(K)\leq \frac{K}{2}} + \Pr_{\nu',\alg}\prns{ T_1(K)>\frac{K}{2} } }
    \\&\geq \frac{K\eps}{4\tau}\exp\prns{ -\kl{\PP_{\nu,\alg}}{\PP_{\nu',\alg}} } \tag{Bretagnolle-Huber \cref{lem:bretagnolle-huber}}
    \\&= \frac{K\eps}{4\tau}\exp\prns{ -\sum_{a\in\Acal}\Eb[\nu,\alg]{ T_a(K) }\kl{\nu(a)}{\nu'(a)} } \tag*{\citep[Lemma 15.1]{lattimore2020bandit}}
    \\&= \frac{K\eps}{4\tau}\exp\prns{ -\Eb[\nu,\alg]{ T_i(K) }\kl{\nu(i)}{\nu'(i)} } \tag{other arms are the same for $\nu,\nu'$}
    \\&= \frac{K\eps}{4\tau}\exp\prns{ -\Eb[\nu,\alg]{ T_i(K) }\kl{\nu(i)}{\nu'(i)} }
    \\&\geq \frac{K\eps}{4\tau}\exp\prns{ -\frac{8K\eps^2}{(A-1)\tau} }.
\end{align*}
The last inequality uses two facts. By definition of $i = \argmin_{a>1}\Eb[\nu,\alg]{T_a(K)}$, $\Eb[\nu,\alg]{T_i(K)}\leq \frac{K}{A-1}$.
Also, by \cref{lem:kl-bernoulli}, $\kl{\nu(i)}{\nu'(i)}\leq8\eps^2\tau^{-1}$.
Setting $\eps^2=\frac{(A-1)\tau}{8K}$ and noting $2\max\braces{a,b}\geq a+b$ gives the desired lower bound.
\end{proof}

\begin{lemma}\label{lem:cvar-bernoulli}
For any $\tau\in(0,\nicefrac12)$ and $\eps\in[0,\tau]$, we have
\begin{align*}
    \cvar_\tau(\ber(1-\tau+\eps))=\tau^{-1}\eps.
\end{align*}
\end{lemma}
\begin{proof}
The CDF of $X\sim\ber(1-\tau+\eps)$ is as follows,
\begin{align*}
    F(x) =
    \begin{cases}
        &0, \text{ if } x < 0,
        \\&\tau-\eps, \text{ if } x\in[0,1),
        \\&1, \text{ if } x\geq 1.
    \end{cases}
\end{align*}
Therefore, $F^\dagger(\tau)=\inf\braces{x:F(x)\geq\tau} = 1$ for any $\eps > 0$, and it is $0$ when $\eps=0$.
By \cref{lem:cvar-quantile-optimality}, we have
\begin{align*}
    \cvar_\tau(\ber(1-\tau))=0-\tau^{-1}\Eb{(0-X)^+}=0,
\end{align*}
and
\begin{align*}
    \cvar_\tau(\ber(1-\tau+\eps))=1-\tau^{-1}\Eb{(1-X)^+}=1-\tau^{-1}(\tau-\eps) = \tau^{-1}\eps.
\end{align*}
\end{proof}

\begin{lemma}\label{lem:kl-bernoulli}
For any $\tau\in(0,\nicefrac12)$ and $\eps\in[0,\tau]$, we have
\begin{align*}
    \kl{\ber(1-\tau)}{\ber(1-\tau+\eps)} \leq 2\eps^2\tau^{-1}.
\end{align*}
\end{lemma}
\begin{proof}
By explicit computation, we have
\begin{align*}
    &\phantom{=}\kl{\ber(1-\tau)}{\ber(1-\tau+\eps)}
    \\&= \tau\log\prns{\frac{\tau}{\tau-\eps}} + (1-\tau)\log\prns{\frac{1-\tau}{1-\tau+\eps}}
    \\&\leq \tau\log\prns{\frac{\tau}{\tau-\eps}} + \tau\log\prns{\frac{\tau}{\tau+\eps}}
    \\&=-\tau\log\prns{ 1-\frac{\eps^2}{\tau^2} }
    \\&\leq 2\eps^2\tau^{-1}.
\end{align*}
The first inequality is because $f(x)=x\log\prns{\frac{x}{x+\eps}}$ is a decreasing function and $1-\tau\geq\tau$.
The second inequality is because $-\log(1-x)\leq 2x$ for $x\in[0,1]$.
\end{proof}

\subsection{Lower bound for CVaR RL}
\rlLowerBound*
\begin{proof}[Proof of \cref{cor:rl-lower-bound}]
Fix any $\tau,A,H$.
Consider an MDP where the states are represented by an $A$-balanced tree with depth $H$ (each node of the tree is a state).
The initial state $s_1$ is the root, and based on the action $a_1$, transits to the $a_1$-th node in the next layer.
The process repeats until we've reached one of the $A^{H-1}$ leaves,
where a reward is given (which also depends on the action taken at the leaf).
There are no rewards until the last step.
The number of states is $S=1+A+...+A^{H-1}$, since the $h$-th layer of the tree has $A^{h-1}$ states.

Since there are no rewards until the last step, running in this MDP reduces to a MAB
with $A^H$ ``arms'' where the ``arms'' are the sequences of actions $a_{1:H}$.
So, by \cref{thm:mab-lower-bound}, for any RL algorithm, there is an MDP constructed this way (with Bernoulli rewards at the end)
such that if $K\geq\sqrt{ \frac{A^H-1}{8\tau} }$, then $\Eb{\rlregret_\tau(K)}\geq\frac{1}{24e}\sqrt{ \frac{(A^H-1)K}{\tau} }$.
The key observation is that $A^H-1 = (A-1)\prns{A^{H-1}+A^{H-2}+\dots+A+1}=(A-1)S$.
This concludes the proof.
\end{proof}

\section{Proofs for \mabalg{}}
For any arm $a\in\Acal$, let $b^\star_a$ denote the $\tau$-th quantile of $R(a)$, so
\begin{align*}
    b^\star_a &= \argmax_{b\in[0,1]}\braces{ b-\tau^{-1}\Eb[R\sim\nu(a)]{(b-R)^+} }
    \\\cvar_\tau(R(a)) &= b^\star_a-\tau^{-1}\Eb[R\sim\nu(a)]{(b^\star_a-R)^+}.
\end{align*}
Let us denote
\begin{align*}
    \mu(b,a)&=\Eb[R\sim\nu(a)]{(b-R)^+},
    \\\wh\mu_k(b,a)&=\frac{1}{N_k(a)}\sum_{i=1}^{k-1}(b-r_i)^+\I{a_i=a}.
\end{align*}
Recall that $a^\star$ is the arm with the highest $\cvar_\tau$.

For any $\delta\in(0,1)$, w.p. at least $1-\delta$, uniform Bernstein implies that for all $b,a$,
\begin{align}
    \abs{ \wh\mu_k(b,a)-\mu(b,a) } \leq \sqrt{ \frac{2\tau\log(2AK/\delta)}{N_k(a)} } + \frac{\log(2AK/\delta)}{N_k(a)}. \label{eq:mab-mu-muhat-concentrate}
\end{align}
Note that our bonus \cref{eq:mab-ucb-bon} is constructed to match the upper bound.
This implies that $\wh\mu_k-\bon_k$ is a pessimistic estimate of $\mu$.
\begin{lemma}[Pessimism]\label{lem:mab-pessimism}
For all $k\in[K]$
\begin{align*}
    \min_{a\in\Acal}\braces{ \wh\mu_k(b^\star_a,a)-\bon_k(a) } \leq \mu(b^\star_{a^\star},a^\star).
\end{align*}
\end{lemma}
\begin{proof}
Fix any $k\in[K]$.
By \cref{eq:mab-mu-muhat-concentrate}, for all $a\in\Acal$,
\begin{align*}
    \wh\mu_k(b^\star_a,a)-\bon_k(a)\leq \mu(b^\star_a,a).
\end{align*}
Hence,
\begin{align*}
    \min_{a\in\Acal}\braces{ \wh\mu_k(b^\star_a,a)-\bon_k(a) }
    &\leq \wh\mu_k(b^\star_{a^\star},a^\star)-\bon_k(a^\star)
    \\&\leq \mu(b^\star_{a^\star},a^\star).
\end{align*}
\end{proof}

\banditUpper*
\begin{proof}[Proof of \cref{thm:bandit upper}]
\begin{align*}
    &\phantom{=}\mabregret_\tau(K)
    \\&= \sum_{k=1}^K\cvar_\tau^\star-\cvar_\tau(R(a_k))
    \\&= \sum_{k=1}^K \braces{b^\star_{a^\star}-\tau^{-1}\mu(b^\star_{a^\star},a^\star)} -\cvar_\tau(\nu(a_k))
    \\&\leq \sum_{k=1}^K \braces{b^\star_{a^\star}-\tau^{-1}\min_{a\in\Acal}\prns{ \wh\mu_k(b^\star_{a^\star},a)-\bon_k(a) } } -\cvar_\tau(\nu(a_k)) \tag{pessimism \cref{lem:mab-pessimism}}
    \\&= \sum_{k=1}^K \max_{a\in\Acal}\braces{b^\star_{a^\star}-\tau^{-1}\prns{ \wh\mu_k(b^\star_{a^\star},a)-\bon_k(a) } } -\cvar_\tau(\nu(a_k))
    \\&\leq K\eps+\sum_{k=1}^K \max_{a\in\Acal}\braces{\wh b_{a,k}-\tau^{-1}\prns{ \wh\mu_k(\wh b_{a,k},a)-\bon_k(a) } } -\cvar_\tau(\nu(a_k)) \tag{$\wh b_{a,k}$ is $\eps$-optimal}
    \\&= K\eps+\sum_{k=1}^K \braces{\wh b_{a_k,k}-\tau^{-1}\prns{ \wh\mu_k(\wh b_{a_k,k},a_k)-\bon_k(a_k) } } -\cvar_\tau(\nu(a_k)) \tag{defn. of $a_k$}
    \\&\leq K\eps+\sum_{k=1}^K \tau^{-1}\bon_k(a_k) + \max_{b\in[0,1]}\braces{b-\tau^{-1}\wh\mu_k(b,a_k) } -\cvar_\tau(\nu(a_k))
    \\&= K\eps+\sum_{k=1}^K \tau^{-1}\bon_k(a_k) + \wh\cvar_\tau(\braces{r_i}_{i\in\Ical_k(a_k)})-\cvar_\tau(\nu(a_k))
    \\&\leq K\eps+\sum_{k=1}^K \sqrt{\frac{2L}{N_k(a)\tau}} + \frac{L}{N_k(a)\tau} + \sqrt{\frac{3L}{N_k(a)\tau}} + \frac{15L}{N_k(a)\tau} \tag{$\cvar_\tau$ concentration \cref{thm:cvar-concentration-main}}
    \\&\leq K\eps+\sum_{k=1}^K \sqrt{\frac{10L}{N_k(a)\tau}} + \frac{16L}{N_k(a)\tau}
    \\&\leq K\eps+\sqrt{10L\tau^{-1}}\cdot\sqrt{ AKL } + 16L\tau^{-1}\cdot A\log(K). \tag{elliptical potential \cref{lem:tabular-elliptical-potential}}
\end{align*}
A technical detail is that $\cvar_\tau$ concentration only applies when $N_k(a)\geq 25L$.
We can trivially bound the total regret of the episodes when $N_k(a)<25L$ by $25AL$.
Also, we remark the concentration step applies since $\braces{r_i}_{i\in\Ical_k(a)}$ are i.i.d. via the tape framework, so we do not need to generalize \cref{thm:cvar-concentration-main} to martingale sequences.
Finally, setting $\eps=\sqrt{\tau^{-1}A/2K}$ renders it lower order.
\end{proof}

\newpage
\section{Proofs for Augmented MDP}\label{sec:augmented-mdp}
We first define the memory-MDP model, where the MDP is also equipped with a memory generator $M_h$, which generates $m_h\sim M_h(s_h,a_h,r_h,\Hcal_h)$.
These memories are stored into the history $\Hcal_h=(s_t,a_t,r_t,m_t)_{t\in[h-1]}$ and may be used by history dependent policies in future time steps.
Concretely, rolling out $\pi$ proceeds as follows: for any $h\in[H]$,
$a_h\sim\pi_h(s_h,\Hcal_h)$, $s_{h+1}\sim P^\star(s_h,a_h)$, $r_h\sim R(s_h,a_h)$ and $m_h\sim M_h(s_h,a_h,r_h,\Hcal_h)$.

We can also extend the above formulation to the augmented MDP, where the state is augmented with $b$ as in \cref{sec:augmented-mdp-bellman-eqs}.
Here, the history is $\Haug_h=(s_t,b_t,a_t,r_t,m_t)_{t\in[h-1]}$.
Let $\Piaug_\Hcal$ represent the set of history dependent policies in this augmented MDP with memory.
Also, recall that $\Piaug$ is the set of Markov, deterministic policies in the augmented MDP.

The $V$ function is defined for these multiple types of policies:
\begin{align*}
    \pi\in\Pi_\Hcal: V^\pi_h(s_h,b_h;\Hcal_h) &= \Eb[\pi]{ \prns{ b_h-\sum_{t=h}^Hr_t }^+\mid s_h,b_h,\Hcal_h }
    \\\rho\in\Piaug: V^\rho_h(s_h,b_h) &= \Eb[\rho]{ \prns{ b_h-\sum_{t=h}^Hr_t }^+\mid s_h,b_h }
    \\\rho\in\Piaug_\Hcal: V^\rho_h(s_h,b_h;\Haug_h) &= \Eb[\rho]{ \prns{ b_h-\sum_{t=h}^Hr_t }^+\mid s_h,b_h,\Haug_h }
\end{align*}

Notice that rolling out $\rho,b$ in the augmented MDP is equivalent to rolling out $\pi^{\rho,b}$ in the original MDP, where
\begin{align*}
    \pi^{\rho,b}_h(s_h,\Hcal_h) = \rho_h(s_h,b-r_1-...-r_{h-1}).
\end{align*}
Thus, it's evident that their $V$ functions should match.
\begin{lemma}\label{lem:augmented-policy-value}
Fix any $\rho\in\Piaug,h\in[H]$, augmented state $(s_h,b_h)$ and history $\Hcal_h$.
Then, we have $V^\rho_h(s_h,b_h)=V_h^{\pi^{\rho,b}}(s_h,b_h;\Hcal_h)$ for $b=b_h+r_1+...+r_{h-1}$.
In particular, we have $V^\rho_1(s_1,\cdot)=V^{\pi^{\rho,b}}_1(s_1,\cdot)$.
\end{lemma}
\begin{proof}
The setting of $b$ in the lemma satisfies $b_h=b-r_1-...-r_{h-1}$.
Therefore, the trajectories of $(\rho,b)$ and $\pi^{\rho,b}$ are exactly coupled.
\end{proof}

We now show the key result of this section.
The theorem shows that the $V^\star,U^\star$ functions defined via the Bellman optimality equations (from \cref{sec:augmented-mdp-bellman-eqs})
correspond to the $V,U$ functions of $\rho^\star$.
Furthermore, the Markov (in augmented state) and deterministic $\rho^\star$ is in fact an optimal policy amongst
all history-dependent policies in the augmented MDP with memory!
This result and our proof is analogous to the ``Markov optimality theorem'' of vanilla RL, \eg, \citep{puterman2014markov}, \citep[Theorem 1.7]{rltheorybookAJKS}.
\begin{theorem}\label{thm:rho-star-optimality}
For all $h$ we have $U^\star_h = U^{\rho^\star}_h$ and $V^\star_h = V^{\rho^\star}_h$.
Furthermore, for all $s_h,b_h,\Haug_h$, we have
\begin{align*}
    V^\star_h(s_h,b_h) = \inf_{\rho\in\Piaug_\Hcal}V^\rho_h(s_h,b_h;\Haug_h).
\end{align*}
In particular, $V^\star_1(s_1,b) = \inf_{\rho\in\Piaug_\Hcal}V^\rho_1(s_1,b)$ for all $b$.
\end{theorem}
\begin{proof}
We first prove the claim that $U^\star_h = U^{\rho^\star}_h$ and $V^\star_h = V^{\rho^\star}_h$.
The base case of $H+1$ is trivial since $V_{H+1}(s,b)=b^+$ everywhere.
For the inductive step, fix any $h\in[H]$ and suppose the claim is true for $h+1$.
Then,
\begin{align*}
    U^{\rho^\star}_h(s_h,b_h,a_h)
    &= \Eb[s_{h+1}\sim P^\star(s_h,a_h),r_h\sim R(s_h,a_h)]{ V^{\rho^\star}_{h+1}(s_{h+1},b_h-r_h) } \tag{Bellman Eqns}
    \\&= \Eb[s_{h+1}\sim P^\star(s_h,a_h),r_h\sim R(s_h,a_h)]{ V^\star_{h+1}(s_{h+1},b_h-r_h) } \tag{IH}
    \\&= U^\star_h(s_h,b_h,a_h). \tag{Bellman Opt. Eqns}
\end{align*}
This proves that $U^\star_h = U^{\rho^\star}_h$.
For $V$,
\begin{align*}
    V^{\rho^\star}_h(s_h,b_h)
    &= \Eb[a_h\sim\rho^\star_h(s_h,b_h)]{U^{\rho^\star}_h(s_h,b_h,a_h)} \tag{Bellman Eqns}
    \\&= \Eb[a_h\sim\rho^\star_h(s_h,b_h)]{U^\star_h(s_h,b_h,a_h)} \tag{above claim}
    \\&= \min_{a_h\in\Acal} U^\star_h(s_h,b_h,a_h) \tag{defn. of $\rho^\star_h$}
    \\&= V^\star_h(s_h,b_h). \tag{Bellman Opt. Eqns}
\end{align*}
Therefore, we've shown that $V^\star_h = V^{\rho^\star}_h$.

We also prove the second claim inductively.
The base case is again trivial since $V_{H+1}(s,b)=b^+$ everywhere.
For the inductive step, fix any $h\in[H]$ and suppose the claim is true for $h+1$.
Now fix any $s_h,b_h$ and $\Haug_h$,
\begin{align*}
    &\inf_{\rho\in\Piaug_\Hcal}V^\rho_h(s_h,b_h;\Haug_h)
    \\&=\inf_{\rho\in\Piaug_\Hcal}\Eb[\rho]{ \prns{b_h-\sum_{t=h}^Hr_t}^+\mid s_h,b_h,\Haug_h }
    \\&\geq\inf_{\rho\in\Piaug_\Hcal}\Eb[a_h,s_{h+1},r_h,m_h]{ \inf_{\rho'\in\Piaug_\Hcal}\Eb[\rho']{\prns{b_h-\sum_{t=h}^Hr_t}^+\mid s_{h+1},b_{h+1},\Haug_{h+1}} }
    \\&= \inf_{\rho\in\Piaug_\Hcal}\Eb[a_h\sim\rho_h(s_h,b_h,\Haug_h)]{\Eb[s_{h+1}\sim P^\star(s_h,a_h),r_h\sim R(s_h,a_h)]{ V_{h+1}^\star(s_{h+1},b_h-r_h) } } \tag{IH}
    \\&= \min_{a\in\Acal}\Eb[s_{h+1}\sim P^\star(s_h,a_h),r_h\sim R(s_h,a_h)]{ V_{h+1}^\star(s_{h+1},b_h-r_h) } \tag{$\star$}
    \\&= V^\star_h(s_h,b_h). \tag{by defn.}
\end{align*}
There are three key steps.
First, the inequality is due to expanding out one step, where $a_h\sim \rho_h(s_h,b_h,\Haug_h),s_{h+1}\sim P^\star(s_h,a_h),r_h\sim R(s_h,a_h),m_h\sim M_h(s_h,a_h,r_h,\Hcal_h)$,
then push the inf for future steps inside the expectation.
Second, the IH invocation is significant as it essentially removes dependence of the memory hallucinations $m_h$.
Third, the step marked with $\star$ is significant since, regardless of the history, the current best action is just to minimize the inner function (which is independent of the history).
We also have $V^\star_h(s_h,b_h)\leq \inf_{\rho\in\Piaug_\Hcal}V^\rho_h(s_h,b_h;\Haug_h)$ since by the first part of the claim, $V^\star_h$ is the value of $\rho^\star\in\Piaug_\Hcal$.
Thus, we've shown $V^\star_h(s_h,b_h)= \inf_{\rho\in\Piaug_\Hcal}V^\rho_h(s_h,b_h;\Haug_h)$.
\end{proof}

As a corollary of the above theorem, we can restrict the policy class to
history-dependent policies on the non-augmented MDP (and without history).
\optimalityAugmentedMarkovPolicies*
\begin{proof}[Proof of \cref{thm:aug_mdp_theorem}]
The first equality is directly from \cref{thm:rho-star-optimality}.
We now prove the second equality. For any $b$,
\begin{align*}
    &\min_{\rho\in\Piaug} V^{\pi^{\rho,b}}_1(s_1,b)
    \\&=\min_{\rho\in\Piaug} V^\rho_1(s_1,b) \tag{\cref{lem:augmented-policy-value}}
    \\&=\min_{\pi\in\Piaug_\Hcal} V^\pi_1(s_1,b) \tag{\cref{thm:rho-star-optimality}}
    \\&\leq\min_{\pi\in\Pi_\Hcal} V^\pi_1(s_1,b)
    \\&\leq\min_{\rho\in\Piaug} V^{\pi^{\rho,b}}_1(s_1,b).
\end{align*}
The last two inequalities is due to considering strictly smaller sets of policies.
Therefore, we have equality throughout, which proves the claim.
\end{proof}

\section{Proofs for \rlalg{}}

\subsection{The high probability good event}
In this section, we derive all the high probability results needed in the remainder of the proof.
Fix any failure probability $\delta\in(0,1)$.
Then, w.p. at least $1-\delta$, for all $h\in[H],k\in[K],s\in\Scal,a\in\Acal$, we have, for all $b\in[0,1],s'\in\Scal$,
\begin{align}
    &\abs{ \prns{\wh P_{k}(s,a)-P^\star(s,a)}^\top \Eb[r_h]{V^\star_{h+1}(\cdot,b-r_h)} } \leq \sqrt{\frac{L}{N_{k}(s,a)}}, \label{eq:hoeffding-V-star}
    \\&\abs{ \prns{\wh P_{k}(s,a)-P^\star(s,a)}^\top \Eb[r_h]{V^\star_{h+1}(\cdot,b-r_h)} } \leq \sqrt{\frac{2\Varb[s'\sim \wh P_{k}(s,a)]{ \Eb[r_h]{V^\star_{h+1}(s',b-r_h)} }L}{N_{k}(s,a)}} + \frac{L}{N_{k}(s,a)}, \label{eq:bernstein-V-star}
    \\&\abs{ \wh P_{k}(s'\mid s,a)-P^\star(s'\mid s,a) }\leq \sqrt{ \frac{2P^\star(s'\mid s,a)L}{N_{k}(s,a)} } + \frac{L}{N_{k}(s,a)}. \label{eq:bernstein-p-hat}
\end{align}
where $r_h\sim R(s,a)$ in the expectations.
\begin{proof}
\cref{eq:hoeffding-V-star} is due to uniform Hoeffding applied to $\Eb[s_{h+1},r_h]{V^\star_{h+1}(s_{h+1},b-r_h)}$, which is 1-Lipschitz in $b$ by \cref{lem:V-star-lipschitz}.
\cref{eq:bernstein-p-hat} is due to standard Bernstein's inequality on the indicator random variable on $(s,a,s')$, \ie, $\I{(s_{h,k},a_{h,k},s_{h+1,k})=(s,a,s')}$.
\cref{eq:bernstein-V-star} is due to uniform empirical Bernstein applied to $\Eb[s_{h+1},r_h]{V^\star_{h+1}(s_{h+1},b-r_h)}$.
In \cref{sec:concentration-lemmas}, we derive and review these uniform results.
\end{proof}
\begin{lemma}\label{lem:V-star-lipschitz}
For any $h\in[H]$ and $s\in\Scal$, $V^\star_h(s,\cdot)$ is $1$-Lipschitz.
\end{lemma}
\begin{proof}
We proceed by induction.
Let $b,b'\in[0,1]$ be arbitrary.
At $h=H+1$, $\abs{V^\star_{H+1}(s,b)-V^\star_{H+1}(s,b')}=\abs{b^+-(b')^+}\leq \abs{b-b'}$ since the ReLU is $1$-Lipschitz.
For the inductive step, fix any $h$ and suppose the claim is true at $h+1$. Then
$\abs{V^\star_h(s,b)-V^\star_h(s,b')}=\abs{ \min_a \Eb[s_{h+1},r_h]{V^\star_{h+1}(s_{h+1},b-r_h)}- \min_a \Eb[s_{h+1},r_h]{V^\star_{h+1}(s_{h+1},b'-r_h)} }\leq \max_a\abs{\Eb[s_{h+1},r_h]{V^\star_{h+1}(s_{h+1},b-r_h)-V^\star_{h+1}(s_{h+1},b'-r_h)}}\leq \abs{b-b'}$, by the IH.
The expectations are over $s_{h+1}\sim P^\star(s,a)$ and $r_h\sim R(s,a)$.
\end{proof}

We now show that the projected error between $\wh P_{k}(s,a)$ and $P^\star$ can be bounded in two ways.
\begin{lemma}\label{lem:projected-error-phat-pstar}
For any $\delta\in(0,1)$, w.p. at least $1-\delta$, we have for all $f:\Scal\to[0,1]$,
\begin{align*}
    &\abs{ \prns{\wh P_{k}(s,a)-P^\star(s,a)}^\top f } \leq \min\braces{8\sqrt{\frac{SL}{N_{k}(s,a)}}, \frac{\Eb[s'\sim P^\star(s,a)]{f(s')}}{H} + \xi_{k}(s,a) },
    \\&\text{where } \xi_{k}(s,a) := \min\braces{1,\frac{2HSL}{N_{k}(s,a)}}.
\end{align*}
\end{lemma}
\begin{proof}
Fix any $f:\Scal\to[0,1]$.
The first bound of $\sqrt{\frac{SL}{N_{k}(s,a)}}$ follows from applying Hoeffding on an $\eps$-net of the space of $f$'s,
\ie, for each $g$ in the net, we have $\abs{ \prns{\wh P_{k}(s,a)-P^\star(s,a)}^\top g }\leq \sqrt{\frac{L}{N_{k}(s,a)}}$.
This $\eps$-net has $\ell_2$ bounded by $\sqrt{S}$. This gives the metric entropy $\log (1+2\sqrt{S}/\eps)^S \approx S\log(S/\eps)$.
Setting $\eps = \frac{1}{HK}$ makes the error lower order, \ie, $\frac{1}{HK}\leq\frac{1}{N_{k}(s,a)}$, which gives the uniform result over all $f$'s.
The detailed proof is in \citep[Lemma 7.2]{rltheorybookAJKS}.

The second bound also appears in \citet{rltheorybookAJKS} as Lemma 7.8.
We prove its proof for completeness:
\begin{align*}
    \abs{ \prns{\wh P_{k}(s,a)-P^\star(s,a)}^\top f }
    &\leq\sum_{s'} \abs{ \wh P_{k}(s'\mid s,a)-P^\star(s'\mid s,a) } f(s')
    \\&\leq\sum_{s'}f(s') \sqrt{ \frac{2P^\star(s'\mid s,a)L}{N_{k}(s,a)} } + \frac{f(s')L}{N_{k}(s,a)} \tag{\cref{eq:bernstein-p-hat}}
    \\&\leq \sqrt{S\frac{\sum_{s'}2P^\star(s'\mid s,a) f^2(s')L}{N_{k}(s,a)}} + \frac{SL}{N_{k}(s,a)} \tag{C-S}
    \\&\leq \frac{SHL}{N_{k}(s,a)} + \frac{\sum_{s'} P^\star(s'\mid s,a) f(s')}{H} + \frac{SL}{N_{k} (s,a)} \tag{AM-GM}.
\end{align*}
Finally, since $\wh P_{k}(s,a)^\top f$ and $P^\star(s,a)^\top f$ are both in $[0,1]$,
a trivial bound is $1$, which is why $\xi_{h,k}$ can be truncated.
\end{proof}

Finally, we also have consequences of Azuma's inequality \cref{lem:mult-azuma}.
W.p. at least $1-\delta$, for all $h\in[H]$,
\begin{align}
&\sum_{k=1}^K\Eb[\wh\rho^k,\wh b_k]{ 2\bon_{h,k}(s_h,b_h,a_h)+\xi_{h,k}(s_h,a_h)\mid\Ecal_k}
\leq 6L + 2\sum_{k=1}^K 2\bon_{h,k}(s_{h,k},b_{h,k},a_{h,k})+\xi_{h,k}(s_{h,k},a_{h,k}), \tag{Azuma 1} \label{eq:azuma-regret-decomp}
\end{align}
where we used the fact that WLOG we truncated the bonus to be at most $1$ (by sentence below \cref{eq:key-bonus-inequality}),
so $\nm{2\bon_{h,k}+\xi_{h,k}}_\infty\leq 3$.
$\Ecal_{k}$ denotes the complete trajectories from episodes $1,2,...,k-1$

For the Bernstein bonus proofs, we'll also need,
\begin{align}
&\sum_{h=1}^H\sum_{k=1}^K\Eb[s'\sim P^\star(s_{h,k},a_{h,k}),r\sim R(s_{h,k},a_{h,k})]{ \prns{ \wh V^\uparrow_{h+1,k}(s',b_{h,k}-r)-\wh V^\downarrow_{h+1,k}(s',b_{h,k}-r) }^2 \mid \Ecal_k,\Hcal_{h,k}} \nonumber
\\&\leq \sqrt{HKL} + \sum_{h=1}^H\sum_{k=1}^K\prns{ \wh V^\uparrow_{h+1,k}(s_{h+1,k},b_{h+1,k})-\wh V^\downarrow_{h+1,k}(s_{h+1,k},b_{h+1,k}) }^2, \tag{Azuma 2} \label{eq:azuma-V-upper-lower-diff}
\end{align}
and
\begin{align}
&\sum_{h=1}^H\sum_{k=1}^K\Eb[s'\sim P^\star(s_{h,k},a_{h,k}),r\sim R(s_{h,k},a_{h,k})]{ \prns{ V^{\wh\rho^k}_{h+1}(s',b_{h,k}-r)-\wh V^\downarrow_{h+1,k}(s',b_{h,k}-r) }^2 \mid \Ecal_k,\Hcal_{h,k}} \nonumber
\\&\leq \sqrt{HKL} + \sum_{h=1}^H\sum_{k=1}^K\prns{ V^{\wh\rho^k}_{h+1}(s_{h+1,k},b_{h+1,k})-\wh V^\downarrow_{h+1,k}(s_{h+1,k},b_{h+1,k}) }^2, \tag{Azuma 3} \label{eq:azuma-V-rho-lower-diff}
\end{align}
where we've used that the envelope is at most $1$ and $b_{h+1,k}=b_{h,k}-r_{h,k}$.
Here, $\Hcal_{h,k}=(s_{t,k},a_{t,k},r_{t,k})_{t\in[h-1]}$ denotes the history before $h$ for the $k$-th episode.
Also, for all $h\in[H]$,
\begin{align}
&\sum_{k=1}^K\Varb[s'\sim P^\star(s_{h,k},a_{h,k})]{ \Eb[r\sim R(s_{h,k},a_{h,k})]{ V^{\wh\rho^k}_{h+1}(s',b_{h,k}-r) } } \nonumber
\\&\leq 2L + 2\sum_{k=1}^K \Eb[\wh\rho^k,\wh b_k]{ \Varb[s'\sim P^\star(s_h,a_h)]{ \Eb[r\sim R(s_h,a_h)]{ V^{\wh\rho^k}_{h+1}(s',b_{h,k}-r) } } \mid \Ecal_{k}}. \tag{Azuma 4} \label{eq:azuma-variance-V-rhok}
\end{align}
Also, for all $h,t\in[H]$ where $t\geq h$,
\begin{align}
    \sum_{k=1}^K\Eb[\wh\rho^k,s_h=s_{h,k},b_h=b_{h,k}]{ 2\bebon_{t,k}(s_t,b_t,a_t)+\xi_{t,k}(s_t,a_t)\mid\Ecal_k}
    &\leq 6L + 2\sum_{k=1}^K2\bebon_{t,k}(s_{t,k},b_{t,k},a_{t,k})+\xi_{t,k}(s_{t,k},a_{t,k}). \tag{Azuma 5} \label{eq:azuma-sum-bonuses-start-at-h}
\end{align}
Finally a standard Azuma also gives, for all $h\in[H]$,
\begin{align}
    \sum_{k=1}^K\Eb[\wh\rho^k,\wh b_k]{ 2\bebon_{h,k}(s_h,b_h,a_h)+\xi_{h,k}(s_h,a_h) \mid\Ecal_k } \leq 3\sqrt{KL} + \sum_{k=1}^K2\bebon_{h,k}(s_{h,k},b_{h,k},a_{h,k})+\xi_{h,k}(s_{h,k},a_{h,k}) \tag{Azuma 6} \label{eq:azuma-standard-for-bonus-sum}
\end{align}

Henceforth, we always condition on the union of these high probability statements to be true.

\subsection{Key lemmas for \rlalg{}}
In general, the bonus should be designed to satisfy for all $h\in[H],k\in[K]$,
\begin{align}
    \forall s,b,a: \abs{ \prns{\wh P_{k}(s,a)-P^\star(s,a)}^\top \Eb[r_h\sim R(s,a)]{V_{h+1}^\star(\cdot,b-r_h) } }
    \leq \bon_{h,k}(s,b,a). \tag{\bon$\bigstar$} \label{eq:key-bonus-inequality}
\end{align}
The bonus only needs to satisfy this inequality for our proofs to work.
WLOG, since the left hand side is the difference between two numbers in $[0,1]$, we can always assume bonus to be truncated by $1$, i.e. has envelope $1$.

We say that pessimism is satisfied at $h\in[H],k\in[K]$ if
\begin{align}
    \forall s,b: \wh V_{h,k}^\downarrow(s,b)\leq V^\star_h(s,b). \tag{Pessimism ($V^\downarrow$)}\label{eq:pessimism}
\end{align}

\begin{lemma}\label{lem:pessimism-inductive-step}
For any $k\in[K],h\in[H]$, suppose \ref{eq:pessimism} holds at $(h+1,k)$ and \ref{eq:key-bonus-inequality} holds at $(h,k)$.
Then \ref{eq:pessimism} holds at $(h,k)$.
\end{lemma}
\begin{proof}
First, we prove pessimism for $\wh U_{h,k}^\downarrow$.
For any $s,b,a$, we have
\begin{align*}
    &\phantom{=}\wh U_{h,k}^\downarrow(s,b,a)-U^\star_h(s,b,a)
    \\&=\wh P_{k}(s,a)^\top\Eb[r_h\sim R(s,a)]{\wh V_{h+1,k}^\downarrow(\cdot,b-r_h)}-\bon_{h,k}(s,b,a)- P^\star(s,a)^\top\Eb[r_h\sim R(s,a)]{V^\star_{h+1}(\cdot,b-r_h)}
    \\&\leq\prns{\wh P_{k}(s,a)-P^\star(s,a)}^\top \Eb[r_h\sim R(s,a)]{V^\star_{h+1}(\cdot,b-r_h)}-\bon_{h,k}(s,b,a) \tag{IH}
    \\&\leq 0. \tag*{by \ref{eq:key-bonus-inequality}}
\end{align*}
To complete the proof,
if $\wh V_{h,k}^\downarrow(s,b)=0$, it is trivially pessimistic, and if not,
\begin{align*}
    \wh V_{h,k}^\downarrow(s,b)-V^\star_h(s,b)
    &=\min_a\braces{\wh U_{h,k}^\downarrow(s,b,a)}-\min_a\braces{U^\star_h(s,b,a)}
    \\&\leq \max_a\braces{\wh U_{h,k}^\downarrow(s,b,a)-U^\star_h(s,b,a)}
    \\&\leq 0.
\end{align*}
\end{proof}

Remarkably, we show Simulation lemma also holds for \rlalg{}.
Here, it is also required that the bonus satisfies \ref{eq:key-bonus-inequality}.
As for notation, recall $\Ecal_k$ represents the episodes before and not including $k$.
\begin{lemma}[Simulation Lemma]\label{lem:simulation}
Fix any $k\in[K],t\in[H]$. Then, for all $s_t,b_t$, we have
\begin{align}
    &V_t^{\wh\rho^k}(s_t,b_t)-\wh V_{t,k}^\downarrow(s_t,b_t) \nonumber
    \\&\leq \sum_{h=t}^H \Eb[\wh\rho^k,s_t,b_t]{\bon_{h,k}(s_h,b_h,a_h) +
    \prns{P^\star(s_h,a_h)-\wh P_{k}(s_h,a_h) }^\top \wh V_{h+1,k}^\downarrow(\cdot,b_{h+1}) \mid \Ecal_k}. \label{eq:simulation-1-IH}
\end{align}
Furthermore, if we assume that \ref{eq:key-bonus-inequality} holds, then for all $s_t,b_t$,
\begin{align}
    V_t^{\wh\rho^k}(s_t,b_t)-\wh V_{t,k}^\downarrow(s_t,b_t)\leq
    \sum_{h=t}^H \prns{1+1/H}^{h-t} \Eb[\wh\rho^k,s_t,b_t]{2\bon_{h,k}(s_h,b_h,a_h)+\xi_{h,k}(s_h,a_h) \mid \Ecal_k }. \label{eq:simulation-2-IH}
\end{align}
In particular,
\begin{align*}
    V_1^{\wh\rho^k}(s_1,b)-\wh V_{1,k}^\downarrow(s_1,b) \leq e\sum_{h=1}^H \Eb[\wh\rho^k]{ 2\bon_{h,k}(s_h,b_h,a_h) + \xi_{h,k}(s_h,a_h)\mid\Ecal_k}.
\end{align*}
\end{lemma}
\begin{proof}
Fix any $k$ and $t$. All expectations in the proof will condition on $\Ecal_k$; this way, the randomness is only from rolling in the policy $\wh\rho^k$ and not over any of the prior episodes.

\textbf{First claim: } Let's first show \cref{eq:simulation-1-IH} by induction.
The base case is $t=H+1$, we have $V_{H+1}^{\wh\rho^k}(s,b)=\wh V_{H+1,k}^\downarrow(s,b)=b^+$, so $V_{H+1}^{\wh\rho^k}-\wh V_{H+1,k}=0$.

For the inductive step, fix any $t\leq H$ and suppose \cref{eq:simulation-1-IH} is true for $t+1$.
Let us denote $a_t = \wh\rho^k_t(s_t,b_t) = \argmin_a\wh U_{t,k}(s_t,b_t,a)$, so $\wh V_{t,k}^\downarrow(s_t,b_t) = \max\braces{\wh U_{t,k}^\downarrow(s_t,b_t,a_t), 0} \geq \wh U_{t,k}^\downarrow(s_t,b_t,a_t)\geq \wh P_{t,k}(s_t,a_t)^\top \Eb[r_t]{\wh V_{t+1,k}^\downarrow(\cdot,b_{t+1})}-\bon_{t,k}(s_t,b_t,a_t)$, where $b_{t+1}=b_t-r_t$ is the random next budget.
So, we have
\begin{align*}
    &V_{t}^{\wh\rho^k}(s_t,b_t) - \wh V_{t,k}^\downarrow(s_t,b_t)
    \\&\leq U_t^{\wh\rho^k}(s_t,b_t,a_t)- \wh U_{t,k}^\downarrow(s_t,b_t,a_t) \nonumber
    \\&= \bon_{t,k}(s_t,b_t,a_t) - \wh P_{t,k}(s_t,a_t)^\top\Eb[r_t]{\wh V_{t+1,k}^\downarrow(\cdot,b_{t+1})} + P^\star_t(s_t,a_t)^\top \Eb[r_t]{V_{t+1,k}^{\wh\rho^k}(\cdot,b_{t+1})} \nonumber
    \\&= \bon_{t,k}(s_t,b_t,a_t) + \prns{P^\star_t(s_t,a_t)-\wh P_{t,k}(s_t,a_t)}^\top\Eb[r_t]{\wh V_{t+1,k}^\downarrow(\cdot,b_{t+1})}
    \\&+ P^\star_t(s_t,a_t)^\top\Eb[r_t]{V_{t+1}^{\wh\rho^k}(\cdot,b_{t+1}) - \wh V_{t+1,k}^\downarrow(\cdot,b_{t+1})}
    \\&\leq \bon_{t,k}(s_t,b_t,a_t) + \prns{P^\star_t(s_t,a_t)-\wh P_{t,k}(s_t,a_t)}^\top\Eb[r_t]{\wh V_{t+1,k}^\downarrow(\cdot,b_{t+1})}
    \\&+ \Eb[s_{t+1}\sim P^\star_t(s_t,a_t)]{ \sum_{h=t+1}^H \Eb[\wh\rho^k,s_{t+1},b_{t+1}]{ \bon_{h,k}(s_h,b_h,a_h) + \prns{P^\star(s_h,a_h)-\wh P_{k}(s_h,a_h) }^\top\wh V_{h+1,k}^\downarrow(\cdot,b_{h+1}) } } \tag{IH}
    \\&= \sum_{h=t}^H \Eb[\wh\rho^k,s_t,b_t]{\bon_{h,k}(s_h,b_h,a_h) + \prns{P^\star(s_h,a_h)-\wh P_{k}(s_h,a_h) }^\top \wh V_{h+1,k}^\downarrow(\cdot,b_{h+1}) }.
\end{align*}
This concludes the proof for the first claim.

\textbf{Second claim: }
Now let us show \cref{eq:simulation-2-IH} by induction.
The base case at $t=H+1$ is same as the first claim.
For the inductive step, fix any $t\leq H$ and suppose \cref{eq:simulation-2-IH} is true for $t+1$. Then, continuing from the line before invoking the IH of the first claim, we have
\begin{align*}
    &V_{t}^{\wh\rho^k}(s_t,b_t) - \wh V_{t,k}^\downarrow(s_t,b_t)
    \\&\leq \bon_{t,k}(s_t,b_t,a_t) + \prns{P^\star_t(s_t,a_t)-\wh P_{t,k}(s_t,a_t)}^\top\Eb[r_t]{\wh V_{t+1,k}^\downarrow(\cdot,b_{t+1})}
    \\&+ P^\star_t(s_t,a_t)^\top\Eb[r_t]{\prns{V_{t+1}^{\wh\rho^k}(\cdot,b_{t+1}) - \wh V_{t+1,k}^\downarrow(\cdot,b_{t+1})}}
    \\&= \bon_{t,k}(s_t,b_t,a_t) + \prns{P^\star_t(s_t,a_t)-\wh P_{t,k}(s_t,a_t)}^\top\Eb[r_t]{\wh V_{t+1,k}^\downarrow(\cdot,b_{t+1})-V_{t+1}^\star(\cdot,b_{t+1})}
    \\&+ \prns{P^\star_t(s_t,a_t)-\wh P_{t,k}(s_t,a_t)}^\top\Eb[r_t]{V_{t+1}^\star(\cdot,b_{t+1})}
    + P^\star_t(s_t,a_t)^\top \Eb[r_t]{V_{t+1}^{\wh\rho^k}(\cdot,b_{t+1}) - \wh V_{t+1,k}^\downarrow(\cdot,b_{t+1})}
    \\&\leq \bon_{t,k}(s_t,b_t,a_t) + \xi_{t,k}(s_t,a_t) + \frac{1}{H} P^\star_t(s_t,a_t)^\top \Eb[r_t]{V_{t+1}^\star(\cdot,b_{t+1})-\wh V_{t+1,k}^\downarrow(\cdot,b_{t+1})} \tag{\cref{lem:projected-error-phat-pstar}}
    \\&+ \bon_{t,k}(s_t,b_t,a_t) + P^\star_t(s_t,a_t)^\top \Eb[r_t]{V_{t+1}^{\wh\rho^k}(\cdot,b_{t+1}) - \wh V_{t+1,k}^\downarrow(\cdot,b_{t+1})} \tag{premise (\ref{eq:key-bonus-inequality})}
    \\&\leq 2\bon_{t,k}(s_t,b_t,a_t) + \xi_{t,k}(s_t,a_t) + (1+1/H) P^\star_t(s_t,a_t)^\top\Eb[r_t]{V_{t+1}^{\wh\rho^k}(\cdot,b_{t+1}) - \wh V_{t+1,k}^\downarrow(\cdot,b_{t+1})} \tag{$V^\star\leq V^{\wh\rho^k}$}
    \\&\leq 2\bon_{t,k}(s_t,b_t,a_t) + \xi_{t,k}(s_t,a_t)
    \\&+ (1+1/H)\sum_{h=t+1}^H(1+1/H)^{h-t-1}\Eb[\wh\rho^k,s_t,b_t]{ 2\bon_{h,k}(s_h,b_h,a_h)+\xi_{h,k}(s_h,a_h) }. \tag{IH}
\end{align*}
This completes the inductive proof.
Observing that $(1+1/H)^H\leq\exp(1/H)^H = e$ gives the corollary.
\end{proof}

\subsection{\rlalg{} with Hoeffding Bonus}

The Hoeffding bonus $\hobon_{h,k}(s,a)$ defined in \ref{eq:hoeffding-bonus-def} satisfies the
crucial bonus requirement \ref{eq:key-bonus-inequality} by the uniform Hoeffding's inequality result of \cref{eq:hoeffding-V-star}.
Thus, we have pessimism for all $k,h$ with the Hoeffding bonus.

\hoeffingBonusRegret*
\begin{proof}[Proof of \cref{thm:hoeffding-bonus-regret}]
Let $R(\rho^k,\wh b_k)$ denote the distribution of returns from rolling in $\wh\rho^k$ starting at $\wh b_k$.
For any $k$, we have
\begin{align}
    \cvar_\tau(R(\wh\rho^k,\wh b_k))
    &= \max_{b\in[0,1]}\braces{ b-\tau^{-1}\Eb[\wh\rho^k,\wh b_k]{\prns{b-\sum_{t\in[H]} r_t}^+} } \nonumber
    \\&\geq \wh b_k-\tau^{-1}\Eb[\wh\rho^k,\wh b_k]{\prns{\wh b_k-\sum_{t\in[H]} r_t}^+} \nonumber
    \\&= \wh b_k-\tau^{-1}V_1^{\wh\rho^k}(s_1,\wh b_k). \label{eq:cvar-of-rho-hat-lower-bound}
\end{align}
Therefore,
\begin{align*}
    &\phantom{=}\rlregret_\tau(K)
    \\&=\sum_{k=1}^K \cvar_\tau^\star-\cvar_\tau(R(\wh\rho^k,\wh b_k))
    \\&=\sum_{k=1}^K \braces{b^\star-\tau^{-1}V_1^\star(s_1,b^\star)}-\cvar_\tau(R(\wh\rho^k,\wh b_k))
    \\&\leq\sum_{k=1}^K \braces{b^\star-\tau^{-1}\wh V_{1,k}^\downarrow(s_1,b^\star)}-\cvar_\tau(R(\wh\rho^k,\wh b_k)) \tag{\ref{eq:pessimism}}
    \\&\leq\sum_{k=1}^K \braces{\wh b_k-\tau^{-1}\wh V_{1,k}^\downarrow(s_1,\wh b_k)}-\braces{\wh b_k-\tau^{-1}V^{\wh\rho^k}_1(s_1,\wh b_k)} \tag{defn. of $\wh b_k$ and \cref{eq:cvar-of-rho-hat-lower-bound}}
    \\&=\tau^{-1}\sum_{k=1}^K \prns{ V_1^{\wh\rho^k}(s_1,\wh b_k)-\wh V_{1,k}^\downarrow(s_1,\wh b_k) }
    \\&\leq e\tau^{-1}\sum_{(h,k)\in[H]\times[K]}\Eb[\wh\rho^k,\wh b_k]{ 2\hobon_{h,k}(s_h,a_h)+\xi_{h,k}(s_h,a_h)\mid\Ecal_k} \tag{Simulation \cref{lem:simulation}}
    \\&\leq 6e\tau^{-1}HL+2e\tau^{-1}\sum_{(h,k)\in[H]\times[K]}2\hobon_{h,k}(s_{h,k},a_{h,k})+\xi_{h,k}(s_{h,k},a_{h,k}) \tag{\cref{eq:azuma-regret-decomp}}
    \\&\leq 6e\tau^{-1}HL+2e\tau^{-1}\prns{ 2\sqrt{L}\cdot\sqrt{SAHKL} + 2HSL\cdot SA\log(K)} \tag{elliptical potential \cref{lem:tabular-elliptical-potential}}
    \\&\leq 4e\tau^{-1}\sqrt{SAHK}L + 10e\tau^{-1}S^2AHL^2,
\end{align*}
which concludes the proof.
\end{proof}

\subsection{\rlalg{} with Bernstein Bonus}
When running \cref{alg:ucbvi-cvar} with the Bernstein bonus $\bebon_{h,k}(s,b,a)$ (\cref{eq:bernstein-bonus-def}), we need to also show that $\wh V^\uparrow_{h,k}$ are optimistic for $V^\star_h$.
We say that optimism is satisfied at $(h,k)\in[H]\times[K]$ if
\begin{align}
    \forall s,b: V^\star_h(s,b)\leq \wh V^\uparrow_{h,k}(s,b). \tag{Optimism ($V^\uparrow$)}\label{eq:optimism}
\end{align}
\begin{lemma}\label{lem:optimism-inductive-step}
For any $k\in[K],h\in[H]$, suppose \ref{eq:optimism} holds at $(h+1,k)$ and \ref{eq:key-bonus-inequality} holds at $(h,k)$.
Then \ref{eq:optimism} holds at $(h,k)$.
\end{lemma}
\begin{proof}
First, we prove optimism for $\wh U^\uparrow_{h,k}$. For any $s,b,a$ we have
\begin{align*}
    &\phantom{=}U^\star_h(s,b,a)-\wh U^\uparrow_{h,k}(s,b,a)
    \\&=P^\star(s,a)^\top\Eb[r_h]{V^\star_{h+1}(\cdot,b-r_h)} - \wh P_{k}(s,a)^\top\Eb[r_h]{\wh V^\uparrow_{h+1,k}(\cdot,b-r_h)}-\bon_{h,k}(s,b,a)
    \\&\leq \prns{ P^\star(s,a) - \wh P_{k}(s,a)}^\top\Eb[r_h]{ V^\star_{h+1}(\cdot,b-r_h) } -\bon_{h,k}(s,b,a) \tag{IH}
    \\&\leq 0. \tag*{by \cref{eq:key-bonus-inequality}}
\end{align*}
To complete the proof, if $\wh V^\uparrow(s,b)=1$, then it is trivially optimistic, and if not
\begin{align*}
    V^\star_h(s,b)-\wh V^\uparrow_{h,k}(s,b)
    &= \min_a U^\star_h(s,b,a)-\wh U^\uparrow_{h,k}(s,b,\wh\rho^k(s,b))
    \\&\leq U^\star_h(s,b,\wh\rho^k(s,b))-\wh U^\uparrow_{h,k}(s,b,\wh\rho^k(s,b))
    \\&\leq 0.
\end{align*}
\end{proof}

\begin{lemma}\label{lem:bernstein-bonus-pessimism-lemma}
For any $(h,k)\in[H]\times[K]$, if \ref{eq:pessimism} and \ref{eq:optimism} both hold at $(h+1,k)$,
then \ref{eq:key-bonus-inequality} holds at $(h,k)$ for the Bernstein bonus $\bebon$.
\end{lemma}
\begin{proof}
Recall that uniform empirical Bernstein \cref{eq:bernstein-V-star} gave us the following inequality: for all $s,a$ and $b$,
\begin{align*}
    &\abs{ \prns{\wh P_{k}(s,a)-P^\star(s,a)}^\top\Eb[r_h]{V^\star_{h+1}(\cdot,b-r_h)} }
    \leq \sqrt{\frac{2\Varb[s'\sim \wh P_{k}(s,a)]{ \Eb[r_h]{V^\star_{h+1}(s',b-r_h)} }L}{N_{k}(s,a)}} + \frac{L}{N_{k}(s,a)}. \tag*{\cref{eq:bernstein-V-star} revisited.}
\end{align*}
Now apply the useful triangle inequality of variances $\sqrt{\Var(X)}\leq \sqrt{\Var(Y)}+\sqrt{\Var(X-Y)}$ \citep[Eqn. 51]{zanette2019tighter},
\begin{align*}
    \sqrt{\Varb[s'\sim \wh P_{k}(s,a)]{ \Eb[r_h]{V^\star_{h+1}(s',b-r_h)}}}
    &\leq \sqrt{\Varb[s'\sim \wh P_{k}(s,a)]{ \Eb[r_h]{\wh V^\downarrow_{h+1}(s',b-r_h)}}}
    \\&+ \sqrt{\Varb[s'\sim\wh P_{k}(s,a)]{ \Eb[r_h]{V_{h+1}^\star(s',b-r_h)-\wh V_{h+1,k}^\downarrow(s',b-r_h)} }}.
\end{align*}
The first term is in the bonus. The second term is bounded by the correction term of the bonus as follows,
\begin{align*}
    &\Varb[s'\sim\wh P_{k}(s,a)]{ \Eb[r_h]{V_{h+1}^\star(s',b-r_h)-\wh V_{h+1,k}^\downarrow(s',b-r_h)} }
    \\&\leq\Eb[s'\sim\wh P_{k}(s,a)]{ \prns{ \Eb[r_h]{ V_{h+1}^\star(s',b-r_h)-\wh V_{h+1,k}^\downarrow(s',b-r_h) } }^2 }
    \\&\leq\Eb[s'\sim\wh P_{k}(s,a),r\sim R(s,a)]{ \prns{ V_{h+1}^\star(s',b-r)-\wh V_{h+1,k}^\downarrow(s',b-r) }^2 } \tag{Jensen}
    \\&\leq \Eb[s'\sim\wh P_{k}(s,a),r\sim R(s,a)]{ \prns{\wh V_{h+1,k}^\uparrow(s',b-r)-\wh V_{h+1,k}^\downarrow(s',b-r) }^2 }. \tag{premise: $\wh V_{h+1,k}^\downarrow\leq V_{h+1}^\star\leq \wh V_{h+1,k}^\uparrow$}
\end{align*}
This upper bound is a part of the Bernstein bonus.
Thus, we've shown that $\bebon$ dominates the error, and so \ref{eq:key-bonus-inequality} is satisfied at $(h,k)$.
\end{proof}

A key corollary of \cref{lem:bernstein-bonus-pessimism-lemma,lem:pessimism-inductive-step,lem:optimism-inductive-step} is
that we have \ref{eq:pessimism} and \ref{eq:optimism} \emph{for all} $(h,k)\in[H]\times[K]$ with the Bernstein bonus.
Indeed, for any $k$, we first apply \cref{lem:bernstein-bonus-pessimism-lemma} to get that \ref{eq:key-bonus-inequality} is satisfied at $(H,k)$ (as optimism/pessimism trivially holds at $H+1$).
Then, apply \cref{lem:pessimism-inductive-step,lem:optimism-inductive-step} to get \ref{eq:pessimism} and \ref{eq:optimism} at $(H,k)$.
Then, apply \cref{lem:bernstein-bonus-pessimism-lemma} to get that \ref{eq:key-bonus-inequality} is satisfied at $(H-1,k)$.
Continue in this fashion until we've shown \ref{eq:key-bonus-inequality},~\ref{eq:pessimism} and \ref{eq:optimism} for all $h\in[H]$.
\begin{theorem}
The Bernstein bonus satisfies \ref{eq:key-bonus-inequality},~\ref{eq:pessimism} and \ref{eq:optimism} at all $(h,k)\in[H]\times[K]$.
\end{theorem}

We now prove the regret guarantee for the Bernstein bonus.
The main body of the proof for \cref{thm:bernstein-bonus-regret} and \cref{thm:bernstein-bonus-regret2} will be the same.
The proofs will only diverge at the end when bounding the sum of variances, where we invoke \cref{ass:continuous-densities}.
\bernsteinBonusRegret*
\bernsteinBonusRegretTwo*
\begin{proof}[Proof of \cref{thm:bernstein-bonus-regret} and \cref{thm:bernstein-bonus-regret2}]
Following the same initial steps as \cref{thm:hoeffding-bonus-regret}, we have
\begin{align*}
    \rlregret_\tau(K)\leq 6e\tau^{-1}HL + 2e\tau^{-1}\sum_{(h,k)\in[H]\times[K]}2\bebon_{h,k}(s_{h,k},b_{h,k},a_{h,k}) + \xi_{h,k}(s_{h,k},a_{h,k}).
\end{align*}
The proof boils down to bounding the sum.

\textbf{Logarithmic-in-$K$ terms: }
First, note that any $\Ocal(1/N_{k}(s,a))$ term will scale logarithmically in $K$.
This includes the $\xi_{h,k}(s,a)$ term, as well as a $\frac{2L}{N_{k}(s,a)}$ from the bonus.
Thus,
\begin{align*}
    \sum_{(h,k)\in[H]\times[K]}\frac{2L+2HSL}{N_{k}(s_{h,k},a_{h,k})} \leq 4HSL\cdot SA\log(K) = 4S^2AHL^2. \tag{\cref{lem:tabular-elliptical-potential}}
\end{align*}

\textbf{The variance correction term of the bonus: }
\begin{align}
    &\sum_{(h,k)\in[H]\times[K]}\sqrt{ \frac{\Eb[s'\sim\wh P_{k}(s_{h,k},a_{h,k}),r\sim R(s_{h,k},a_{h,k})]{ \prns{ \wh V^\uparrow_{h+1,k}(s',b_{h,k}-r) - \wh V^\downarrow_{h+1,k}(s',b_{h,k}-r) }^2}}{N_{k}(s_{h,k},a_{h,k})} }. \label{eq:goal-bonus-correction-term}
\end{align}
First, apply a Cauchy Schwarz. The $\sum_{h,k}\frac{1}{N_{k}(s_{h,k},a_{h,k})}$ term is at most $SAL$ by elliptical potential \cref{lem:tabular-elliptical-potential}.
Then, translate $\wh P$ to $P^\star$ via \cref{lem:projected-error-phat-pstar} to get
\begin{align*}
&\leq\sqrt{\begin{aligned}
SAL \Bigg( &\sum_{(h,k)\in[H]\times[K]}8\sqrt{\frac{SL}{N_{k}(s_{h,k},a_{h,k})}}
\\+ &\sum_{(h,k)\in[H]\times[K]} \Eb[s'\sim P^\star(s_{h,k},a_{h,k}),r\sim R(s_{h,k},a_{h,k})]{\prns{ \wh V^\uparrow_{h+1,k}(s',b_{h,k}-r) - \wh V^\downarrow_{h+1,k}(s',b_{h,k}-r)\mid \Ecal_k,\Hcal_{h,k} }^2 } \Bigg)
\end{aligned}
}
\intertext{Then, switch to the empirical $s,b$ using \cref{eq:azuma-V-upper-lower-diff},}
&\leq \sqrt{SAL \prns{ 8\sqrt{SL}\sqrt{SAHKL} + \sqrt{HKL} + \sum_{(h,k)\in[H]\times[K]} \prns{ \wh V^\uparrow_{h+1,k}(s_{h+1,k},b_{h+1,k}) - \wh V^\downarrow_{h+1,k}(s_{h+1,k},b_{h+1,k}) }^2 } }.
\end{align*}
For the sum term, since each term is at most $1$, we have
\\$\prns{\wh V^\uparrow_{h,k}(s_{h,k},b_{h,k}) - \wh V^\downarrow_{h,k}(s_{h,k},b_{h,k})}^2\leq \prns{\wh V^\uparrow_{h,k}(s_{h,k},b_{h,k}) - \wh V^\downarrow_{h,k}(s_{h,k},b_{h,k})}$.
Then, applying a simulation-like \cref{lem:simulation-on-confidence-interval} to each summand bounds the sum by,
\begin{align*}
    &\sum_{h=2}^H\sum_{k=1}^K \sum_{t=h}^H \Eb[\wh\rho^k,s_h=s_{h,k},b_h=b_{h,k}]{ 2\bebon_{t,k}(s_t,b_t,a_t)+\xi_{t,k}(s_t,a_t)\mid\Ecal_k } \tag{simulation-like \cref{lem:simulation-on-confidence-interval}} \nonumber
    \\&\leq 6H^2L + 2\sum_{h=2}^H\sum_{t=h}^H\sum_{k=1}^K 2\bebon_{t,k}(s_{t,k},b_{t,k},a_{t,k}) + \xi_{t,k}(s_t,a_t) \tag{\cref{eq:azuma-sum-bonuses-start-at-h}}
    \\&\leq 6H^2L + 2H\sum_{t=1}^H\sum_{k=1}^K 2\bebon_{t,k}(s_{t,k},b_{t,k},a_{t,k}) + \xi_{t,k}(s_t,a_t).
\end{align*}
Now, we can loosely bound each Bernstein bonus by $2\sqrt{\frac{2L}{N_{k}(s,a)}}+\frac{L}{N_{k}(s,a)}$, so by elliptical potential \cref{lem:tabular-elliptical-potential},
\begin{align}
    \sum_{t=1}^H\sum_{k=1}^K 2\bebon_{t,k}(s_{t,k},b_{t,k},a_{t,k}) + \xi_{t,k}(s_t,a_t)
    &\leq 4\sqrt{2L}\cdot\sqrt{SAHKL} + (L+2HSL)\cdot SA\log(K)\nonumber
    \\&\leq 6\sqrt{SAHK}L + 3S^2AHL^2. \label{eq:loose-bound-bernstein-bonus-sum}
\end{align}
Therefore, we've shown that
\begin{align}
    \sum_{h=2}^H\sum_{k=1}^K \sum_{t=h}^H \Eb[\wh\rho^k,s_h=s_{h,k},b_h=b_{h,k}]{ 2\bebon_{t,k}(s_t,b_t,a_t)+\xi_{t,k}(s_t,a_t)\mid\Ecal_k }
    \leq 12\sqrt{SAH^3K}L + 12S^2AH^2L^2. \label{eq:double-H-sum-of-expected-bonus}
\end{align}
Combining everything together, we have
\begin{align*}
\text{\cref{eq:goal-bonus-correction-term}}
&\leq \sqrt{ SAL\prns{ 9S\sqrt{AHK}L + 12\sqrt{SAH^3K}L + 12S^2AH^2L^2 } }
\\&\leq 5SAHK^{1/4}L + 4S^{3/2}AHL^2,
\end{align*}
which is lower order in $K$ (the dominant term should scale as $K^{1/2}$).

\textbf{Bounding the empirical variance term:}
We now shift our focus to the variance term of the bonus,
\begin{align}
\sum_{(h,k)\in[H]\times[K]}\sqrt{ \frac{\Varb[s'\sim\wh P_{k}(s_{h,k},a_{h,k})]{ \Eb[r\sim R(s_{h,k},a_{h,k})]{ \wh V^\downarrow_{h+1,k}(s',b_{h,k}-r) } }}{N_{k}(s_{h,k},a_{h,k})} }. \label{eq:goal-bonus-variance-term}
\end{align}
The key idea here is to apply a sequential Law of Total Variance \cref{lem:law-total-variance}, but to do so, we must first convert $\sqrt{\Varb[s'\sim\wh P_{k}(s_{h,k},a_{h,k})]{ \Eb[r_h]{\wh V^\downarrow_{h+1,k}(s',b_{h,k}-r_h) } }}$ to $\sqrt{\Varb[s'\sim P_{h}^\star(s_{h,k},a_{h,k})]{ \Eb[r_h]{ V^{\wh\rho^k}_{h+1}(s',b_{h,k}-r_h) } \mid\Ecal_k}}$.
So we need to bound the difference term, \ie,
\begin{small}
\begin{align}
\textstyle&\sum_{(h,k)\in[H-1]\times[K]}\sqrt{\frac{\Varb[s'\sim\wh P_{k}(s_{h,k},a_{h,k})]{ \Eb[r_h]{\wh V^\downarrow_{h+1,k}(s',b_{h,k}-r_h) } }}{N_{k}(s_{h,k},a_{h,k})}}-\sqrt{\frac{\Varb[s'\sim P_{h}^\star(s_{h,k},a_{h,k})]{ \Eb[r_h]{V^{\wh\rho^k}_{h+1}(s',b_{h,k}-r_h)} \mid \Ecal_k}}{N_{k}(s_{h,k},a_{h,k})}} \label{eq:bernstein-switch-to-population-variance-diff}
\intertext{Switch the empirical variance to the (conditional) population one, which incurs a $\sum_{h,k}2\sqrt{\frac{L}{N_{k}(s_{h,k},a_{h,k})}}$ term (\cref{sec:concentration-lemmas}).
Then, use $\sqrt{\Var(X+Y)}\leq\sqrt{\Var(X)}+\sqrt{\Var(Y)}$ (Eqn 51 of \citep{zanette2019tighter}) to get, }
\textstyle&\leq\sum_{(h,k)\in[H-1]\times[K]} \frac{2\sqrt{L}}{N_{k}(s_{h,k},a_{h,k})} + \sum_{(h,k)\in[H-1]\times[K]}\sqrt{\frac{\Varb[s'\sim P^\star(s_{h,k},a_{h,k})]{ \Eb[r_h]{ \wh V^\downarrow_{h+1,k}(s',b_{h+1,k}) - V^{\wh\rho^k}_{h+1}(s',b_{h+1,k}) } \mid \Ecal_k}}{N_{k}(s_{h,k},a_{h,k})}} \nonumber
\textstyle\\&\leq 2\sqrt{L}\cdot SA\log(K) + \sqrt{SAL \sum_{(h,k)\in[H-1]\times[K]} \Varb[s'\sim P^\star(s_{h,k},a_{h,k})]{ \Eb[r_h]{ V^{\wh\rho^k}_{h+1}(s',b_{h,k}-r_h)-\wh V^\downarrow_{h+1,k}(s',b_{h,k}-r_h)}\mid \Ecal_k} }, \nonumber
\end{align}    
\end{small}
where the second inequality is due to elliptical potential \cref{lem:tabular-elliptical-potential} and Cauchy-Schwarz.
Focusing on the sum inside the square root,
\begin{align*}
&\sum_{(h,k)\in[H-1]\times[K]} \Varb[s'\sim P^\star(s_{h,k},a_{h,k})]{ \Eb[r_h]{ V^{\wh\rho^k}_{h+1}(s',b_{h,k}-r_h)-\wh V^\downarrow_{h+1,k}(s',b_{h,k}-r_h)}\mid \Ecal_k}
\\&\leq \sum_{(h,k)\in[H-1]\times[K]} \Eb[s',r_h\sim (P^\star\circ R)(s_{h,k},a_{h,k})]{ \prns{ V^{\wh\rho^k}_{h+1}(s',b_{h,k}-r_h)-\wh V^\downarrow_{h+1,k}(s',b_{h,k}-r_h)}^2 \mid \Ecal_k} \tag{Jensen}
\\&\leq \sqrt{HKL} + \sum_{(h,k)\in[H-1]\times[K]} \prns{ V^{\wh\rho^k}_{h+1}(s_{h+1,k},b_{h+1,k})-\wh V^\downarrow_{h+1,k}(s_{h+1,k},b_{h+1,k})}^2 \tag{\cref{eq:azuma-V-rho-lower-diff}}
\\&\leq \sqrt{HKL} + \sum_{(h,k)\in[H-1]\times[K]} \prns{ V^{\wh\rho^k}_{h+1}(s_{h+1,k},b_{h+1,k})-\wh V^\downarrow_{h+1,k}(s_{h+1,k},b_{h+1,k})} \tag{r.v. is in $[0,1]$}
\\&\leq \sqrt{HKL} + \sum_{(h,k)\in[H-1]\times[K]} \sum_{t=h}^H\Eb[\wh\rho^k,s_h=s_{h,k},b_h=b_{h,k}]{ 2\bebon_{t,k}(s_t,b_t,a_t)+\xi_{t,k}(s_t,a_t)} \tag{simulation lemma \cref{lem:simulation}}
\\&\leq \sqrt{HKL} + 12\sqrt{SAH^3K}L + 12S^2AH^2L^2. \tag{\cref{eq:double-H-sum-of-expected-bonus}}
\end{align*}
Combining the steps, we have shown that the switching cost to the population variance is at most
\begin{align*}
    \text{\cref{eq:bernstein-switch-to-population-variance-diff}}\leq 2SAL^2 + \sqrt{SAL\prns{ 13\sqrt{SAH^3KL} + 12S^2AH^2L^2 }} \leq 4SAHK^{1/4}L + 6S^{3/2}AHL^2
\end{align*}
which is again lower order.

\textbf{(Key step) Bounding the dominant (population) variance term: }
\begin{align}
\sum_{(h,k)\in[H]\times[K]}\sqrt{ \frac{\Varb[s'\sim P^\star_{h}(s_{h,k},a_{h,k})]{ \Eb[r\sim R(s_{h,k},a_{h,k})]{V^{\wh\rho^k}_{h+1}(s',b_{h,k}-r)} \mid\Ecal_k }}{N_{k}(s_{h,k},a_{h,k})} }, \label{eq:goal-bonus-dominant-term}
\end{align}
First apply a Cauchy-Schwarz (as usual) and then the law of total variance,
\begin{align*}
&\leq\sqrt{SAL \sum_{(h,k)\in[H]\times[K]} \Varb[s'\sim P^\star_{h}(s_{h,k},a_{h,k})]{ \Eb[r\sim R(s_{h,k},a_{h,k})]{V^{\wh\rho^k}_{h+1}(s',b_{h,k}-r) } } }
\\&\leq \sqrt{SAL \prns{2HL + 2\sum_{(h,k)\in[H]\times[K]} \Eb[\wh\rho^k,\wh b_k]{\Varb[s'\sim P^\star_{h}(s_{h},a_{h})]{ \Eb[r\sim R(s_h,a_h)]{ V^{\wh\rho^k}_{h+1}(s',b_{h}-r) } } \mid\Ecal_k } } } \tag{\cref{eq:azuma-variance-V-rhok}}
\\&\leq \sqrt{SAL \prns{2HL + 2\sum_{(h,k)\in[H]\times[K]} \Eb[\wh\rho^k,\wh b_k]{\Varb[s'\sim P^\star_{h}(s_{h},a_{h}),r\sim R(s_h,a_h)]{ V^{\wh\rho^k}_{h+1}(s',b_{h}-r) } } \mid\Ecal_k } } \tag{joint variance is larger}
\\&= \sqrt{SAL \prns{2HL + 2\sum_{k=1}^K\Varb[\wh\rho^k,\wh b_k]{\prns{ \wh b_k-\sum_{t=1}^Hr_t }^+\mid\Ecal_k} } }, \tag{Law of Total Variance \cref{lem:law-of-total-variance-applied}}
\end{align*}
Below, we give two ways to bound the sum,
\begin{align}
    \sum_{k=1}^K\Varb[\wh\rho^k,\wh b_k]{\prns{ \wh b_k-\sum_{t=1}^Hr_t }^+\mid\Ecal_k}. \label{eq:bernstein-sums-of-probabilities}
\end{align}
The first way is to simply bound it by a probability, which is trivially at most $1$. This results in \cref{thm:bernstein-bonus-regret}.
To prove \cref{thm:bernstein-bonus-regret2}, we show a second more refined approach, which uses \cref{ass:continuous-densities} to bound each variance by $\tau$, plus a lower order term.

Before doing so, we first prepare to conclude the proof by recapping all the terms in the regret decomposition.
First, we collected $4S^2AHL^2$ from the $1/N_{k}(s,a)$ terms.
\cref{eq:goal-bonus-correction-term} is the correction term inside the Bernstein bonus.
\cref{eq:bernstein-switch-to-population-variance-diff} is the switching cost from empirical variance (in the Bernstein bonus) to the population variance that we want to bound now, which is \cref{eq:goal-bonus-dominant-term}.
We also multiply back the $2\sqrt{2L}\leq 3\sqrt{L}$ factor we omitted from above. So,
\begin{align*}
    &4S^2AHL^2 + 3\sqrt{L}\prns{ \text{\cref{eq:goal-bonus-correction-term}} + \text{\cref{eq:bernstein-switch-to-population-variance-diff}} + \text{\cref{eq:goal-bonus-dominant-term}} }
    \\&\leq 4S^2AHL^2 + 3\sqrt{L}\prns{ \prns{5SAHK^{1/4}L+4S^{3/2}AHL^2} + \prns{4SAHK^{1/4}L + 6S^{3/2}AHL^2} + \text{\cref{eq:goal-bonus-dominant-term}} }
    \\&= 27SAHK^{1/4}L^2 + 34S^2AHL^3 + 3\sqrt{L}\cdot \text{\cref{eq:goal-bonus-dominant-term}}.
\end{align*}
Thus, the regret is at most
\begin{align*}
    \rlregret_\tau(K)
    &\leq 6e\tau^{-1}HL + 2e\tau^{-1}\prns{ 27SAHK^{1/4}L^2 + 34S^2AHL^3 + 3\sqrt{L}\cdot \text{\cref{eq:goal-bonus-dominant-term}} }
    \\&\leq 6e\tau^{-1}\sqrt{L}\cdot \text{\cref{eq:goal-bonus-dominant-term}} + 54e\tau^{-1}SAHK^{1/4}L^2 + 70e\tau^{-1}S^2AHL^3.
\end{align*}

\textbf{First bound for \cref{eq:goal-bonus-dominant-term} (for \cref{thm:bernstein-bonus-regret}): }
Since $x^+=x\I{x\geq 0}$, for any random variable $X\in[0,1]$, we have $\Varb{X^+}\leq \Eb{X^2\I{X\geq 0}} \leq \Pr(X\geq 0)$.
Therefore,
\begin{align*}
\sum_{k=1}^K\Varb[\wh\rho^k,\wh b_k]{\prns{ \wh b_k-\sum_{t=1}^Hr_t }^+\mid\Ecal_k}
&\leq \sum_{k=1}^K\Pr_{\wh\rho^k,\wh b_k}\prns{ \sum_{h=1}^Hr_h\leq \wh b_k \mid\Ecal_k}
\leq K
\end{align*}
Certainly, each probability is bounded by $1$.
Hence,
\begin{align*}
\text{\cref{eq:goal-bonus-dominant-term}}\leq \sqrt{SAL\prns{2HL + 2K}} \leq \sqrt{2SAKL} + \sqrt{2SAH}L.
\end{align*}
Combining everything together, we get
\begin{align*}
    \rlregret_\tau(K)
    &\leq 6\sqrt{2}e\tau^{-1}\sqrt{SAK}L + 54e\tau^{-1}SAHK^{1/4}L^2 + 82e\tau^{-1}S^2AHL^3,
\end{align*}
which concludes the proof for \cref{thm:bernstein-bonus-regret}.

\textbf{Second bound for \cref{eq:goal-bonus-dominant-term} (for \cref{thm:bernstein-bonus-regret2}): }
Define
\begin{align*}
    &f(b) = b-\tau^{-1}\Eb[\wh\rho^k,\wh b_k]{\prns{b-\sum_{t=1}^H r_t}^+\mid\Ecal_k},
    &b_k^\star=\argmax_{b\in[0,1]}f(b),
    \\&\wh f(b) = b-\tau^{-1}\wh V_{1,k}^\downarrow(s_1,b),
    &\wh b_k=\argmax_{b\in[0,1]}\wh f(b).
\end{align*}
So, $b_k^\star$ as the $\tau$-th quantile of $R(\wh\rho^k,\wh b_k)$.
By pessimism, we have $\wh V_{1,k}^\downarrow(s_1,b)\leq V^\star_1(s_1,b)\leq \Eb[\wh\rho^k,\wh b_k]{ \prns{b-\sum_{t=1}^Hr_t}^+ }$, since $V^\star_1$ is the minimum amongst all history-dependent policies, including $(\wh\rho^k,\wh b_k)$.
Thus, we have $\wh f(b)\geq f(b)$ for all $b\in[0,1]$.
In particular, we have
\begin{align*}
    f(b_k^\star)-f(\wh b_k)
    &= f(b_k^\star)-\wh f(\wh b_k)+\wh f(\wh b_k)-f(\wh b_k)
    \\&\leq f(b_k^\star)-\wh f(b_k^\star)+\wh f(\wh b_k)-f(\wh b_k) \tag{$\wh b_k$ is argmax of $\wh f$}
    \\&\leq \wh f(\wh b_k)-f(\wh b_k) \tag{$f(b_k^\star)\leq \wh f(b_k^\star)$ by pessimism}
    \\&\leq \tau^{-1}\prns{ V^{\wh\rho^k}_1(s_1,\wh b_k)-\wh V_{1,k}^\downarrow(s_1,\wh b_k) }.
\end{align*}

Now we invoke \cref{ass:continuous-densities} with \cref{lem:strongly-convex-under-continuity} which implies
\begin{align*}
    &\frac{\densitymin}{2\tau}\prns{ b_k^\star-\wh b_k }^2 \leq f(b_k^\star)-f(\wh b_k) \leq \tau^{-1}\prns{ V^{\wh\rho^k}_1(s_1,\wh b_k)-\wh V_{1,k}^\downarrow(s_1,\wh b_k) }
    \\&\implies \prns{b_k^\star-\wh b_k}^2 \leq 2\densitymin^{-1}\prns{ V^{\wh\rho^k}_1(s_1,\wh b_k)-\wh V_{1,k}^\downarrow(s_1,\wh b_k) }.
\end{align*}
Using $\Var(X)\leq 2(\Var(Y)+\Var(X-Y))$,
\begin{align*}
    &\text{\cref{eq:bernstein-sums-of-probabilities}}=\sum_{k=1}^K\Varb[\wh\rho^k,\wh b_k]{\prns{ \wh b_k-\sum_{t=1}^Hr_t }^+\mid\Ecal_k}
    \\&\leq 2\sum_{k=1}^K\Varb[\wh\rho^k,\wh b_k]{\prns{ b^\star_k-\sum_{t=1}^Hr_t }^+\mid\Ecal_k} + 2\sum_{k=1}^K\Varb[\wh\rho^k,\wh b_k]{\prns{ \wh b_k-\sum_{t=1}^Hr_t }^+-\prns{ b^\star_k-\sum_{t=1}^Hr_t }^+\mid\Ecal_k}
    \\&\leq 2\sum_{k=1}^K\Pr_{\wh\rho^k,\wh b_k}\prns{ \sum_{t=1}^Hr_t\leq b^\star_k\mid\Ecal_k } + 2\sum_{k=1}^K \prns{ \wh b_k-b^\star_k }^2 \tag{ReLU is $1$-Lipschitz}
    \\&\leq 2K\tau + 4\densitymin^{-1}\sum_{k=1}^K \prns{ V^{\wh\rho^k}_1(s_1,\wh b_k)-\wh V_{1,k}^\downarrow(s_1,\wh b_k) }
    \\&\leq 2K\tau + 4\densitymin^{-1}\sum_{k=1}^K\sum_{h=1}^H\Eb[\wh\rho^k,\wh b_k]{ 2\bebon_{h,k}(s_h,b_h,a_h)+\xi_{h,k}(s_h,a_h)\mid \Ecal_k } \tag{simulation \cref{lem:simulation}}
    \\&\leq 2K\tau + 4\densitymin^{-1}\prns{ 3\sqrt{KL} + 6\sqrt{SAHK}L + 3S^2AHL^2 }. \tag{\cref{eq:azuma-standard-for-bonus-sum} and the loose bound in \cref{eq:loose-bound-bernstein-bonus-sum}}
\end{align*}
Therefore,
\begin{align*}
    \text{\cref{eq:goal-bonus-dominant-term}}
    &\leq \sqrt{SAL\prns{ 2HL + 4K\tau + 72\densitymin^{-1}\sqrt{SAHK}L + 24\densitymin^{-1}S^2AHL^2 } }
    \\&\leq 2\sqrt{\tau SAKL} + 9\densitymin^{-1/2}SAHK^{1/4}L + 5\densitymin^{-1/2}S^{3/2}AHL^2.
\end{align*}
Combining everything together, we get
\begin{align*}
\rlregret_\tau(K)\leq 12e\sqrt{\tau^{-1}SAK}L + \prns{54+9\densitymin^{-1/2}}e\tau^{-1}SAHK^{1/4}L^2 + \prns{70+5\densitymin^{-1/2}}e\tau^{-1}S^2AHL^3.
\end{align*}
This concludes the proof for \cref{thm:bernstein-bonus-regret2}.

\end{proof}

\begin{lemma}\label{lem:simulation-on-confidence-interval}
Fix any $k\in[K]$. Then for all $t\in[H]$, for all $s_t,b_t$, we have
\begin{align*}
    \wh V^\uparrow_{t,k}(s_t,b_t) - \wh V^\downarrow_{t,k}(s_t,b_t) \leq \sum_{h=t}^H(1-1/H)^{t-h}\Eb[\wh\rho^k,s_t,b_t]{2\bon_{h,k}(s_h,b_h,a_h)+\xi_{h,k}(s_h,a_h)\mid\Ecal_k}.
\end{align*}
\end{lemma}
\begin{proof}
In this proof, all expectations are conditioned on $\Ecal_k$.
We proceed by induction.
The base case at $t=H+1$ trivially holds since $\wh V^\uparrow_{H+1,k}(s,b) - \wh V^\downarrow_{H+1,k}(s,b) = b^+-b^+=0$.
Now suppose $t\leq H$ and suppose the claim holds for $t+1$.
Then setting $a_t = \wh\rho^k(s_t,b_t)$, we have
\begin{align*}
    &\phantom{=}\wh V^\uparrow_{t,k}(s_t,b_t)-\wh V^\downarrow_{t,k}(s_t,b_t)
    \\&\leq \wh U^\uparrow_{t,k}(s_t,b_t,a_t)-\wh U^\downarrow_{t,k}(s_t,b_t,a_t)
    \\&= \wh P_{t,k}(s_t,a_t)^\top \Eb[r_t]{\wh V^\uparrow_{t+1,k}(\cdot, b_{t}-r_t)-\wh V^\downarrow_{t+1,k}(\cdot, b_{t}-r_t) } + 2\bon_{t,k}(s_t,b_t,a_t)
    \\&= \prns{\wh P_{t,k}(s_t,a_t)-P^\star_t(s_t,a_t)}^\top \Eb[r_t]{ \wh V^\uparrow_{t+1,k}(\cdot, b_{t}-r_t)-\wh V^\downarrow_{t+1,k}(\cdot, b_{t}-r_t) } + 2\bon_{t,k}(s_t,b_t,a_t)
    \\&+ P^\star_t(s_t,a_t)^\top \Eb[r_t]{ \wh V^\uparrow_{t+1,k}(\cdot, b_{t}-r_t)-\wh V^\downarrow_{t+1,k}(\cdot, b_{t}-r_t) }
    \\&\leq 2\bon_{t,k}(s_t,b_t,a_t) + \xi_{t,k}(s_t,a_t) + (1-1/H) P^\star_t(s_t,a_t)^\top \Eb[r_t]{ \wh V^\uparrow_{t+1,k}(\cdot, b_{t}-r_t)-\wh V^\downarrow_{t+1,k}(\cdot, b_{t}-r_t) } \tag{\cref{lem:projected-error-phat-pstar}}
    \\&\leq 2\bon_{t,k}(s_t,b_t,a_t) + \xi_{t,k}(s_t,a_t) + \sum_{h=t+1}^H(1-1/H)^{1+t-(h+1)}\Eb[\wh\rho^k,s_t,b_t]{ 2\bon_{h,k}(s_h,b_h,a_h)+\xi_{h,k}(s_h,a_h) },
\end{align*}
where $\EE_{r_t}$ is short for $\EE_{r_t\sim R_t(s_t,a_t)}$.
\end{proof}

\begin{lemma}\label{lem:law-of-total-variance-applied}
For any $k\in[K]$, we have
\begin{align*}
    \Varb[\wh\rho^k,\wh b_k]{\prns{ \wh b_k-\sum_{h=1}^H r_h }^+\mid\Ecal_k}
    = \sum_{h=1}^{H-1} \Eb[\wh\rho^k,\wh b_k]{ \Varb[s'\sim P^\star(s_h,a_h),r_h\sim R(s,a)]{V^{\wh\rho^k}_{h+1}(s',b_{h}-r_h) }\mid\Ecal_k}.
\end{align*}
\end{lemma}
\begin{proof}
Apply Law of Total Variance \cref{lem:law-total-variance} with
$Y=\prns{\wh b_k-\sum_{h=1}^H r_h}^+, X_h=(s_h,a_h,r_{h-1})$ for $h\in[H]$ (when $h=1$, $r_0$ is omitted), and $\Hcal=\Ecal_k$ being the trajectories from the past episodes $1,2,...,k-1$.
Here, $s_h,a_h,r_h$ are collected from rolling in with $\wh\rho^k$ starting from $\wh b_k$, and note that $\wh b_k$ is a constant conditioned on $\Ecal_k$.
\cref{lem:law-total-variance} gives
\begin{align*}
    \Varb[\wh\rho^k,\wh b_k]{\prns{ \wh b_k-\sum_{h=1}^H r_h }^+\mid\Ecal_k}
    &=\Eb{\Varb{Y\mid X_{1:H},\Ecal_k}\mid\Ecal_k} + \sum_{h=1}^H \Eb{ \Varb{ \Eb{Y\mid X_{1:h},\Ecal_k } \mid X_{1:h-1},\Ecal_k} \mid\Ecal_k}
    \\&= \sum_{h=1}^H\Eb{ \Varb{ \Eb{Y\mid X_{1:h},\Ecal_k } \mid X_{1:h-1},\Ecal_k} \mid\Ecal_k}.
\end{align*}
The first term is zero because once we condition on $X_{1:H},\Ecal_k$, the term $\prns{ \wh b_k-\sum_t r_t }^+$ is a constant, and variance of constants is zero.
Now consider each summand. The outer expectation is taken over $s_{1:h-1},a_{1:h-1},r_{1:h-2}$ from rolling in $\wh\rho^k$. The variance is taken over $s_h\sim P^\star(s_{h-1},a_{h-1}),r_{h-1}\sim R_{h-1}(s_{h-1},a_{h-1})$, and deterministically picking $a_h=\wh\rho_h^k(s_h,b_h)$ where $b_h=\wh b_k-\sum_{t=1}^{h-1}r_t$. The inner expectation is over the remainder of the trajectory, which is $s_{h+1:H},a_{h+1:H},r_{h:H}$. Therefore,
\begin{align*}
    &\Eb{ \Varb{ \Eb{Y\mid X_{1:h},\Ecal_k } \mid X_{1:h-1},\Ecal_k} \mid\Ecal_k}
    \\&= \Eb[\wh\rho^k,\wh b_k]{ \Varb{ \Eb[\wh\rho^k,\wh b_k]{ \prns{\wh b_k-\sum_{t=1}^H r_t }^+ \mid X_{1:h},\Ecal_k } \mid X_{1:h-1},\Ecal_k} \mid\Ecal_k }
    \\&= \Eb[\wh\rho^k,\wh b_k]{ \Varb{ U^{\wh\rho^k}_h(s_h,b_h,a_h) \mid X_{1:h-1},\Ecal_k} \mid\Ecal_k } \tag{$b_{h-1}=\wh b_k-r_1-...-r_{h-1}$}
    \\&= \Eb[\wh\rho^k,\wh b_k]{ \Varb[s_h\sim P_{h-1}^\star(s_{h-1},a_{h-1}),r_{h-1}\sim R_{h-1}(s_{h-1},a_{h-1})]{ U^{\wh\rho^k}_h(s_h,b_h,a_h)} \mid\Ecal_k }
    \\&= \Eb[\wh\rho^k,\wh b_k]{ \Varb[s_h\sim P_{h-1}^\star(s_{h-1},a_{h-1}),r_{h-1}\sim R_{h-1}(s_{h-1},a_{h-1})]{ V^{\wh\rho^k}_h(s_h,b_h)} \mid\Ecal_k }. \tag{$a_h=\wh\rho^k(s_h,b_h)$}
\end{align*}
Note that in the special case of $h=1$, we have $s_1$ is not random and $r_0$ is omitted.
So, the variance is taken over a constant, which is zero.
\end{proof}

\begin{lemma}\label{lem:strongly-convex-under-continuity}
Let $\pi$ be a history-dependent policy such that its return distribution $R(\pi)$ is continuously distributed and has a density $p$ lower bounded by $\densitymin$.
Then, the function
\begin{align*}
    f(b) = \tau^{-1}\Eb[\pi]{\prns{b-\sum_h r_h}^+}-b
\end{align*}
is $\tau^{-1}\densitymin$ strongly convex.
\end{lemma}
\begin{proof}
Observe that
\begin{align*}
    f'(b) = \tau^{-1}\Pr_{\pi}\prns{ \sum_hr_h\leq b }-1.
\end{align*}
Since we've assumed that $R(\pi)$ is continuously distributed with density $p$, so
\begin{align*}
    f''(b) = \tau^{-1}p(b).
\end{align*}
Since $f''(b)\geq \tau^{-1}\densitymin$ for all $b$, we have $f''$ is strongly convex with that parameter.
\end{proof}

\subsection{Auxiliary Lemmas}

\begin{lemma}[Azuma]\label{lem:mult-azuma}
Let $\braces{X_i}_{i\in[N]}$ be a sequence of random variables supported on $[0,1]$, adapted to filtration $\braces{\Fcal_i}_{i\in[N]}$.
For any $\delta\in(0,1)$, we have w.p. at least $1-\delta$,
\begin{align*}
&\sum_{t=1}^N\Eb{X_t\mid\Fcal_{t-1}}\leq \sum_{t=1}^NX_t+\sqrt{N\log(2/\delta)}, \tag{Standard Azuma}
\\&\sum_{t=1}^N\Eb{X_t\mid\Fcal_{t-1}}\leq 2\sum_{t=1}^NX_t + 2\log(1/\delta). \tag{Multiplicative Azuma}
\end{align*}
\end{lemma}
\begin{proof}
For standard Azuma, see \citet[Theorem 13.4]{tongzhangbook}.
For multiplicative Azuma, apply \citep[Theorem 13.5]{tongzhangbook} with $\lambda = 1$. The claim follows, since $\frac{1}{1-\exp(-\lambda)}\leq 2$.
\end{proof}

Below we recall the standard elliptical potential lemma \citep{lattimore2020bandit,rltheorybookAJKS}.
Regarding terminology, we remark that this lemma is also known as the ``pigeonhole argument'' in the tabular RL literature \citep{azar2017minimax,zanette2019tighter}.
The term ``elliptical potential'' is more commonly used in the linear MDP setting \citep{jin2020provably}, of which tabular RL is a special case.
\begin{lemma}[Elliptical Potential]\label{lem:tabular-elliptical-potential}
For any sequence of states and actions $\braces{s_{h,k},a_{h,k}}_{h\in[H],k\in[K]}$, we have
\begin{align*}
    &\sum_{k=1}^K\sum_{h=1}^H \frac{1}{N_{k}(s_{h,k},a_{h,k})} \leq SA\log(K),
    \\&\sum_{k=1}^K\sum_{h=1}^H\frac{1}{\sqrt{N_{k}(s_{h,k},a_{h,k})}} \leq \sqrt{HSAK\log(K)}.
\end{align*}
\end{lemma}
\begin{proof}
For the first claim, observe that in the sum, $\frac{1}{1},\frac{1}{2},\frac{1}{3},\dots$ can appear at most $SA$ times.
And since we run for $K$ episodes, the maximum denominator is $K$. Therefore, we have
\begin{align*}
    \sum_{k=1}^K\sum_{h=1}^H \frac{1}{N_{k}(s_{h,k},a_{h,k})} \leq SA\sum_{k=1}^K \frac{1}{k} \leq SA\log(K).
\end{align*}
For the second claim,
\begin{align*}
    \sum_{k=1}^K\sum_{h=1}^H\frac{1}{\sqrt{N_{k}(s_{h,k},a_{h,k})}}
    &\leq \sqrt{ KH \prns{ \sum_{k,h=1}^{K,H}\frac{1}{N_{k}(s_{h,k},a_{h,k})} } }
    \\&\leq \sqrt{ SAHK\log(K) }.
\end{align*}
\end{proof}

\begin{lemma}[Sequential Law of Total Conditional Variance]\label{lem:law-total-variance}
For any random variables $Y,X_1,X_2,...,X_N,\Hcal$, we have
\begin{align*}
    \Varb{Y\mid\Hcal}=\Eb{\Varb{Y\mid X_{1:N},\Hcal}\mid\Hcal} + \sum_{t=1}^N \Eb{ \Varb{ \Eb{Y\mid X_{1:t},\Hcal } \mid X_{1:t-1},\Hcal} \mid\Hcal}.
\end{align*}
Notice, for each summand, the inner expectation is taken over $Y$, the variance is taken over $X_t$, and outer expectation is taken over $X_{1:t-1}$.
\end{lemma}
\begin{proof}
Recall the Law of Total Conditional Variance (LTCV): for any random variables $Y,Z_1,Z_2$,
\begin{align*}
    \Varb{Y\mid Z_1}=\Eb{\Varb{Y\mid Z_1,Z_2}\mid Z_1}+\Varb{\Eb{Y\mid Z_1,Z_2}\mid Z_1}.
\end{align*}
We now inductively prove the desired claim by recursively applying the (LTCV).
The base case is $N=1$, which follows immediately from LTCV applied to $Z_1=\Hcal,Z_2=X_1$.
For the inductive case, fix any $N$ and suppose the claim is true for $N$. Now consider $N+1$, where we have $Y,X_1,X_2,...,X_{N+1},\Hcal$.
By the IH, we have
\begin{align*}
    \Varb{Y\mid\Hcal}=\Eb{\Varb{Y\mid X_{1:N},\Hcal}\mid\Hcal} + \sum_{t=1}^N \Eb{ \Varb{ \Eb{Y\mid X_{1:t},\Hcal } \mid X_{1:t-1},\Hcal}\mid\Hcal}.
\end{align*}
Now applying LTCV on the first term with $Z_1=(X_{1:N},\Hcal),Z_2=X_{N+1}$, we have
\begin{align*}
    \Varb{Y\mid X_{1:N},\Hcal}
    &= \Eb{\Varb{Y\mid X_{1:N+1},\Hcal} \mid X_{1:N},\Hcal } + \Varb{\Eb{Y\mid X_{1:N+1},\Hcal}\mid X_{1:N},\Hcal},
\end{align*}
and therefore,
\begin{align*}
    \Eb{\Varb{Y\mid X_{1:N},\Hcal}\mid\Hcal}
    &= \Eb{ \Varb{Y\mid X_{1:N+1},\Hcal}\mid\Hcal } + \Eb{\Varb{\Eb{Y\mid X_{1:N+1},\Hcal}\mid X_{1:N},\Hcal}\mid\Hcal},
\end{align*}
which concludes the proof.
\end{proof}

\section{Proofs for \rlalg{} with discretized rewards}
Recall the discretized MDP $\disc(\Mcal)$, as introduced in \cref{sec:discretized-mdp-computational-efficiency}.
It is a copy of the true MDP $\Mcal$ except its rewards are rounded to an $\eps$-net.
\Ie, let $\phi(r)=\eta\ceil{r/\eta}$ be the rounding up operator of $r$ onto the net, so that $0\leq\phi(r)-r\leq\eta$.
Concretely, the reward distribution is $R(s,a;\disc(\Mcal))=R(s,a;\Mcal)\circ\phi^{-1}$.
Our proofs are inspired by \citet[Lemma B.1,Lemma B.5]{bastani2022regret}.

\textbf{From $\disc(\Mcal)$ to $\Mcal$:}
Fix any $\rho\in\Piaug$ and $b\in[0,1]$ (which we'll run in $\disc(\Mcal)$). Then define an adapted policy, which is a history-dependent policy in $\Mcal$, as follows,
\begin{align*}
    \adapted(\rho,b)_h(s_h,r_{1:h-1})=\rho_h(s_h,b_1-\phi(r_1)-...-\phi(r_{h-1})).
\end{align*}
Intuitively, as $\adapted(\rho,b)$ runs $\Mcal$, it uses the history to emulate the evolution of $b$ in $\disc(\Mcal)$.
Let $Z_{\rho,b,\disc(\Mcal)}$ be the returns from running $\rho,b$ in $\disc(\Mcal)$.
Let $Z_{\adapted(\rho,b),\Mcal}$ be the returns from running $\adapted(\rho,b)$ in $\Mcal$.
We show that $Z_{\rho,b,\disc(\Mcal)}-H\eta\leq Z_{\adapted(\rho,b),\Mcal}\leq Z_{\rho,b,\disc(\Mcal)}$ w.p. $1$ via a coupling argument.
\begin{lemma}\label{lem:from-disc-MDP-to-MDP}
We almost surely have $Z_{\rho,b,\disc(\Mcal)}-H\eta\leq Z_{\adapted(\rho,b),\Mcal}\leq Z_{\rho,b,\disc(\Mcal)}$.
Therefore, if $F_{\rho,b,\disc(\Mcal)}$ is the CDF of $Z_{\rho,b,\disc(\Mcal)}$ and $F_{\adapted(\rho,b),\Mcal}$ is the CDF of $Z_{\adapted(\rho,b),\Mcal}$, we have
\begin{align*}
    \forall x\in\RR: F_{\rho,b,\disc(\Mcal)}(x)\leq F_{\adapted(\rho,b),\Mcal}(x)\leq F_{\rho,b,\disc(\Mcal)}(x+H\eta).
\end{align*}
\end{lemma}
\begin{proof}
Let $s_1,a_1,r_1,s_2,a_2,r_2,\dots$ be the trajectory of running $\adapted(\rho,b)$ in $\Mcal$.
Let $\wh s_1,\wh a_1,\wh r_1,\wh s_2,\wh a_2,\wh r_2,\dots$ be the trajectory of running $\rho,b$ in $\disc(\Mcal)$.
We couple these two trajectories by making $\adapted(\rho,b)$ in $\Mcal$ follow $\rho,b$ in $\disc(\Mcal)$.
Set $\wh s_1=s_1$.
By definition of $\adapted(\rho,b)$, $\wh a_1=a_1$.
By definition of $\disc(\Mcal)$, $\wh r_1=\phi(r_1)$.
Continuing in this fashion, we have $\wh s_t=s_t, \wh a_t=a_t, \wh r_t=\phi(r_t)$ for all $t\in[H]$.
This is a valid coupling since the actions are sampled from the exact same distribution, \ie, $a_h\sim \rho_h(b-\phi(r_1)-...-\phi(r_{h-1}))$,
and by the transitions of $\wh b_h$ in $\disc(\Mcal)$, we have $\wh b_h = b-\wh r_1-...-\wh r_{h-1}$ which
is exactly what was inputted into $\rho_h$ by $\adapted(\rho,b)$.

Since $r\leq\phi(r)$ for all $r$, we have
\begin{align*}
    Z_{\adapted(\rho,b),\Mcal}=\sum_{t=1}^H r_t\leq \sum_{t=1}^H \phi(r_t) = \sum_{t=1}^H \wh r_t = Z_{\rho,b,\disc(\Mcal)}.
\end{align*}
Since $\phi(r)-\eta\leq r$ for all $r$, we have
\begin{align*}
    Z_{\rho,b,\disc(\Mcal)}= \sum_{t=1}^H \phi(r_t)\leq -H\eta + \sum_{t=1}^H r_t = Z_{\adapted(\rho,b),\Mcal}-H\eta.
\end{align*}
To conclude the proof, recall a basic fact about couplings and stochastic comparisons.
For two random variables $X,Y$ in the same probability space and a constant $c$, if $\PP(X\leq Y+c) = 1$, we have $F_Y(t)\leq F_X(t+c)$ for all $x$.
This is because $F_Y(x)-F_X(x+c) = \PP(Y\leq t\cap X>t+c)\leq \PP(X-Y>c)=0$.
\end{proof}

\textbf{From $\Mcal$ to $\disc(\Mcal)$:}
Fix any $\rho\in\Piaug$ and $b\in[0,1]$ (which we'll run in $\Mcal$).
Then define a discretized policy, which is a history-dependent policy in the discretized MDP $\disc(\Mcal)$ \emph{with memory} (as in \cref{sec:augmented-mdp}), as follows,
\begin{align*}
    &\disc(\rho,b)_h(s_h,m_{1:h-1}) = \rho_h(s_h,b-m_1-\dots-m_{h-1}),
    \\&m_h\sim R(s_h,a_h)\mid \phi(R(s_h,a_h))=r_h.
\end{align*}
With this definition, although we receive reward $\wh r_h$ (on the discrete grid) when running in $\disc(\Mcal)$,
the memory element $m_h$ exactly imitates a random reward that would have been received in the true MDP $\Mcal$.
Then, the discretized policy $\disc(\rho,b)$ will instead follow these exact rewards $m_h$ rather than what has been received.

Let $Z_{\rho,b,\Mcal}$ be the returns from running $\rho,b$ in $\Mcal$.
Let $Z_{\disc(\rho,b),\disc(\Mcal)}$ be the returns from running $\disc(\rho,b)$ in $\disc(\Mcal)$ (which memory as described above).
We show that $Z_{\rho,b,\Mcal}\leq Z_{\disc(\rho,b),\disc(\Mcal)}$ w.p. $1$ via a coupling argument.
\begin{lemma}\label{lem:from-MDP-to-disc-MDP}
We almost surely have $Z_{\rho,b,\Mcal}\leq Z_{\disc(\rho,b),\disc(\Mcal)}$.
Therefore, if $F_{\rho,b,\Mcal}$ is the CDF of $Z_{\rho,b,\Mcal}$ and
$F_{\disc(\rho,b),\disc(\Mcal)}$ is the CDF of $Z_{\disc(\rho,b),\disc(\Mcal)}$, we have
\begin{align*}
    \forall x\in\RR: F_{\disc(\rho,b),\disc(\Mcal)}(x)\leq F_{\rho,b,\Mcal}.
\end{align*}
\end{lemma}
\begin{proof}
Let $s_1,a_1,r_1,s_2,a_2,r_2,...$ be the trajectory of running $\rho,b$ in $\Mcal$.
Let $\wh s_1,\wh a_1,\wh r_1,\wh m_1,\wh s_2,\wh a_2,\wh r_2,\wh m_2,...$ be the trajectory of running $\disc(\rho,b)$ in $\disc(\Mcal)$ with memory.
We couple these two trajectories by making $\rho,b$ in $\Mcal$ follow $\disc(\rho,b)$ in $\disc(\Mcal)$.
Set $s_1=\wh s_1$. By definition of $\disc(\rho,b)$, $a_1=\wh a_1$.
Then, set $r_1=\wh m_1$.
Note that $r_1$ is sampled by first sampling a discrete $\wh r_1$, then sampling $\wh m_1$ from the conditional reward distribution of the interval that rounds to $\wh r_1$.
By law of total probability, this is indeed equivalent to sampling directly from the unconditional reward distribution.
Continuing in this fashion, we have $\wh s_t=s_t,\wh a_t=a_t, \wh r_t=\phi(r_t), \wh m_t=m_t$ for all $t\in[H]$.
Importantly, the policies actions match because $b-\wh m_1-...-\wh m_{h-1}=b-r_1-...-r_{h-1}$.
Therefore, we always have
\begin{align*}
    Z_{\rho,b,\Mcal}=\sum_{t=1}^Hr_t\leq \sum_{t=1}^H\phi(r_t)=\sum_{t=1}^H \wh r_t = Z_{\disc(\rho,b),\disc(\Mcal)}.
\end{align*}
As with the previous proof, stochastic dominance implies the claim on CDFs.
\end{proof}

Now we show two useful consequences of the above coupling results.
\begin{theorem}\label{thm:adapted-discretized-mdp}
Fix any $\rho\in\Piaug$ and $b_1\in[0,1]$.
Then,
\begin{align*}
    \forall b: 0\leq \Eb[\adapted(\rho,b_1),\Mcal]{\prns{b-\sum_{h=1}^Hr_h}^+}-\Eb[\rho,b_1,\disc(\Mcal)]{\prns{b-\sum_{h=1}^Hr_h}^+} \leq H\eta,
\end{align*}
and
\begin{align*}
    -\tau^{-1}H\eta\leq \cvar_\tau(\adapted(\rho,b_1),\Mcal)-\cvar_\tau(\rho,b,\disc(\Mcal))\leq 0.
\end{align*}
\end{theorem}
\begin{proof}
Let $f(b) = \Eb[\adapted(\rho,b_1),\Mcal]{\prns{b-\sum_{h=1}^Hr_h}^+}$ and let $F=F_{\adapted(\rho,b_1),\Mcal}$.
Similarly, let $\disc(f)(b) = \Eb[\rho,b_1,\disc(\Mcal)]{\prns{b-\sum_{h=1}^Hr_h}^+}$ and $\disc(F)=F_{\rho,b_1,\disc(\Mcal)}$.
Both $f(0)=\disc(f)(0)=0$. Also, their derivatives are $F$ and $\disc(F)$ respectively.
By \cref{lem:from-disc-MDP-to-MDP}, $\disc(F)(t)\leq F(t)\leq \disc(F)(t+H\eta)$.

First, we show $\disc(f)(b)-f(b)\leq 0$.
By the fundamental theorem of Calculus,
\begin{align*}
    \disc(f)(b)-f(b)=\int_0^b \disc(F)(t)-F(t)\diff t \leq 0.
\end{align*}

Next, we show $f(b)-\disc(f)(b)\leq H\eta$.
By the fundamental theorem of Calculus (FTC),
\begin{align*}
    f(b)-\disc(f)(b)
    &\leq \int_0^b F(t)\diff t-\int_0^b\disc(F)(t)\diff t
    \\&\leq \int_{H\eta}^{b+H\eta}\disc(F)(t)\diff t-\int_0^b\disc(F)(t)\diff t \tag{\cref{lem:from-disc-MDP-to-MDP}}
    \\&\leq \int_b^{b+H\eta}\disc(F)(t)\diff t
    \\&= \disc(f)(b+H\eta)-\disc(f)(b) \tag{FTC}
    \\&\leq H\eta, \tag{ReLU is $1$-Lipschitz}
\end{align*}
Altogether, we've shown $0\leq f(b)-\disc(f)(b)\leq H\eta$ for all $b$.
This immediately implies the claims about CVaR:
\begin{align*}
    &\cvar_\tau(\adapted(\rho,b_1),\Mcal)-\cvar_\tau(\rho,b_1,\disc(\Mcal))
    \\&=\max_b\braces{b-\tau^{-1}f(b) }-\max_b\braces{b-\tau^{-1}\disc(f)(b) }
    \\&\leq \tau^{-1}\max_b \prns{ \disc(f)(b)-f(b) }
    \\&\leq 0,
\end{align*}
and similarly,
\begin{align*}
    &\cvar_\tau(\rho,b_1,\disc(\Mcal))-\cvar_\tau(\adapted(\rho,b_1),\Mcal)
    \\&\leq \tau^{-1}\max_b \prns{ f(b)-\disc(f)(b) }
    \\&\leq \tau^{-1}H\eta.
\end{align*}

\end{proof}

\begin{theorem}\label{thm:discretized-cvar-dominates}
We have,
\begin{align*}
    &\forall b\in[0,1]: V^\star_1(s_1,b;\disc(\Mcal))\leq V^\star_1(s_1,b;\Mcal),
\end{align*}
which implies
\begin{align*}
    &\cvar_\tau^\star(\disc(\Mcal))\geq\cvar_\tau^\star(\Mcal).
\end{align*}
\end{theorem}
\begin{proof}
Fix any $\rho\in\Piaug, b\in[0,1]$.
By \cref{lem:from-MDP-to-disc-MDP}, we have $F_{\disc(\rho,b),\disc(\Mcal)}(x)\leq F_{\rho,b,\Mcal}$.
Applying the same FTC-style arguments in \cref{thm:adapted-discretized-mdp}, we have
\begin{align*}
    \forall b': \Eb[\disc(\rho,b),\disc(\Mcal)]{\prns{b'-\sum_{h=1}^Hr_h}^+}\leq \Eb[\rho,b,\Mcal]{\prns{b'-\sum_{h=1}^Hr_h}^+}.
\end{align*}
Setting $b'=b$ and using the definition of $V$ functions, we have $V^{\disc(\rho,b)}(s_1,b;\disc(\Mcal))\leq V^{\rho}_1(s_1,b;\Mcal)$.
Then, since $V^\star_1$ is the minimum history-dependent policy in this memory-MDP (\cref{thm:rho-star-optimality}) implies that
\begin{align*}
    V^\star_1(s_1,b;\disc(\Mcal))\leq V^{\disc(\rho,b)}_1(s_1,b;\disc(\Mcal)) \leq V^{\rho}_1(s_1,b;\Mcal).
\end{align*}
Since $\rho\in\Piaug$ was arbitrary and the minimum is attained by $\rho^\star\in\Piaug$ (\cref{thm:rho-star-optimality})
this implies that $V^\star_1(s_1,b;\disc(\Mcal))\leq V^\star_1(s_1,b;\Mcal)$, as needed.
For the CVaR claim,
\begin{align*}
    &\cvar_\tau^\star(\Mcal)-\cvar_\tau^\star(\disc(\Mcal))
    \\&= \max_b\braces{b-\tau^{-1}V^\star_1(s_1,b;\Mcal)}-\max_b\braces{b-\tau^{-1}V^\star_1(s_1,b;\disc(\Mcal))}
    \\&\leq\tau^{-1} \max_b \prns{ V^\star_1(s_1,b;\disc(\Mcal))-V^\star_1(s_1,b;\Mcal) } \leq 0.
\end{align*}
\end{proof}
In the proof above, we highlight that $\disc(\Mcal)$ is the MDP with memory.
In the proof of \citet[Lemma B.6]{bastani2022regret}, this detail was glossed over as
their ``history-dependent policy'' in \citet[Lemma B.5]{bastani2022regret} does not exactly fit into the
vanilla history-dependent policy framework (as in \cref{sec:problem-setup});
their policies are coupled through time with the $\alpha$ parameter,
which is disallowed \emph{a priori} by the history-dependent policy framework.
Our formalism with the memory-MDP resolves this ambiguity.

\subsection{Amendment of Bernstein proof in the discretized MDP}\label{sec:amend-bernstein-proof-discretized}
In this section, we amend the proof of the ``Second bound for \cref{eq:goal-bonus-dominant-term}'' in the Bernstein regret bound.

\begin{theorem}\label{thm:bernstein-regret-bound2-discretized}
Suppose we're running \rlalg{} in the discretized MDP and assume \cref{ass:continuous-densities-discretized} holds.
For any $\delta\in(0,1)$, w.p. at least $\delta$, we have the regret of \cref{thm:bernstein-bonus-regret2} plus an additional term,
\begin{align*}
    18e\tau^{-1}\sqrt{\densitymin^{-1} SAHK\eta}L.
\end{align*}
Thus, setting $\eta=1/\sqrt{K}$ makes this an lower order term.
\end{theorem}
\begin{proof}[Proof of \cref{thm:bernstein-regret-bound2-discretized}]
Recall that
\begin{align*}
    \text{\cref{eq:goal-bonus-dominant-term}}\leq \sqrt{SAL \prns{2HL + 2\sum_{k=1}^K\Varb[\wh\rho^k,\wh b_k,\disc(\Mcal)]{\prns{ \wh b_k-\sum_{t=1}^Hr_t }^+\mid\Ecal_k} } }
\end{align*}
where note the randomness is taken over trajectories from $\disc(\Mcal)$, as that's the MDP we're working in.
Define
\begin{align*}
    &f(b) = b-\tau^{-1}\Eb[\adapted(\wh\rho^k,\wh b_k),\Mcal]{\prns{b-\sum_{t=1}^H r_t}^+\mid\Ecal_k},
    &b_k^\star=\argmax_{b\in[0,1]}f(b),
    \\&\wh f(b) = b-\tau^{-1}\wh V_{1,k}^\downarrow(s_1,b),
    &\wh b_k=\argmax_{b\in[0,1]}\wh f(b).
\end{align*}

\emph{A priori,} the pessimism argument a priori only applies to policies in the discretized MDP.
But thanks to \cref{thm:discretized-cvar-dominates}, it also holds here,
\begin{align*}
    \wh V_{1,k}^\downarrow(s_1,b)\leq V^\star_1(s_1,b;\disc(\Mcal))\leq V^\star_1(s_1,b;\Mcal)\leq \Eb[\adapted(\wh\rho^k,\wh b_k),\Mcal]{ \prns{b-\sum_{t=1}^Hr_t}^+ }.
\end{align*}
Thus, we also have $\wh f(b)\geq f(b)$ for all $b$, and the same argument as before gives
\begin{align*}
    f(b^\star_k)-f(\wh b_k)
    &\leq\tau^{-1}\prns{ \Eb[\adapted(\wh\rho^k,\wh b_k),\Mcal]{\prns{\wh b_k-\sum_{t=1}^H r_t}^+\mid\Ecal_k} - \wh V_{1,k}^\downarrow(s_1,\wh b_k) }
    \\&\leq \tau^{-1}\prns{ V^{\wh\rho^k}_1(s_1,\wh b_k;\disc(\Mcal)) - \wh V_{1,k}^\downarrow(s_1,\wh b_k) } + \tau^{-1}H\eta. \tag{\cref{thm:adapted-discretized-mdp}}
\end{align*}
By \cref{ass:continuous-densities-discretized} (which applies to $\adapted(\wh\rho^k,\wh b_k)$), \cref{lem:strongly-convex-under-continuity} applies and we have
\begin{align*}
    \prns{b^\star_k-\wh b_k}^2 \leq 2\densitymin^{-1}\prns{ V^{\wh\rho^k}_1(s_1,\wh b_k;\disc(\Mcal)) - \wh V_{1,k}^\downarrow(s_1,\wh b_k) + H\eta }.
\end{align*}
Using $\Var(X)\leq 2(\Var(Y)+\Var(X-Y))$,
\begin{align*}
    &\sum_{k=1}^K\Varb[\wh\rho^k,\wh b_k,\disc(\Mcal)]{\prns{ \wh b_k-\sum_{t=1}^Hr_t }^+\mid\Ecal_k}
    \\&\leq 2\sum_{k=1}^K\Varb[\wh\rho^k,\wh b_k,\disc(\Mcal)]{\prns{ b_k^\star -\sum_{t=1}^Hr_t }^+\mid\Ecal_k} + 2\sum_{k=1}^K\Varb[\wh\rho^k,\wh b_k,\disc(\Mcal)]{\prns{ b_k^\star -\sum_{t=1}^Hr_t }^+-\prns{ \wh b_k -\sum_{t=1}^Hr_t }^+\mid\Ecal_k}
    \\&\leq 2\sum_{k=1}^K\Pr_{\wh\rho^k,\wh b_k,\disc(\Mcal)}\prns{ \sum_{t=1}^Hr_t\leq b_k^\star } + 2\sum_{k=1}^K\prns{\wh b_k-b_k^\star}^2 \tag{ReLU is $1$-Lipschitz}
    \\&\leq 2\sum_{k=1}^K\Pr_{\adapted(\wh\rho^k,\wh b_k),\Mcal}\prns{ \sum_{t=1}^Hr_t\leq b_k^\star } +
    4\densitymin^{-1}\sum_{k=1}^K\prns{V^{\wh\rho^k}_1(s_1,\wh b_k;\disc(\Mcal)) - \wh V_{1,k}^\downarrow(s_1,\wh b_k) + H\eta} \tag{\cref{lem:from-disc-MDP-to-MDP}}
    \\&\leq 2K\tau + 4\densitymin^{-1}HK\eta + 4\densitymin^{-1} \sum_{k=1}^K\prns{V^{\wh\rho^k}_1(s_1,\wh b_k;\disc(\Mcal)) - \wh V_{1,k}^\downarrow(s_1,\wh b_k)}.
\end{align*}
Previously in \cref{thm:bernstein-bonus-regret2}, we also had the $2K\tau + 4\densitymin^{-1} \sum_{k=1}^K\prns{V^{\wh\rho^k}_1(s_1,\wh b_k;\disc(\Mcal)) - \wh V_{1,k}^\downarrow(s_1,\wh b_k)}$ term.
This can be bounded as before.
We've only incurred an extra $4\densitymin^{-1}HK\eta$ term.
So, \cref{eq:goal-bonus-dominant-term} incurs an extra
\begin{align*}
    \sqrt{SAL( 2\cdot 4\densitymin^{-1}HK\eta )}.
\end{align*}
Finally, \cref{eq:goal-bonus-dominant-term} gets multiplied by $6e\tau^{-1}\sqrt{L}$, which gives a final extra regret of $18e\tau^{-1}\sqrt{\densitymin^{-1} SAHK\eta}L$.
\end{proof}

\subsection{Translating regret from discretized MDP to true MDP}
In this section, we prove that running our algorithms in the imagined discretized MDP also has low regret in the true MDP.
Let us first recall the setup.
When running \rlalg{}, we will provide the discretization parameter $\eta$.
The algorithm will discretize the rewards received from the environment when updating $b$ -- in this way, it is running in its own hallucinated discretized MDP, while the regret we care about is in the true MDP.
Since the algorithm is essentially running in the hallucinated discretized MDP, our regret bound applies in the discretized MDP, \ie, roll-outs in expectations are with respect to $\disc(\Mcal)$,
\begin{align*}
    \sum_{k=1}^K\cvar_\tau^\star(\disc(\Mcal))-\cvar_\tau(\wh\rho^k,\wh b_k;\disc(\Mcal))\leq C.
\end{align*}
When rolling out $\wh\rho^k$ from $\wh b_k$ in the hallucinated discretized MDP, we are essentially running $\adapted(\wh\rho^k)$ as described in \cref{thm:adapted-discretized-mdp}.
So the \emph{true} regret in the real MDP is,
\begin{align*}
    &\sum_{k=1}^K\cvar_\tau^\star(\Mcal)-\cvar_\tau(\adapted(\wh\rho^k,\wh b_k);\Mcal)
    \\&\leq\sum_{k=1}^K\cvar_\tau^\star(\disc(\Mcal))-\cvar_\tau(\wh\rho^k,\wh b_k;\disc(\Mcal))+\tau^{-1}H\eta \tag{\cref{thm:discretized-cvar-dominates,thm:adapted-discretized-mdp}}
    \\&\leq C+K\tau^{-1}H\eta.
\end{align*}
In other words, when discretizing our algorithm, we pay an extra regret of at most $K\tau^{-1}H\eta$, where $\eta$ is the discretization parameter. Setting $\eta=1/K$ renders this term lower order.

\subsection{Computational Complexity}\label{sec:appendix-discrete-computational-complexity}
In this section, we compute the running time complexity of \rlalg{} under discretization of $\eta$.
There are two places where discretization comes in,
\begin{enumerate}
    \item At each $h$, we only compute $\wh U^\downarrow_{h,k}(s,b,a)$ for all $s,a$ and $b$ in the grid. So assuming each step takes $T_{step}$, the total run time of DP is $\Ocal(SAH\eta^{-1} T_{step})$.
    \item When computing $\wh b_k$, we only need to search over $\grid$, since we know that the returns distribution is supported on the $\grid$. Thus, the optimal solution, which is the $\tau$-th quantile, lives on the grid.
    This computation costs $\Ocal(\eta^{-1})$, which is lower order.
\end{enumerate}
So the total runtime is $\Ocal(K\cdot SAH\eta^{-1}T_{step})$.
For running with the Hoeffding bonus, each step is dominated by computing the expectation $\wh P_{k}(s,a)^\top\Eb[r_h\sim R(s,a)]{ \wh V_{h+1,k}^\downarrow(\cdot,b-r_h) }$, as the bonus term is a constant.
In the discretized MDP, this expectation can be computed using only grid elements, so $T_{step}=\Ocal(S\eta^{-1})$.

When running with the Bernstein bonus, we also need to consider the complexity of computing the bonus term.
In the bonus term (\cref{eq:bernstein-bonus-def}), the expectation term $\Eb[s'\sim \wh P_{k}(s,a),r_h\sim R(s,a)]{\prns{ \wh V^\uparrow_{h+1,k}(s',b-r_h)-\wh V^\downarrow_{h+1,k}(s',b-r_h) }^2 }$ can be computed in $\Ocal(S\eta^{-1})$.
Notably, the variance term $\Varb[s'\sim \wh P_{k}(s,a)]{\Eb[r_h\sim R(s,a)]{ \wh V^\downarrow_{h+1,k}(s',b') }}$ can also be computed in $\Ocal(S\eta^{-1})$ by first computing the empirical mean (which takes $\Ocal(S\eta^{-1})$).
So for the Berstein bonus, we also have $T_{step}=\Ocal(S\eta^{-1})$.

So the total running time of \rlalg{} with discretized rewards is $\Ocal(S^2AHK\eta^{-2})$.
As remarked by \citep{auer2008near,azar2017minimax}, we can also reduce the computational cost by selectively recomputing the DP after sufficiently many observations have passed.

\newpage
\section{Minimax Lower Bounds for Quantile Estimation}
\begin{theorem}\label{thm:quantile-estimation-minimax}
Let $m(p)=\mathbb I[p>1/2]$ be the median of $\operatorname{Ber}(p)$.
For any $n$,
$$
\inf_{f:\{0,1\}^n\to[0,1]}\sup_{p\in[0,1]}\mathbb E_{Y_1,\dots,Y_n\sim_{\text{iid}}\operatorname{Ber}(p),\,\hat m_n\sim \operatorname{Ber}(f(Y_1,\dots,Y_n))}[(\hat m_n-m(p))^2]\geq\frac12.
$$
That is, the minimax mean-squared error of any (potentially randomized) estimator for the median of a Bernoulli based on $n$ observations thereon is bounded away from 0 for all $n$.
\end{theorem}
\begin{proof}
Let $g(f,p)$ denote the value of the game (the objective of the above inf-sup). Let $\mathbb P_{p}$ denote the measure of $(Y_1,\dots,Y_n)\sim\operatorname{Ber}(p)^n$.
Fix any $\epsilon\in(0,\frac{e-1}{e+1}]$.
Then
\begin{align*}
\inf_{f:\{0,1\}^n\to[0,1]}\sup_{p\in[0,1]} g(f,p)
&\geq
\inf_{f:\{0,1\}^n\to[0,1]}\left(\frac12g(f,(1+\epsilon)/2)+\frac12g(f,(1-\epsilon)/2)\right)
\\&
\geq \inf_{f:\{0,1\}^n\to\{0,1\}}\left(\frac12g(f,(1+\epsilon)/2)+\frac12g(f,(1-\epsilon)/2)\right)
\\&
=\inf_{f:\{0,1\}^n\to\{0,1\}}
\left(
\frac12\mathbb P_{(1+\epsilon)/2}(f(Y_1,\dots,Y_n)\neq 1)
+
\frac12\mathbb P_{(1-\epsilon)/2}(f(Y_1,\dots,Y_n)\neq 0)
\right)
\\&
=\inf_{f:\{0,1\}^n\to\{0,1\}}
\frac12-\left(
\frac12\mathbb P_{(1+\epsilon)/2}(f(Y_1,\dots,Y_n)=1)
-
\frac12\mathbb P_{(1-\epsilon)/2}(f(Y_1,\dots,Y_n)=1)
\right)
\\&
\geq
\frac12-\frac12\sqrt{\frac12\operatorname{KL}(\mathbb P_{(1+\epsilon)/2},\mathbb P_{(1-\epsilon)/2})}
\\&
=
\frac12-\frac12\sqrt{\frac12 n \epsilon \log((1+\epsilon)/(1-\epsilon))}
\\&
\geq
\frac12-\sqrt{\frac{n}{8}}\sqrt{\frac{e+1}{e-1}}\epsilon.
\end{align*}
The first line is because worst-case risk upper bounds any Bayesian risk. The second because Bayesian risk is optimized by non-randomized estimators. The third by writing the expectation of a 0-1 variable as a probability. The fourth by total probability. The fifth by Pinsker's inequality. The sixth by evaluating the divergence. And the last by convexity of $\log((1+\epsilon)/(1-\epsilon))$.

Since $\epsilon\in(0,\frac{e-1}{e+1}]$ was arbitrary (for fixed $n$), the conclusion is reached.
\end{proof}
\end{document}